\documentclass[nohyperref]{article}
\pdfoutput=1
\usepackage{microtype}
\usepackage{graphicx}
\usepackage{subfigure}
\usepackage{booktabs} 

\usepackage{hyperref}

\usepackage{authblk}
\usepackage{natbib}
\usepackage{algorithm}
\usepackage{algorithmic}
\usepackage{fullpage}
\setcitestyle{authoryear,round,citesep={;},aysep={,},yysep={;}}

\usepackage{graphicx}
\usepackage{subfigure}
\usepackage{xcolor}

\usepackage{booktabs}
\usepackage{color-edits}
\addauthor{jamie}{red}
\addauthor{matt}{brown}
\addauthor{saba}{teal}
\addauthor{ali}{purple}
\addauthor{pat}{orange}

\usepackage{hyperref}
\hypersetup{
    colorlinks=true,
    linkcolor=red,
    citecolor=blue,
    filecolor=magenta,      
    urlcolor=cyan,
    pdfpagemode=FullScreen,
    }
\usepackage{amsmath}
\usepackage{amsthm}
\usepackage{thmtools}
\usepackage{thm-restate}
\usepackage{amssymb}
\usepackage{overpic}
\usepackage{enumitem}
\usepackage{multirow}
\usepackage{dsfont}
\usepackage{mathtools}
\usepackage{float}
\usepackage{sidecap}
\usepackage{cleveref}

\newcommand{\dataset}{\mathcal{D}}
\newcommand{\R}{\mathbb{R}}
\newcommand{\N}{\mathbb{N}}
\newcommand{\indi}{\mathds{1}}

\newcommand{\argmin}{\operatornamewithlimits{argmin}}
\newcommand{\Nrunf}{\text{\#\,Uns}}
\newcommand{\Maxviol}{\text{MaxVi}}
\newcommand{\Meanviol}{\text{MeanVi}}

\newcommand{\cost}{\text{Co}}
\newcommand{\Obj}{\text{Obj}}
\newcommand{\sI}{\mathcal{I}}
\newcommand{\sE}{\mathcal{E}}
\newcommand{\sF}{\mathcal{F}}
\newcommand{\sS}{\mathcal{S}}
\newcommand{\sT}{\mathcal{T}}
\newcommand{\B}{\mathrm{B}}
\newcommand{\dist}{\Delta}
\newcommand{\sol}{\mathrm{SOL}}
\newcommand{\eps}{\varepsilon}
\newcommand{\violation}{\text{Vi}}
\newcommand{\ol}{\overline} 

\newcommand{\indf}{individually fair}

\newcommand{\ipst}{IP-stability}
\newcommand{\ipste}{IP-stable}

\newcommand{\depth}{\textrm{depth}}

\definecolor{mygreen}{rgb}{0,0.5,0}

\newcommand\jamie[1]{\textcolor{mygreen}{JM:#1}}
\newcommand\ali[1]{\textcolor{red}{AV: #1}}

\newtheorem{lemma}{Lemma}
\newtheorem{definition}{Definition}
\newtheorem{theorem}{Theorem}
\newtheorem{claim}[theorem]{Claim}
\newtheorem{corollary}{Corollary}

\begin{document}
\date{}

\title{Individual Preference Stability for Clustering}

\author[1]{Saba Ahmadi\thanks{saba@ttic.edu}}
\author[2]{Pranjal Awasthi\thanks{pranjalawasthi@google.com}}
\author[3]{Samir Khuller\thanks{samir.khuller@northwestern.edu}}
\author[4]{Matth{\"a}us Kleindessner\thanks{matkle@amazon.de}}
\author[5]{Jamie Morgenstern\thanks{jamiemmt@cs.washington.edu}}
\author[3]{Pattara Sukprasert\thanks{pattara@u.northwestern.edu}}
\author[1]{Ali Vakilian\thanks{vakilian@ttic.edu}}
\affil[1]{Toyota Technological Institute at Chicago, USA}
\affil[2]{Google, USA}
\affil[3]{Northwestern University, USA}
\affil[4]{Amazon Web Service, Germany}
\affil[5]{University of Washington, USA}

\maketitle

\begin{abstract}
In this paper, we propose a natural notion of individual preference (IP) stability for clustering, which asks that every data point, on average, is closer to the points in
its own cluster than to the points in any other cluster. Our notion can be motivated from several perspectives, including game theory and algorithmic fairness. 
We study several questions related to our proposed notion.
We first show that deciding whether a given data set allows for an IP-stable clustering in general is NP-hard. 
As a result, we explore the design of efficient algorithms for finding IP-stable clusterings in some restricted metric spaces. We present a
polytime algorithm to find a clustering satisfying exact IP-stability on the real line, and an efficient algorithm to find an IP-stable 2-clustering for a tree metric. 
We also consider relaxing the stability constraint, i.e., every data point should not be too far from its own cluster compared to any other cluster. 
For this case, we provide 
polytime algorithms with different guarantees.
We evaluate 
some of 
our algorithms and several standard clustering approaches on~real~data~sets. 
\end{abstract}



\section{Introduction}\label{section_introduction}

Clustering, a foundational topic in unsupervised learning, aims to find a partition of a dataset into sets where the 
intra-set similarity is higher than the inter-set similarity.  
The problem of clustering 
can be formalized in numerous 
ways with different ways of measuring similarity within and between sets, such as centroid-based 
formulations 
like
$k$-means, $k$-median and $k$-center \citep[e.g., ][]{kmeans_plusplus}, hierarchical partitionings \citep[e.g., ][]{dasgupta2002performance}, or spectral clustering \citep[e.g., ][]{Luxburg_tutorial}. 
All 
these formulations have also been considered 
under 
additional constraints \citep[e.g., ][]{wagstaff2001constrained};  
in particular, a recent line of work studies clustering under various fairness constraints \citep[e.g., ][]{fair_clustering_Nips2017}. 
Most formulations of clustering are NP-hard, which has led to both the understanding of approximation algorithms and conditions under which (nearly) optimal clusterings can be found in computationally efficient ways. Such conditions that have been studied in recent years are variants of {\em perturbation stability}, i.e., small perturbations to the distances do not affect the optimal clustering \citep{ackerman2009clusterability, awasthi2012center,bilu2012stable}, or are variants of {\em approximation stability}, i.e., approximately optimal clusterings  according to some objective are all close to each other~\citep{balcan2013clustering,ostrovsky2013effectiveness}.
In this work, we introduce a distinct notion of stability in clustering, \emph{individual preference stability} of a clustering, which measures whether data points can reduce their individual objective by reassigning themselves to another cluster. 

The study of individual preference (IP) stability, and clusterings which are IP-stable, has a number of motivations behind it, both in applications and connections to other concepts in computing. For example, clustering can be used  to design curricula for a collection of students, with the goal that each student is assigned to a cluster with the most similar learning needs.  One might also design personalized marketing campaigns where customers are first clustered--the campaign's personalization will be more effective if customers have most affinity to the clusters to which they are assigned. More generally, if one  first partitions a dataset and then designs a collection of interventions or treatments for each subset, IP stability ensures individuals are best represented by the cluster they belong to. 

This notion has connections to several other concepts studied in computing, both for clustering and for other algorithmic problems. Suppose each data point ``belongs'' to an individual, and that individual chooses which cluster she wants to join to minimize her average distance to points in that cluster. If her features are fixed (and she cannot change them), an IP-stable clustering will correspond to a Nash equilibrium of this game.  In this sense, IP-stability is related to the 
concept of stability from matching theory~\citep{roth1984evolution,gale1985some}. IP-stability is also a natural notion of individual fairness for clustering, asking that each individual prefers the cluster they belong to over any other cluster (and directly captures envy-freeness of a clustering such as has been recently studied within the context of classification~\citep{NEURIPS2019_e94550c9}). Finally, the study of IP-stability (whether such a clustering exists, whether an objective-maximizing clustering is also IP-stable) may open up a new set of conditions under which approximately optimal clusterings can be found more efficiently than in the worst~case.

\paragraph{Our Contributions} We propose a natural notion of IP-stability clustering and present a comprehensive analysis of its properties. Our notion requires that each data point, on average, is closer to the points in its own cluster than to the points in any other cluster (Section~\ref{section_IPdefinition}). We show that, in general, {\ipste} clusterings might not exist, even for Euclidean data sets~in~$\R^2$ (Section~\ref{section_IPdefinition}).

Moreover, we prove that deciding whether a given data set has an {\ipste} $k$-clustering is NP-hard, even for $k=2$ and when the distance function is assumed to be a metric~(Section~\ref{section_np_hardness}). Here, and in the following, $k$ is the desired number of clusters.
This naturally motivates the study of approximation algorithms for the problem and the study of the problem in special metric spaces. 
  
A clustering is called $t$-approximately IP-stable if the average distance of each data point to its own cluster is not more than $t$ times its average distance to the points in any other cluster. By exploiting the techniques from metric embedding, in Section~\ref{sec:exclusion-approx}, first we provide a polynomial time algorithm that finds an  $\mathcal{O}(\log^2{n}/\varepsilon)$-approximation IP-stable clustering for $(1-\varepsilon)$-fraction of points in any metric space. Second, by designing a modified single-linkage approach as a pre-processing step, in Section~\ref{sec:underlying_stable_clustering} we provide an ``efficient'' approximation algorithm when the input has a ``well-separated'' IP-stable clustering. More precisely, if the input has an IP-stable clustering with clusters of size at least $\alpha \cdot n$, then our algorithm will find an $\tilde{\mathcal O}(\frac{1}{\alpha})$-approximation IP-stable clustering of the input in polynomial time.

When the data set lies on the real line, surprisingly, we show that an {\ipste} $k$-clustering always exists, and we design an efficient algorithm with running time $\mathcal O(kn)$ to find one, where $n$ is the number of data points (Section~\ref{section_1dim}). In addition to finding a IP-stable clustering, one might want to optimize an additional global objective; we study such a problem in Appendix~\ref{appendix_p_equals_infty}, where we minimize the deviation of clusters' sizes from desirable target sizes. We propose a DP solution which runs in $\mathcal O(n^3 k)$ time for this problem. Moreover, we show how to efficiently find an {\ipste} 2-clustering when the distance function is a tree metric~(Section~\ref{sec:tree-metric}).
We in fact conjecture that the result on the line can be extended to any tree, and 
while we show a positive result only for $k=2$, we believe that a similar result holds for any $k$.


  %
%
%
Finally, we perform extensive experiments on real data sets and compare the performance of 
several standard clustering algorithms such as $k$-means++ and $k$-center 
w.r.t. IP-stability (Section~\ref{section_experiments}).  
%
Although in the worst-case the 
violation of IP-stability by solutions produced by these algorithms can be arbitrarily large (as we show in Section~\ref{section_np_hardness} / Appendix~\ref{sec:hard_instances}), 
some of these algorithms 
perform surprisingly well in practice. 
We also study simple heuristic modifications to make the standard algorithms more aligned with the notion of IP-stability, in particular for 
linkage clustering 
(Appendix~\ref{appendix_exp_general}).

\section{Individual Preference Stability}\label{section_IPdefinition}

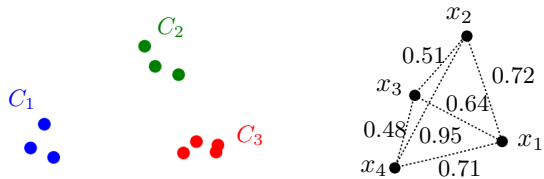
\begin{figure}[t]
 \centering
 \includegraphics[scale=0.9]{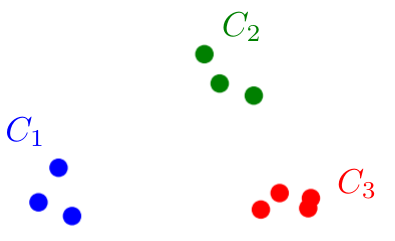}
 \hspace{7mm}
  \begin{overpic}[scale=0.25,trim=0 40 0 40 ]{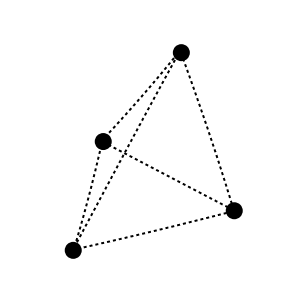}
 \put(86,14){$x_1$}
\put(50,78){$x_2$}
\put(15,42){$x_3$}
\put(7,1){$x_4$}
\put(73,46){\small $0.72$}
\put(50,32){\small $0.64$} 
 \put(46,-1){\small $0.71$}
  \put(8.5,19){\small $0.48$}
  \put(28,56){\small $0.51$}
  \put(36.5,15.5){\small $0.95$}
\end{overpic}

 \caption{Two data sets in $\R^2$ with $d$ equaling the Euclidean metric. \textbf{Left:} An 
 {\ipste} 3-clustering with clusters $C_1,C_2,C_3$.  
 \textbf{Right:} Four points for which there is no {\ipste} 2-clustering. For example, if 
 the two clusters were 
 $C_1=\{x_1,x_4\}$ and $C_2=\{x_2,x_3\}$, then $x_1$ would not be stable because of  
$d(x_1,x_4)=0.71>0.68=[d(x_1,x_2)+d(x_1,x_3)]/2$,
and if the two clusters were
 $C_1 = \{x_1\}$ and $C_2 =\{x_2,x_3,x_4\}$, then $x_4$ would not be stable because
 $d(x_1,x_4) = 0.71 < 0.715 = [d(x_2,x_4) + d(x_3,x_4)]/2$.
 }\label{figure_example}
\end{figure}

Our notion of individual preference stability
applies to a data set~$\dataset$ together 
with a 
given  dissimilarity function~$d$ that measures how 
close
two data points are. 
We 
use the terms  dissimilarity and distance synonymously. 
We assume $d:\dataset\times \dataset\rightarrow\R_{\geq 0}$ to be 
symmetric with $d(x,x)=0$, but not necessarily to 
be a metric (i.e., to additionally satisfy the triangle~inequality and $d(x,y)=0 \Leftrightarrow x=y$).


Our stability notion defines what it means that a data point is 
{\ipste} 
in a clustering of $\dataset$; namely: 
a data point is {\ipste} if the average distance to the points 
in its own cluster (the point itself excluded) is not greater than the average distance to the points in any other cluster. 
Then 
a clustering 
of $\mathcal{D}$ 
is said to be {\ipste} if every data point 
in $\mathcal{D}$ is~stable.

For the rest of the paper we assume $\dataset$ to be finite. Our definition of {\ipst} for clustering can then be formally stated as follows 
(for $l\in\N$, we 
let 
$[l]=\{1,\ldots,l\}$):

\begin{definition}
[Individual preference (IP) stability]
\label{def_ip_stable_clustering}
 Let $\mathcal{C}=(C_1,\ldots,C_k)$ be a $k$-clustering of $\dataset$, that is 
 $\dataset=C_1\dot{\cup}\ldots\dot{\cup} C_k$ and 
 $C_i\neq \emptyset$ for $i\in[k]$. For $x\in \dataset$, we write $C(x)$ 
 for the cluster $C_i$ that $x$ belongs to. 
 We say that  $x\in \dataset$ is 
  {\ipste}
if either $C(x)=\{x\}$ or
 \begin{align}\label{def_ip_stable_ineq}
\frac{1}{|C(x)|-1}
\sum_{y\in C(x)} 
d(x,y)\leq \frac{1}{|C_i|}\sum_{y\in C_i} d(x,y) 
 \end{align}
for all $i\in[k]$ with $C_i\neq C(x)$. The clustering $\mathcal{C}$ is 
an {\ipste} $k$-clustering
if every $x\in \dataset$ is {\ipste}.
For brevity, instead of 
 {\ipste} we may only say stable.
\end{definition}


\begin{figure}[t]
\centering
\begin{overpic}[scale=0.24,trim=80 80 80 80,clip]{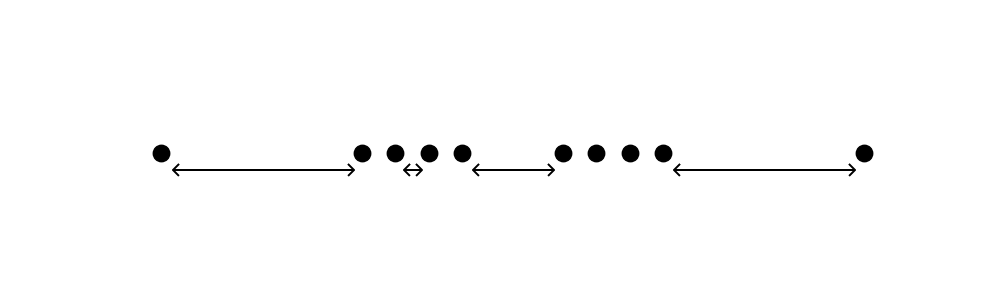}
 \put(21,1.2){$8$}
\put(51,1.2){$1$}
\put(38.3,0.6){\small $\frac{1}{3}$}
\put(81,1.2){$8$}
 \end{overpic}
 
 \vspace{2mm}
 \includegraphics[scale=0.24,trim=80 80 100 80,clip]{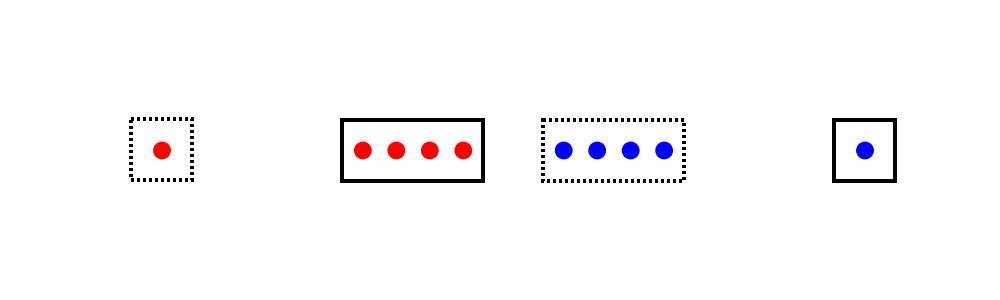}
 
 \caption{A data set 
 on the real line with more than one IP-stable clustering. \textbf{Top:} The data set and the distances between the points. 
 \textbf{Bottom:} The same data set with two 
 IP-stable 
 2-clusterings (one encoded by color: red vs blue / one encoded by frames: solid vs dotted boundary).
 }\label{figure_example_2}
\end{figure}

We 
discuss some important observations about {\ipst} 
as defined in Definition~\ref{def_ip_stable_clustering}: 
if in a clustering all clusters are well-separated and sufficiently far apart, 
then this clustering is {\ipste}. 
An example of such a scenario is provided in the left part of Figure~\ref{figure_example}. 
Hence, at least for such  simple clustering problems
with an ``obvious'' solution, 
{\ipst} does not conflict with the clustering 
goal of
partitioning 
the data set  
such that 
``data points in the same cluster are similar to each other, and data points in different clusters are dissimilar'' \citep[][p.~306]{celebi2016}.
However, there are 
also 
data sets for which no 
{\ipste} $k$-clustering exists 
(for a fixed~$k$ 
and a given distance function~$d$).\footnote{
Of course, the trivial $1$-clustering $\mathcal{C}=(\dataset)$ or  
the 
trivial $|\dataset|$-clustering that puts every 
data 
point in a singleton are 
stable, 
and for a trivial distance function~$d\equiv 0$, every clustering is stable.
}
This can even happen for Euclidean data sets and $k=2$, as the right part of Figure~\ref{figure_example} shows.
If a data set allows 
for an {\ipste} $k$-clustering, there might be more than one {\ipste} $k$-clustering. An example of this is
shown in 
Figure~\ref{figure_example_2}. This example 
also 
illustrates 
that 
{\ipst} does not necessarily work towards the 
aforementioned 
clustering 
goal.  
Indeed, in Figure~\ref{figure_example_2} the 
two 
clusters of the  clustering encoded by the frames, which is 
stable,  are not even~contiguous.  

These observations raise 
a 
number 
of questions such as: when does 
an IP-stable $k$-clustering exist? Can we efficiently decide whether an IP-stable $k$-clustering exists? 
If an IP-stable $k$-clustering exists, can we efficiently compute it? Can we minimize some (clustering) objective 
over the set of all 
IP-stable clusterings? 
How do standard clustering algorithms such as 
Lloyd's algorithm (aka $k$-means; \citealp{lloyd_algorithm}) 
or linkage clustering \citep[e.g.,][Section~22]{shalev2014understanding}
perform in terms of 
{\ipst}? 
Are there simple modifications to these 
algorithms 
in order to 
improve their stability?
In this paper, we explore some of these questions as outlined 
in Section~\ref{section_introduction}. 
%

\subsection{Related Work and Concepts}\label{sec_related_work_and_concepts}

\paragraph{Clustering Stability}
There is a large body of work on the design of efficient clustering algorithms, both in the worst case and under various stability notions \citep{awasthi2014center}. 
Some works 
have studied stability notions called ``average stability'' that are similar to our notion of IP-stability. The work of \citet{vempala2008} studies properties of a similarity function that are sufficient in order to approximately recover 
(in either a list 
model 
or a tree model) 
an unknown 
ground-truth clustering. One of the weaker properties they consider is 
the  average attraction property, 
which 
requires inequality~\eqref{def_ip_stable_ineq} with $t=1$ to hold 
for the ground-truth clustering, 
but 
with 
an additive 
gap of~$\gamma>0$ between the 
left and the right side 
of 
\eqref{def_ip_stable_ineq}. 
\citeauthor{vempala2008} show that the average attraction property is 
sufficient to successfully cluster in the list model, but with the length of the list 
being exponential in $1/\gamma$, and is not sufficient to successfully cluster in the tree model.
%
The works discussed above utilize strong forms of average stability to bypass computational barriers and recover the ground-truth clustering. We, on the other hand, focus on both this problem and the complementary question of when does such stability property even hold, either exactly or approximately. Related notions of clustering stability such as perturbation stability and approximation stability have also been studied. However, the goal in these works is to approximate a clustering objective such as $k$-means or $k$-median under stability assumptions on the optimizer of the objective \citep{ackerman2009clusterability,awasthi2012center,bilu2012stable, balcan2013clustering, balcan2016clustering,makarychev2016metric}.

Closely related to our work is the paper by  \citet{daniely2012clustering}. 
They show that if there is an unknown ground-truth clustering with cluster size at least $\alpha n$ satisfying a slightly stronger requirement than IP-stability, i.e., each point, on average, is at least a factor 3 closer to the points in its own cluster than to the points in any other cluster, then they can find an $\mathcal O(1)$-approximately IP-stable clustering. However, their algorithm runs in time $n^{\Omega(\log 1/\alpha)} = \Omega(n^{\log k})$.



\paragraph{Individual Fairness and Fairness in Clustering}
In the fairness literature, fairness notions are commonly categorized into notions of individual fairness or group fairness. The latter ensures fairness for groups of individuals (such as men vs women), whereas the former aims to ensure fairness for \emph{single} individuals. Our proposed notion of IP-stability can be interpreted as a notion of individual fairness for clustering.
Individual fairness was originally proposed by \citet{fta2012} for classification, 
requiring that similar data points (as measured by a given task-specific metric) should receive a similar prediction. 
It has also been studied in the setting of online learning \citep{joseph2016, joseph2018, gillen2018}. 
Recently, two notions of individual fairness for clustering have been proposed. The first notion, proposed by \citet{jung2020} and also studied by \citet{mahabadi2020}, \citet{chakrabarty2021better} and \citet{vakilian2022improved}, asks that every data point is somewhat close to a center, where ``somewhat'' depends on how close the data point is to its $\lceil n/k \rceil$ nearest neighbors. \citeauthor{jung2020} motivate their notion from the perspective of facility location, but it can also be seen in the context of data points that strive to be well represented by being close to a center similarly to the notion of \citet{chen2019} discussed above. Note that our proposed 
notion of IP-stability 
defines ``being well represented'' in terms of the average distance of a data point to the other points in its cluster, rather than the distance to a center, and hence is not restricted to centroid-based clustering.
%
%
The second notion,  proposed by 
\citet{anderson2020}, is analogous to the notion of individual fairness by \citet{fta2012} for classification. It considers probabilistic clustering, where each data point is mapped to a distribution over a set of centers, and requires that similar data points are mapped to similar 
distributions. Interestingly, 
\citeauthor{anderson2020} allow 
the similarity measure used to define fairness 
and the 
similarity measure 
used to evaluate clustering quality 
to 
be the same or different. However, individual fairness as introduced by \citet{fta2012} has often been 
deemed impractical due to 
requiring access to a task-specific metric \citep{ilvento2019}, and it is unclear where a fairness similarity measure that is different from the clustering similarity measure should come from. 
Note that both notions are different from our notion: (1) there exists clusterings that are IP-stable, but will achieve $\alpha = \infty$ according $\alpha$-fairness notion of \citeauthor{jung2020} and (2) the distributional property of \citeauthor{anderson2020} is satisfied with uniform distribution whereas our stability clusterings do not always exist.


%

\paragraph{Hedonic Games}
A related line of work is the class of hedonic games as a model of coalition formation \citep{dreze1980hedonic,BOGOMOLNAIA2002201,elkind2016price}. In a hedonic game the utility of a player only depends on the identity of players belonging to her coalition.  \citet{feldman-hedonic-clustering} study clustering from the perspective of hedonic games. Our work is different from theirs in the sense that in our model the data points are not~selfish~players.

\section{NP-Hardness}\label{section_np_hardness}

We show that in general metrics, the problem of deciding whether an IP-stable $k$-clustering exists is NP-hard.
For such a result it is crucial to specify how an input instance is encoded: we assume that a data set~$\dataset$ together with a distance function~$d$ is represented by the distance matrix~$(d(x,y))_{x,y\in\dataset}$. 
Under this assumption we can prove the following~theorem:

\begin{theorem}[NP-hardness of {\ipste} clustering]\label{theorem_hardness}
Deciding whether a data set~$\dataset$ together with a distance function~$d$ has an IP-stable $k$-clustering 
is NP-hard. This holds even if $k=2$ and $d$ is a metric distance.   
\end{theorem}
Due to space limitation, the proof of this theorem and all other missing proofs are deferred to the appendix.

The proof is via a reduction from a variant of 3-SAT where the number of clauses is equal to the number of variables and each variable occurs in at most three clauses. 
Unless $\text{P}=\text{NP}$, Theorem~\ref{theorem_hardness} implies that 
for general data sets, even when being guaranteed that an IP-stable $k$-clustering exists, there cannot be any efficient algorithm for computing such an {\ipste} clustering. 
However, as with all NP-hard problems, there are two possible remedies. 
The first remedy is to look at approximately IP-stable clusterings:

\begin{definition}
[Approximate IP-stability]
\label{def_ip_stable_clustering_APPROX}
 Let $\mathcal{C}=(C_1,\ldots, C_k)$ be a $k$-clustering of $\dataset$ and for $x\in \dataset$ let $C(x)$ 
 be the cluster $C_i$ that $x$ belongs to. 
 We say that for some $t\geq 1$, $x\in \dataset$ is 
 $t$-approximately {\ipste}
if either $C(x)=\{x\}$ or
 \begin{align}\label{def_ip_stable_ineq_APPROX}
\frac{1}{|C(x)|-1}
\sum_{y\in C(x)} 
d(x,y)\leq \frac{t}{|C_i|}\sum_{y\in C_i} d(x,y) 
 \end{align}
for every $C_i\neq C(x)$. The clustering $\mathcal{C}$ is 
$t$-approximately {\ipste}
if every $x\in \dataset$ is $t$-approximately {\ipste}.
\end{definition}



An alternative way of defining 
approximate IP-stability, which is explored in Section~\ref{sec:exclusion-approx}, would be to allow a violation of inequality \eqref{def_ip_stable_ineq} in Definition~\ref{def_ip_stable_clustering} for a certain number of  points. 
In our experiments in Section~\ref{section_experiments} we will actually consider both these notions of approximation. 


The second remedy is to restrict our considerations to data sets with some special structure. This is what we do 
in Section~\ref{sec:special-metrics} where we consider the 1-dimensional Euclidean metric and tree metrics. 

To complement Theorem~\ref{theorem_hardness}, 
we show theoretical lower bounds for the standard clustering algorithms $k$-means++, $k$-center, and single linkage. 
Their violation of IP-stability can be arbitrarily large. 

\begin{restatable}{theorem}{badexamples}
\label{thm:hard_instances}
For any $\alpha>1$, there exist separate examples where the clusterings produced by $k$-means++, $k$-center, and single linkage algorithms are $t$-approximately IP-stable only for $t\ge \alpha$. 
\end{restatable}

\section{Approximation Algorithms for IP-Stability}
\label{sec:approx-stability}
In this section, we provide two algorithms for finding approximately IP-stable clustering. Our first algorithm finds a partially IP-stable clustering and the second one finds an approximately IP-stable clustering if the input point set admits a clustering that satisfies a set of requirements that are slightly stronger than the requirements of IP-stability.

\subsection{Partially Stable Clustering}\label{sec:exclusion-approx}

Here, we show that if we only want to cluster $(1-\varepsilon)$-fraction of the points, it is possible to find a $\mathcal O(\frac{\log^2 n}{\varepsilon})$-approximation for IP-stable $k$-clustering in any metric space.

We first define {\em hierarchically well-separated trees} (HSTs). 

\begin{definition}[\citet{Bartal96}] 
    A $t$-hierarchically well-separated tree ($t$-HST) is defined as a rooted weighted tree with the following properties:
    \begin{itemize}
        \item The edge weight from any node to each of its children is the same.
        \item The edge weight along any path from the root to a leaf are decreasing by a factor of at least $t$.
    \end{itemize}
\end{definition}

\citet{FRT04} shows that it is possible to have a tree embedding where distortion is $\mathcal O(\log n)$ in expectation. Recently, \citet{haeupler2021tree} proposes an algorithm that give a worst-case distortion, which they call \emph{partial tree embeddings} if it is allowed to remove a constant fraction of points from the embedding. While their work concerns hop-constrained network design, we provide a simplified version of their result which is useful in our context.

\begin{theorem}[\citet{haeupler2021tree}]
\label{thm:main_tree_embedding}
Given weighted graph $G=(V,E,w)$, $0 < \varepsilon < \frac 1 3$ and root $r \in V$, there is a polynomial-time algorithm which samples from a distribution over partial tree embeddings whose trees are $2$-HST and rooted at $r$ with exclusion probability\footnote{Let $\mathcal D$ be distribution of partial tree metrics of $G$. $\mathcal D$ has \emph{exclusion probability} $\epsilon$ if for all $v \in V$, we have $\text{Pr}_{d \sim \mathcal D}[v~\in~V_d]~\geq~1-\epsilon$} $\varepsilon$ and worst-case distance stretch $\mathcal O(\frac{\log^2 n}{\varepsilon})$.
\end{theorem}


\begin{claim}
\label{clm:hst-extend}
Without loss of generality, we can assume the following two properties from the construction in \Cref{thm:main_tree_embedding}: (1) $V$ is the set of leaves of $T$, and (2) $\depth(u) = \depth(v)$ for any $u,v \in V$.
\end{claim}

Our remaining task is to find an IP-stable clustering of  \emph{leaves} 
(according to the tree metric distance). If we can do that, then it follows that we have an $\mathcal O(\log^2{(n)}/\varepsilon)$-approximately IP-stable $k$-clustering.

\begin{claim}
\label{claim:main_clustering_hst}
    There exists a $k$-clustering $\mathcal C = (C_1,\ldots, C_k)$ of leaves on $t$-HST for any $t \geq 2$ such that $\mathcal C$ is IP-stable for any leaf node. Moreover, for $u,v \in C_i$ and $w \in C_j$, $d_T(u,v) \leq d_T(u,w)$.
\end{claim}

By combining the tree embedding result and our clustering on HSTs we have the following theorem.

\begin{theorem}
\label{thm:main_partial_tree}
    Let $(V,d)$ be a metric. There is a randomized, polynomial time algorithm that produces a clustering $\mathcal{C}=(C_1,\ldots, C_k)$ for $V' \subseteq V$, where $V'$ is taken from $V$ with exclusion probability $\varepsilon$ such that,
    for any node $u~\in~C_i, j\neq~i$, $\overline{d}(u,C_i)~\leq~\mathcal O(\frac{\log^2 n}{\varepsilon}) \overline{d}(u, C_{j })$
    \footnote{We let $\overline{d}(u,C) = \sum_{v \in C} \frac{d(u,v)}{|C \setminus\{u\}|}$ be the average distance from $u$ to cluster $C$. If $|C| = \{u\}$, then $\overline{d}(u,C) = 0$.}.
\end{theorem}
\begin{proof}
    We use Theorem~\ref{thm:main_tree_embedding} to embed $(V,d)$ into a $2$-HST $(V', T)$. We then apply Claim~\ref{claim:main_clustering_hst} to produce $(C_1,\ldots, C_k)$ that is IP-stable in $T$. Since $T$ is a partial tree embedding with worst-case distortion $\mathcal O(\frac{\log^2n}{\varepsilon})$, it follows that $\overline{d}(u,C_i)~\leq~\overline{d}_T(u,C_i)~\leq~\overline{d}_T(u,C_j)~\leq~\mathcal O(\frac{\log^2 n}{\varepsilon})\overline{d}(u,C_j)$.
\end{proof}

\subsection{Instances with Stable Clustering}\label{sec:underlying_stable_clustering}

Next, we show an algorithm that finds an approximately IP-stable clustering, if there is a {\em well-separated} underlying clustering.
Let us first define a well-separated clustering. A similar notion is also defined by~\citet{daniely2012clustering}. 

\begin{definition}
[$(\alpha,\gamma)$-clustering]
 Let $\mathcal{C}=(C_1, \ldots, C_k)$ be a $k$-clustering of $\dataset$, that is 
 $\dataset=C_1\dot{\cup}\ldots\dot{\cup} C_k$ and 
 $C_i\neq \emptyset$ for $i\in[k]$. For $x\in \dataset$, we write $C(x)$ 
 for the cluster $C_i$ that $x$ belongs to. 
 We say that  $\mathcal{C}$ is $(\alpha,\gamma)$-clustering if
 \begin{enumerate}
     \item For all $C_i$, $|C_i| \geq \alpha \cdot n$, where $n = |\dataset|$, and
     \item For all $i\neq j$ and $x \in C_i$, $\overline{d}(x,C_j) \geq \gamma \cdot \overline{d}(x,C_i)$.
 \end{enumerate}
\end{definition}
 
 
 \begin{lemma}
 [\citet{daniely2012clustering}]
 \label{lem:main_danielystructure}
 Let $\mathcal{C} = (C_1, \ldots, C_k)$ be an $(\alpha,\gamma)$-clustering, and let $i \neq j$. Then
 \begin{enumerate}
     \item For every\footnote{\citet{daniely2012clustering} use the term \emph{almost every} to avoid sets of measurement zero.}
     $x \in C_i, y \in C_j$, $\frac{\gamma-1}{\gamma} \ol{d}(y,C_i) \leq d(x,y) \leq \frac{\gamma^2+1}{\gamma(\gamma-1)} \ol{d}(y,C_i)$, and
     \item For every $x,y \in C_i$, $d(x,y) \leq \frac{2}{\gamma-1} \ol{d}(x,C_j)$.
 \end{enumerate}
 \end{lemma}
 
Here, we show that, for large enough $\gamma$, the edges inside a cluster will always be smaller than edges between clusters.

\begin{claim}
\label{clm:main_good_edge_property}
 Let $\mathcal{C} = (C_1, \ldots, C_k)$ be an $(\alpha,\gamma)$-clustering. Let $i \neq j$, and let $x,y \in C_i$ and $z \in C_j$. If $\gamma \geq 2+\sqrt{3}$, then $d(x,y) \leq d(y,z)$.
\end{claim}

\begin{theorem}
\label{thm:main_exact_bf}
Let $\gamma \geq 2+\sqrt{3}$. If there exists an $(\alpha,\gamma)$-clustering, then there is an algorithm that finds such a clustering in time 
$\mathcal O( n^2 \log n + n \cdot \left(\frac{1}{\alpha}\right)^k)$. 
\end{theorem}
\begin{proof}
Let us consider a modification to \emph{single-linkage} algorithm:
We consider edges in non-decreasing order and only merge two clusters if at least one of them has size at most $\alpha n$.
\Cref{clm:main_good_edge_property} suggests that the edges within an underlying cluster will be considered before edges between two clusters. Since we know that any cluster size is at least $\alpha n$, when considering an edge, it is safe to use this condition to merge them. If it is not the case, we can ignore the edge. 

By this process, we will end up with a clustering where each cluster has size at least $\alpha n$. Hence, there are at most $\mathcal O(1/\alpha)$ such clusters. We then can enumerate over all possible clusterings, as there are at most $\mathcal O( \frac{1}{\alpha^k})$ such clusterings, the running time follows.
\end{proof}


Note that the algorithm described in the above theorem has a high running time (especially if $\alpha$ is small). Notice that, before the enumeration, we do not make any mistakes when we run the modified single-linkage, and we get a clustering where each cluster has size at least $\alpha n$. We further show it is possible to define a metric over this clustering and apply \Cref{thm:main_partial_tree} to the metric.
The clustering we get will be approximately IP-stable. 

\paragraph{Conditioning the clusters via single-linkage} When we run the single-linkage algorithm, in addition to the size of any pair of clusters, we can also consider the ratio between each pair of points in two different clusters. We want it so that distances for every pairs of points in two clusters are roughly the same. Moreover, the distance from a point $x$ to its own cluster should not be large compared to the distance from $x$ to any other clusters. We formally define this condition below.

\begin{claim}
\label{clm:main_good_edge_property_2}
 Let $\mathcal{C} = (C_1, \ldots, C_k)$ be an $(\alpha,\gamma)$-clustering. 
 Let $D \subseteq C_i$ and $D' \subseteq C_j$ be two set of points from different clusters. Then for $x \in D$ and $y,y' \in D'$, $\frac{d(x,y)}{d(x,y')} \leq \frac{\gamma^2+1}{(\gamma-1)^2}$.
\end{claim}
\begin{corollary}
\label{col:main_bounded_length}
For two subsets $D,D'$ from different underlying clusters, let $x,x' \in D$ and $y,y' \in D'$. Then
$$\frac{d(x,y)}{d(x',y')} \leq {\left ( \frac{\gamma^2+1}{(\gamma-1)^2} \right )}^2.$$
\end{corollary}

\begin{claim}
\label{clm:main_merge_criteria}
    Let $D,D'$ be two clusters we consider in the single-linkage algorithm.
    Let $e = (x,y)$ be an edge that we merge in an arbitrary step of single-linkage algorithm where $x \in D, y \in D'$, then $D$ and $D'$ must belong to the same underlying clustering if one of the followings is true.
    \begin{enumerate}[leftmargin=*]
        \item $|D| < \alpha \cdot n$ or $|D'| < \alpha \cdot n$,
        \item $\frac{\max_{x\in D, y\in D'} d(x,y)}{\min_{x\in D, y\in D'} d(x,y) } >  {\left ( \frac{\gamma^2+1}{(\gamma-1)^2} \right )}^2$,
        \item There exists $x' \in D$ such that $d(x,x') > \frac{2\gamma}{(\gamma-1)^2} d(x,y)$.
    \end{enumerate}
\end{claim}
    
    
\begin{theorem}
Let $\alpha > 0 ,\gamma \geq 2+\sqrt{3}$. If there exists an $(\alpha, \gamma)$-clustering, then there is a randomized algorithm that finds
a $\mathcal O(\frac{\log^2{(1/\alpha)}}{\alpha})$-approximately IP-stable k-clustering in polynomial time and constant success probability.
\end{theorem}
\begin{proof}
In the first phase, we run the modified single-linkage algorithm: As in the standard single-linkage algorithm, we consider edges in a non-decreasing order of length and merge two clusters using the criteria of \Cref{clm:main_merge_criteria}. 
When this phase finishes, we get a clustering $(D_1, \ldots, D_{\ell})$ with $k \leq \ell \leq \mathcal O(\frac 1 \alpha)$, where
(1) $|D_i| \geq \alpha n$, (2) For any $x \in D_i, y \in D_j$, $d(x,y)$ approximates $\ol{d}(D_i,D_j)$\footnote{$\ol{d}(C,C') = \sum_{x \in C, y\in C'} \frac{d(x,y)}{|C||C'|}$.}
, and (3) $\ol{d}(x,D_i) = \mathcal O( \ol{d}(x,D_j))$. 

This implies that we can pick $\mathcal{R} = \{r_i, \ldots, r_\ell\}$ where each $r_i \in D_i$ is the representative of $D_i$. Then, we show an IP-stable clustering of $\mathcal{R}$ yields an IP-stable clustering of the initial point set as well.

We apply \Cref{thm:main_tree_embedding} to $\mathcal{R}$ with exclusion probability $\varepsilon < \alpha$ to produce a partial tree embedding $T$. By the choice of $\varepsilon$, with constant probability, not a single point in $\mathcal{R}$ is excluded 
We then apply \Cref{claim:main_clustering_hst} to produce a clustering $\mathcal C = ( C_1, \ldots, C_k)$ of $\mathcal{R}$. Let $\mathcal F = (F_1, \ldots, F_k)$ where $F_i = \bigcup_{r_j \in C_i} D_j$ be the final clustering. Next, we show that $\mathcal F$ is an approximately IP-stable clustering. 

    For every $x \in F_i$, let $y$ be the point furthest to $x$ in $F_i$ and $z$ be the point closest to $x$ in $F_j$.
    If we show that $d(x,y) =  \mathcal O( \frac {\log^2{(1/\alpha)}} {\alpha} ) d(x,z)$, then this proves the claim as $\ol{d}(x,F_i) = \mathcal O( d(x,y))$ and $\ol{d}(x,F_j) = \Omega(d(x,z)).$
    
    Let $D_x,D_y,D_z$ be the clusters $x,y,z$ belong to after the merging phases terminates, respectively.
    Also, let $r_x, r_y, r_z$ be the representatives of $D_x,D_y,D_z$.
    By \Cref{claim:main_clustering_hst}, since $r_x,r_y$ belong to the same cluster $C$, and since $r_z$ belongs to another cluster $C'$,
    $d(r_x,r_y) \leq d_T(r_x,r_y) \leq d_T(r_x,r_z) \leq \mathcal O( \frac {\log^2{(1/\alpha)}} {\alpha} ) d(r_x,r_z)$ where the last inequality holds since $T$ is a 2-HST embedding for the representative points with worst-case distortion guarantee of $\mathcal O( \frac {\log^2{(1/\alpha)}} {\alpha} )$.
    Because of \Cref{clm:main_merge_criteria}, $d(x,y) = \Theta(d(r_x,r_y))$ and $d(x,z) = \Theta(d(r_x,r_z))$, it follows that $d(x,y) = \mathcal O( \frac {\log^2{(1/\alpha)}} {\alpha} ) d(x,z)$. Hence, the proof is complete.
\end{proof}
\paragraph{Remark} 
Our analysis relies on the fact that the partial embedding provides a worst-case guarantee. 
However, to the best of our knowledge, a probabilistic tree embedding does not seem to work because the distortion guarantee is on expectation. In other words, 
there could be an edge $e$ such that the distortion is bad, and having only one such edge is enough to break the stability of any $k$-clustering on the tree embedding.

\section{Special Metrics}
\label{sec:special-metrics}
Another approach to circumvent the NP-hardness of IP-stable clustering is to restrict our considerations to data sets with some special structure. Here, we consider $1$-dimensional Euclidean metrics and tree metrics.

\subsection{1-dimensional Euclidean Case}\label{section_1dim}
Here we study the special case of $\dataset\subseteq \R$ where $|\dataset|=n$, and $d$ is the Euclidean metric. We provide an $\mathcal O(kn)$ algorithm that finds an IP-stable $k$-clustering. We show the following Theorem~\ref{thm:1d-stable} holds.

\paragraph{High-level Idea} We start with a specific clustering where all but one cluster are singletons.
By the definition, only nodes in the non-singleton cluster may want to \emph{leave} the cluster (i.e., the IP-stability condition may be violated for a subset of points in that cluster). Lemma~\ref{lemma_boundary_points} shows that if a node $v$ of a cluster $C$ desires to leave $C$, then a boundary node $v'$ of $C$ will want to leave $C$ as well. This allows us to focus only on boundary vertices. There can be at most two vertices per cluster, as long as our clustering is contiguous. We further observe that a cluster will become unstable only when an additional vertex joins the cluster. When this happens, since
the vertex that just joined, which has become a boundary vertex, does not desire to leave the cluster, it must be the boundary vertex on the other end that wants to leave the cluster. Since we maintain the order and the monotonicity of clusters and boundary vertices that we inspect, we end up with $\mathcal O(kn)$ time algorithm. We show that the following theorem holds. 

\begin{theorem}
\label{thm:1d-stable}
Let $\dataset\subseteq \R$ be a point set of size $n$. There exists an algorithm that for any $1\leq k\leq n$, gives an IP-stable $k$-clustering of $\dataset$, and has a time complexity of~$\mathcal O(kn)$. 
\end{theorem}

It is well-known that for any set of $n$ points in a metric space, there exists an embedding to one-dimensional Euclidean space with distortion $\mathcal{O}(n)$. Hence:
\begin{corollary}
Let $\dataset$ be a point set of size $n$ in an arbitrary metric space. There exists a polynomial time algorithm that returns an $\mathcal{O}(n)$-approximately IP-stable $k$-clustering of $\dataset$. 
\end{corollary}

In Appendix~\ref{appendix_p_equals_infty}, we consider an extension of the $1$-dimensional Euclidean case where a target size for each cluster is given, and the goal is to find an IP-stable $k$-clustering that minimizes the total violation from the target cluster sizes.
Formally, the goal is to solve 
\begin{align}\label{1dim-problem}
 \min_{\substack{\mathcal{C}=(C_1,\ldots,C_k):~\mathcal{C}~\text{is an}\\ \text{IP-stable clustering of $\dataset$}\\ \text{with contiguous clusters}}}~ \|(|C_1|-t_1,\ldots,|C_k|-t_k)\|_p,
\end{align}
where $t_1,\ldots,t_k\in[n]$ with $\sum_{i=1}^k t_i=n$ are 
given target cluster sizes, $p\in\R_{\geq 1}\cup\{\infty\}$ and $\|\cdot\|_p$ denotes the $\ell_p$-norm. In Appendix~\ref{appendix_p_equals_infty}, we provide an $\mathcal O(n^3k)$ dynamic programming approach for this problem.
\subsection{IP-Stable Clusterings on Trees}
\label{sec:tree-metric}
In this section, we show how to find  an {\ipste} 2-clustering when $d$ is a tree metric. Let $T = (V,E)$ be a given weighted tree rooted at $r$ ($r$ can be arbitrarily chosen). Our construction will first pick a \emph{boundary edge} among edges adjacent to~$r$. 
We then \emph{rotate} the boundary edge in a systematic manner until the clustering we have, which are two connected components we get by removing the boundary edge~$e$ from $T$, is {\ipste}. 
While it could be the case that non-contiguous {\ipste} clustering exists, we only consider contiguous clustering in our algorithm. Our algorithm implies that such a clustering exists.



For any graph $G=(V,E)$ and $v \in V$, let $C_G(v)$ be the set of vertices reachable from $v$ in $G$. For any tree $T =(V,E)$ and an edge $e \in E$, let $T\setminus e = (V,E \setminus \{e\})$ be the subgraph of $T$ obtained by removing $e$. Let $N(u)$ denote the set of neighbors of $u$ in $T$. We say that $u^f \in N(u)$ is the {\em furthest neighbor} of $u$ if $\overline{d}(u, C_{T{\setminus (u,u^f)}}(u^f)) \geq \overline{d}(u,C_{T{\setminus (u,w)}}(w))$ for any $w \in N(u)$. Next, we define a rotate operation which is crucial for our algorithm and show that the following property holds:



\begin{claim}
\label{claim:rotate-fair}
Suppose the current boundary edge is $(u,v)$, we define the operation $\textrm{rotate}(u)$ to be an operation that moves the boundary edge from $(u,v)$ to $(u,u^f)$. After $\textrm{rotate}(u)$ is called, the clustering defined by the 
   boundary edge $(u,u^f)$ is stable for $u$.
\end{claim}
\paragraph{Algorithm} Now we are ready to describe our algorithm. First, pick an edge $(b_0=r,v)$ arbitrarily among edges adjacent to the root $r$, and call $\textrm{rotate}(b_0)$. After rotation, let $e_1 = (b_0, b_1)$ denote the boundary edge. By Claim~\ref{claim:rotate-fair}, the clustering defined by $e_1$ is stable for $b_0$. If it is also stable for $b_1$, then we have a stable clustering.
If not, suppose the boundary edge is now $e_i = (b_{i-1}, b_i)$. By induction, we assume that the clustering is stable for $b_{i-1}$. If this clustering is not stable, then it is not stable for $b_i$. We call $\textrm{rotate}(b_i)$ to move the boundary edge to $e_{i+1} = (b_i, b_{i+1})$. $b_{i+1}$ cannot be $b_{i-1}$, otherwise, the previous clustering would already be stable for $b_i$. We can repeat this process until we find a stable clustering. The process will terminate because the boundary edge is moved further away from the root at every step. Eventually, we will reach the point where $b_{i}^f = b_{i-1}$. This might happen at a leaf. Once it happens, then we have a stable clustering for the boundary nodes. In~\Cref{appendix:proof-boundary-nodes-stability-suffices}, we show that when considering a contiguous clustering, if the boundary nodes are stable, then every node is stable. In other words:

\begin{claim}
\label{thm:boundary-nodes-stability-suffices}
    Let $e = (u,v)$ be the boundary edge of a 2-clustering of $T$. If $u$ and $v$ are stable, then the clustering is stable for every node $x \in V$. 
\end{claim}

Finally, since the boundary edge is moving away from a fixed vertex, we call $\textrm{rotate}$ at most $T$ times.




\section{Experiments}\label{section_experiments}

%

We ran a number of experiments.\footnote{Code available on \url{https://github.com/amazon-research/ip-stability-for-clustering}} 
Our experiments are intended to  serve as a proof of concept.
They do not focus on the running times of the 
algorithms or their applicability to \emph{large} data sets. Hence,  
we 
only 
use rather small data sets of sizes 500~to~1885.

Let us define some quantities.
For any point $x$, let 
$$\violation(x) = \max_{C_i\neq C(x)}\frac{\frac{1}{|C(x)|-1}
\sum_{y\in C(x)} d(x,y)}{\frac{1}{|C_i|}\sum_{y\in C_i} d(x,y)},$$
where 
we use the convention that $\frac{0}{0}=0$. 
A point~$x$ is IP-stable if and only if $\violation(x)\leq 1$. 
Let $U=\{x\in \dataset: \text{$x$ is not stable}\}$.
We measure the extent to which a $k$-clustering~$\mathcal{C}=(C_1,\ldots,C_k)$ of a dataset $\dataset$ is \mbox{(un-)stable} by 
$\Nrunf = |U|$ (``number of unstable''), $\Maxviol=\max_{x\in\dataset} \violation(x)$ (``maximum violation''), and $\Meanviol$ (``mean violation'') defined as
\begin{align*}
    \Meanviol &=
    \begin{cases}
         \frac {1}{|U|}\sum_{x \in U} \violation(x) &\quad U\neq \emptyset\\
         0 &\quad U=\emptyset
    \end{cases}.
\end{align*}


The clustering~$\mathcal{C}$ is 
{\ipste} if and only if
$\Nrunf=0$ and 
$\Maxviol\leq 1$. 
$\Maxviol$ is the smallest value of $t$ such that $\mathcal{C}$ is a $t$-approximately {\ipste} clustering. 
We measure the quality of~$\mathcal{C}$ w.r.t. the goal of putting similar points into the same cluster by $\cost$ (``cost''), defined as the average within-cluster distance. Formally, 
\begin{align}\label{exp_cost_sq}
\cost=\sum_{i=1}^k \frac{1}{{|C_i| \choose 2}}\sum_{\{x,y\}\in C_i\times C_i}d(x,y).
\end{align}
%
We performed all experiments  
in Python. 
We used the 
standard clustering algorithms from Scikit-learn\footnote{\url{https://scikit-learn.org/}} or SciPy\footnote{\url{https://scipy.org/}} with all parameters set to their default~values.

\subsection{General Data Sets}\label{section_experiments_general}

We performed the same set of experiments on the first 1000 records of the 
Adult data set, the Drug Consumption data set (1885 records), and the Indian Liver Patient data set (579 records), which are all publicly available in the UCI repository \citep{UCI_all_four_data_sets_vers2}. 
%
As distance function~$d$ we  used the Euclidean, 
Manhattan or Chebyshev metric.  
Here we only present the results for the Adult data set 
on the Euclidean metric, the other results  are provided in Appendix~\ref{appendix_exp_general}.
Our observations are largely consistent between the different data sets~and~metrics. 



\paragraph{(Un-)Stability of Standard Algorithms}
Working with  the Adult data set, we only used its six numerical features
(e.g., age, hours worked per week), normalized to zero mean and unit variance, for representing records. 
We applied 
several standard clustering algorithms as well as the group-fair $k$-center algorithm of 
\citet{fair_k_center_2019} (referred to as $k$-center~GF) 
to the data set ($k$-means++; $k$-medoids) or its 
distance matrix
($k$-center using the greedy strategy of \citet{gonzalez1985}; $k$-center GF; single / average~/ complete linkage clustering). 
%
In order to study the extent to which these methods produce (un-)stable clusterings, for $k=2,5,10, 15, 20,30,\ldots,100$, 
we computed 
$\Nrunf$, $\Maxviol$ and $\Meanviol$ as defined 
above 
for the resulting $k$-clusterings. 
For measuring the quality of the clusterings we computed  $\cost$ as defined in~\eqref{exp_cost_sq}.

\newcommand{\scaleA}{0.275}

\begin{figure}[t]
\centering
\includegraphics[width=\columnwidth]{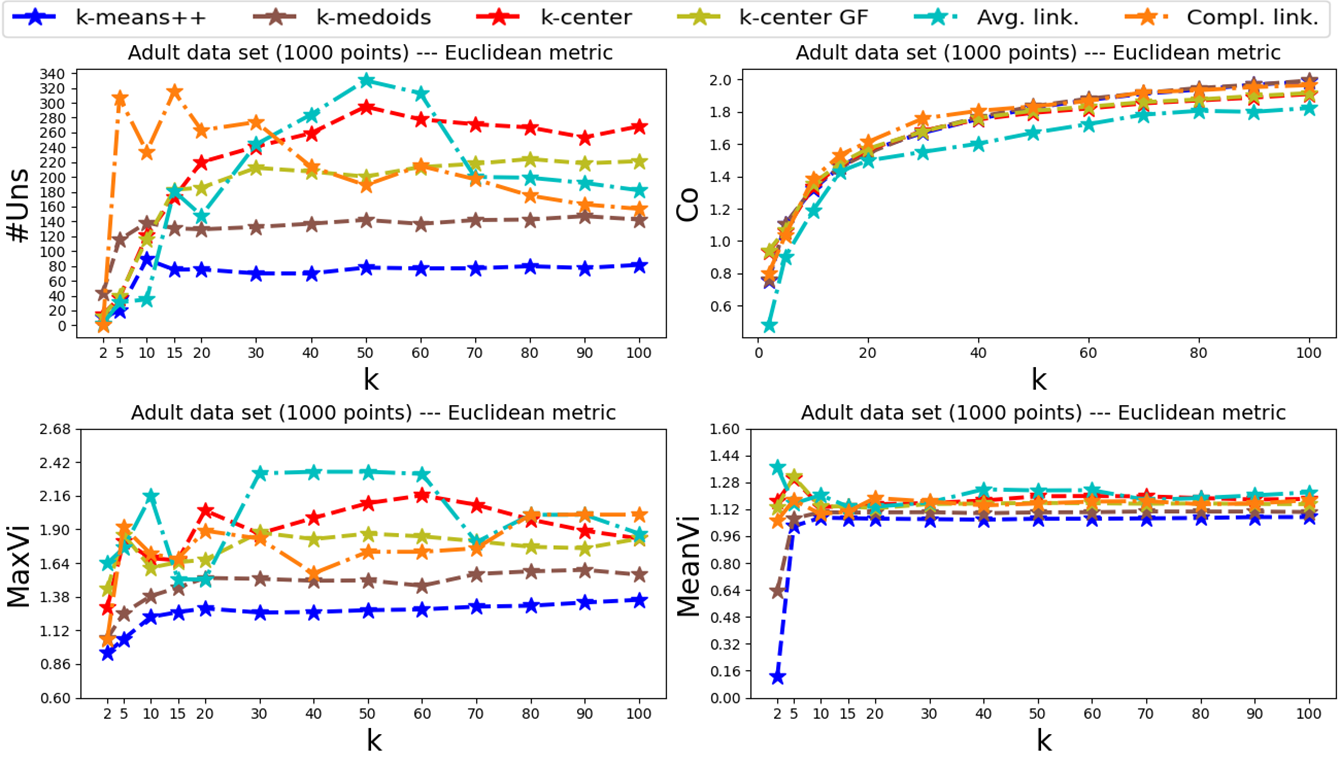}

\vspace{-2mm}
\caption{
$\Nrunf$ \textbf{(top-left)}, 
$\cost$ \textbf{(top-right)}, 
$\Maxviol$ \textbf{(bottom-left)} 
and $\Meanviol$ \textbf{(bottom-right)}
for the clusterings produced by the 
various 
algorithms as a function of 
$k$. $k$-center GF denotes the group-fair $k$-center algorithm of 
\citet{fair_k_center_2019}. 
}
\label{exp_gen_standard_alg_Adult}
\end{figure}


Figure~\ref{exp_gen_standard_alg_Adult} shows the results, where 
we removed the single linkage algorithm since it performs significantly worse than the other algorithms.
For $k$-means++, $k$-medoids, $k$-center and $k$-center GF we show average results obtained from running 
them 
for 
25 times since their outcomes depend on random initializations. 
We can see that, in particular for large values of~$k$, $k$-center, $k$-center GF, and the linkage algorithms can be quite unstable
with 
rather large values of $\Nrunf$ and $\Maxviol$. 
In contrast, $k$-means++ produces quite stable clusterings with $\Nrunf\leq 90$ and $\Maxviol\leq 1.4$ even when $k$ is large. 
For 
a 
baseline comparison, for a random clustering in which every data point was assigned to one of $k=100$ clusters uniformly at random we observed 
$\Nrunf= 990$, $\Maxviol=11.0$, and $\Meanviol=2.5$ on average. 
%
The $k$-medoids algorithm performs worse than $k$-means++, but better than the other algorithms. The clusterings produced by 
$k$-center~GF, which we ran with the constraint of choosing $\lfloor k/2\rfloor$ female and $\lceil k/2\rceil$ male centers,  
are slightly more 
stable than the ones produced by $k$-center.
However, note that it really depends on the data set whether group-fairness correlates with IP-stability or not (see  Appendix~\ref{example_group_fair_vs_indi_fair}). 
All  methods perform similarly w.r.t. the clustering cost~$\cost$.



\paragraph{Heuristics to Improve Standard Algorithms}

One might wonder whether we can  
modify the standard clustering algorithms 
in order to 
make them more stable.  
A natural idea to make any clustering 
more stable is to 
change it locally 
and to iteratively pick a data point that is not stable and assign 
it to the cluster that it is closest to. However, this idea turned out not to work as we observed that  usually we can only pick  a very small number of data points whose reassignment does not cause other 
data 
points that are initially stable to become unstable after the reassignment 
 (see 
 Appendix~\ref{example_local_search}). Another idea that we explored is specifically tied to linkage clustering. We provide the details in 
 Appendix~\ref{appendix_pruning_strategy}.

\subsection{1-dimensional Euclidean Data Sets}\label{section_experiments_1D}
In Appendix~\ref{appendix_exp_1dim} we present experiments with our DP approach 
on real world 1-dimensional Euclidean data sets.

\section{Discussion}\label{section_discussion}



We proposed a notion of IP-stability that aims 
at data points being 
well represented by their clusters. 
Formally, it requires every point, on average, to be closer to the points in its own
cluster than to the points in any other cluster. 
{\ipst} is not restricted to centroid-based clustering and raises numerous questions, some of which we addressed: in a general metric space, we showed it is NP-hard to decide whether an {\ipste} $k$-clustering exists and we provided approximation algorithms for the problem when the input contains a ``well-separated'' clustering or when a small fraction of  inputs can be excluded in the output. 
For one-dimensional Euclidean data sets we can compute an {\ipste} clustering in polynomial time. For the case of tree metrics, we showed how to find an {\ipste} $2$-clustering in polynomial time.
We 
proved 
that in the worst-case scenario some standard clustering algorithms including $k$-means++ provide arbitrarily unstable clusterings. However, our experiments show that $k$-means++ works quite well in practice. 


While our works focus mostly on the upper bound side, understanding the lower bound, e.g., hardness of approximation for {\ipst} is one important open question.
It is a natural direction to explore whether our ideas for {\ipste} $2$-clustering on trees can be extended to $k$-clustering for $k>2$? Finally, it would be very interesting to provide theoretical evidence supporting our empirical findings that show the surprising effectiveness of $k$-means++ for our proposed notion of {\ipst}, in spite of the existence of bad worst-case scenarios.


\section*{Acknowledgements}
SA is supported by the Simons Collaborative grant on Theory of Algorithmic Fairness, and the National Science Foundation
grant CCF-1733556. Part of the research was done when the author was visiting Northwestern University.
JM is supported by fundings from the NSF AI Institute for the Foundations of Machine Learning (IFML), an NSF Career award, and the Simons Collaborative grant on Theory of Algorithmic Fairness.
AV is supported by NSF award CCF-1934843.

\bibliography{mybibfile_fairness}

\begin{thebibliography}{67}
\providecommand{\natexlab}[1]{#1}
\providecommand{\url}[1]{\texttt{#1}}
\expandafter\ifx\csname urlstyle\endcsname\relax
  \providecommand{\doi}[1]{doi: #1}\else
  \providecommand{\doi}{doi: \begingroup \urlstyle{rm}\Url}\fi

\bibitem[Abbasi et~al.(2021)Abbasi, Bhaskara, and
  Venkatasubramanian]{abbasi2021}
M.~Abbasi, A.~Bhaskara, and S.~Venkatasubramanian.
\newblock Fair clustering via equitable group representations.
\newblock In \emph{ACM Conference on Fairness, Accountability, and Transparency
  (ACM FAccT)}, 2021.

\bibitem[Ackerman and Ben-David(2009)]{ackerman2009clusterability}
M.~Ackerman and S.~Ben-David.
\newblock Clusterability: A theoretical study.
\newblock In \emph{International Conference on Artificial Intelligence and
  Statistics (AISTATS)}, 2009.

\bibitem[Ahmadi et~al.(2020)Ahmadi, Galhotra, Saha, and
  Schwartz]{ahmadi2020fair}
S.~Ahmadi, S.~Galhotra, B.~Saha, and R.~Schwartz.
\newblock Fair correlation clustering.
\newblock arXiv:2002.03508 [cs.DS], 2020.

\bibitem[Ahmadian et~al.(2019)Ahmadian, Epasto, Kumar, and
  Mahdian]{ahmadian2019}
S.~Ahmadian, A.~Epasto, R.~Kumar, and M.~Mahdian.
\newblock Clustering without over-representation.
\newblock In \emph{ACM SIGKDD Conference on Knowledge Discovery and Data Mining
  (KDD)}, 2019.

\bibitem[Ahmadian et~al.(2020{\natexlab{a}})Ahmadian, Epasto, Knittel, Kumar,
  Mahdian, Moseley, Pham, Vassilvitskii, and
  Wang]{ahmadian_fair_hierarchical_clustering}
S.~Ahmadian, A.~Epasto, M.~Knittel, R.~Kumar, M.~Mahdian, B.~Moseley, P.~Pham,
  S.~Vassilvitskii, and Y.~Wang.
\newblock Fair hierarchical clustering.
\newblock In \emph{Neural Information Processing Systems (NeurIPS)},
  2020{\natexlab{a}}.

\bibitem[Ahmadian et~al.(2020{\natexlab{b}})Ahmadian, Epasto, Kumar, and
  Mahdian]{ahmadian2020}
S.~Ahmadian, A.~Epasto, R.~Kumar, and M.~Mahdian.
\newblock Fair correlation clustering.
\newblock In \emph{International Conference on Artificial Intelligence and
  Statistics (AISTATS)}, 2020{\natexlab{b}}.

\bibitem[Anagnostopoulos et~al.(2019)Anagnostopoulos, Becchetti, Böhm,
  Fazzone, Leonardi, Menghini, and Schwiegelshohn]{anagnostopoulos2019}
A.~Anagnostopoulos, L.~Becchetti, M.~Böhm, A.~Fazzone, S.~Leonardi,
  C.~Menghini, and C.~Schwiegelshohn.
\newblock Principal fairness: Removing bias via projections.
\newblock arXiv:1905.13651 [cs.DS], 2019.

\bibitem[Anderson et~al.(2020)Anderson, Bera, Das, and Liu]{anderson2020}
N.~Anderson, S.~Bera, S.~Das, and Y.~Liu.
\newblock Distributional individual fairness in clustering.
\newblock arXiv:2006.12589 [cs.LG], 2020.

\bibitem[Arthur and Vassilvitskii(2007)]{kmeans_plusplus}
D.~Arthur and S.~Vassilvitskii.
\newblock k-means++: The advantages of careful seeding.
\newblock In \emph{Symposium on Discrete Algorithms (SODA)}, 2007.

\bibitem[Awasthi and Balcan(2014)]{awasthi2014center}
P.~Awasthi and M.-F. Balcan.
\newblock Center based clustering: A foundational perspective.
\newblock In \emph{Handbook of Cluster Analysis}. CRC Press, 2014.

\bibitem[Awasthi et~al.(2012)Awasthi, Blum, and Sheffet]{awasthi2012center}
P.~Awasthi, A.~Blum, and O.~Sheffet.
\newblock Center-based clustering under perturbation stability.
\newblock \emph{Information Processing Letters}, 112\penalty0 (1-2):\penalty0
  49--54, 2012.

\bibitem[Backurs et~al.(2019)Backurs, Indyk, Onak, Schieber, Vakilian, and
  Wagner]{backurs2019}
A.~Backurs, P.~Indyk, K.~Onak, B.~Schieber, A.~Vakilian, and T.~Wagner.
\newblock Scalable fair clustering.
\newblock In \emph{International Conference on Machine Learning (ICML)}, 2019.

\bibitem[Balcan and Liang(2016)]{balcan2016clustering}
M.-F. Balcan and Y.~Liang.
\newblock Clustering under perturbation resilience.
\newblock \emph{SIAM Journal on Computing}, 45\penalty0 (1):\penalty0 102--155,
  2016.

\bibitem[Balcan et~al.(2008)Balcan, Blum, and Vempala]{vempala2008}
M.-F. Balcan, A.~Blum, and S.~Vempala.
\newblock A discriminative framework for clustering via similarity functions.
\newblock In \emph{ACM Symposium on Theory of Computing (STOC)}, 2008.

\bibitem[Balcan et~al.(2013)Balcan, Blum, and Gupta]{balcan2013clustering}
M.-F. Balcan, A.~Blum, and A.~Gupta.
\newblock Clustering under approximation stability.
\newblock \emph{Journal of the ACM (JACM)}, 60\penalty0 (2):\penalty0 1--34,
  2013.

\bibitem[Balcan et~al.(2019)Balcan, Dick, Noothigattu, and
  Procaccia]{NEURIPS2019_e94550c9}
M.-F. Balcan, T.~Dick, R.~Noothigattu, and A.~Procaccia.
\newblock Envy-free classification.
\newblock In \emph{Neural Information Processing Systems (NeurIPS)}, 2019.

\bibitem[Bartal(1996)]{Bartal96}
Y.~Bartal.
\newblock Probabilistic approximations of metric spaces and its algorithmic
  applications.
\newblock In \emph{Symposium on Foundations of Computer Science (FOCS)}, 1996.

\bibitem[Bera et~al.(2019)Bera, Chakrabarty, Flores, and Negahbani]{bera2019}
S.~Bera, D.~Chakrabarty, N.~Flores, and M.~Negahbani.
\newblock Fair algorithms for clustering.
\newblock In \emph{Neural Information Processing Systems (NeurIPS)}, 2019.

\bibitem[Bercea et~al.(2019)Bercea, Groß, Khuller, Kumar, Rösner, Schmidt,
  and Schmidt]{bercea2019}
I.~O. Bercea, M.~Groß, S.~Khuller, A.~Kumar, C.~Rösner, D.~R. Schmidt, and
  M.~Schmidt.
\newblock On the cost of essentially fair clusterings.
\newblock In \emph{Approximation, Randomization, and Combinatorial
  Optimization. Algorithms and Techniques (APPROX/RANDOM)}, 2019.

\bibitem[Bilu and Linial(2012)]{bilu2012stable}
Y.~Bilu and N.~Linial.
\newblock Are stable instances easy?
\newblock \emph{Combinatorics, Probability and Computing}, 21\penalty0
  (5):\penalty0 643--660, 2012.

\bibitem[Bogomolnaia and Jackson(2002)]{BOGOMOLNAIA2002201}
A.~Bogomolnaia and M.~O. Jackson.
\newblock The stability of hedonic coalition structures.
\newblock \emph{Games and Economic Behavior}, 38\penalty0 (2):\penalty0
  201--230, 2002.

\bibitem[Brubach et~al.(2020)Brubach, Chakrabarti, Dickerson, Khuller,
  Srinivasan, and Tsepenekas]{Brubach2020}
B.~Brubach, D.~Chakrabarti, J.~Dickerson, S.~Khuller, A.~Srinivasan, and
  L.~Tsepenekas.
\newblock A pairwise fair and community-preserving approach to $k$-center
  clustering.
\newblock In \emph{International Conference on Machine Learning (ICML)}, 2020.

\bibitem[Brubach et~al.(2021)Brubach, Chakrabarti, Dickerson, Srinivasan, and
  Tsepenekas]{brubach2021fairness}
B.~Brubach, D.~Chakrabarti, J.~Dickerson, A.~Srinivasan, and L.~Tsepenekas.
\newblock Fairness, semi-supervised learning, and more: A general framework for
  clustering with stochastic pairwise constraints.
\newblock In \emph{AAAI Conference on Artificial Intelligence}, 2021.

\bibitem[Celebi and Aydin(2016)]{celebi2016}
M.~E. Celebi and K.~Aydin.
\newblock \emph{Unsupervised Learning Algorithms}.
\newblock Springer, 2016.

\bibitem[Chakrabarty and Negahbani(2021)]{chakrabarty2021better}
D.~Chakrabarty and M.~Negahbani.
\newblock Better algorithms for individually fair $ k $-clustering.
\newblock \emph{arXiv preprint arXiv:2106.12150}, 2021.

\bibitem[Chen et~al.(2019)Chen, Fain, Lyu, and Munagala]{chen2019}
X.~Chen, B.~Fain, L.~Lyu, and K.~Munagala.
\newblock Proportionally fair clustering.
\newblock In \emph{International Conference on Machine Learning (ICML)}, 2019.

\bibitem[Chierichetti et~al.(2017)Chierichetti, Kumar, Lattanzi, and
  Vassilvitskii]{fair_clustering_Nips2017}
F.~Chierichetti, R.~Kumar, S.~Lattanzi, and S.~Vassilvitskii.
\newblock Fair clustering through fairlets.
\newblock In \emph{Neural Information Processing Systems (NIPS)}, 2017.

\bibitem[Chlamt{\'a}{\v{c}} et~al.(2022)Chlamt{\'a}{\v{c}}, Makarychev, and
  Vakilian]{chlamtavc2022approximating}
E.~Chlamt{\'a}{\v{c}}, Y.~Makarychev, and A.~Vakilian.
\newblock Approximating fair clustering with cascaded norm objectives.
\newblock In \emph{ACM-SIAM Symposium on Discrete Algorithms (SODA)}, 2022.

\bibitem[Daniely et~al.(2012)Daniely, Linial, and Saks]{daniely2012clustering}
A.~Daniely, N.~Linial, and M.~Saks.
\newblock Clustering is difficult only when it does not matter.
\newblock arXiv:1205.4891 [cs.LG]], 2012.

\bibitem[Dasgupta(2002)]{dasgupta2002performance}
S.~Dasgupta.
\newblock Performance guarantees for hierarchical clustering.
\newblock In \emph{International Conference on Computational Learning Theory
  (COLT)}, 2002.

\bibitem[Davidson and Ravi(2020)]{davidson2020}
I.~Davidson and S.~S. Ravi.
\newblock Making existing clusterings fairer: Algorithms, complexity results
  and insights.
\newblock In \emph{AAAI Conference on Artificial Intelligence}, 2020.

\bibitem[Dreze and Greenberg(1980)]{dreze1980hedonic}
J.~H. Dreze and J.~Greenberg.
\newblock Hedonic coalitions: Optimality and stability.
\newblock \emph{Econometrica: Journal of the Econometric Society}, pages
  987--1003, 1980.

\bibitem[Dua and Graff(2019)]{UCI_all_four_data_sets_vers2}
D.~Dua and C.~Graff.
\newblock {UCI} machine learning repository, 2019.
\newblock German Credit data set available on
  \url{https://archive.ics.uci.edu/ml/datasets/Statlog+(German+Credit+Data)}.
  Adult data set available on
  \url{https://archive.ics.uci.edu/ml/datasets/adult}. Drug Consumption data
  set available on
  \url{https://archive.ics.uci.edu/ml/datasets/Drug+consumption+(quantified)}.
  Indian Liver Patient data set available on
  \url{https://archive.ics.uci.edu/ml/datasets/ILPD+(Indian+Liver+Patient+Dataset)}.

\bibitem[Dwork et~al.(2012)Dwork, Hardt, Pitassi, Reingold, and Zemel]{fta2012}
C.~Dwork, M.~Hardt, T.~Pitassi, O.~Reingold, and R.~Zemel.
\newblock Fairness through awareness.
\newblock In \emph{Innovations in Theoretical Computer Science Conference
  (ITCS)}, 2012.

\bibitem[Elkind et~al.(2016)Elkind, Fanelli, and Flammini]{elkind2016price}
E.~Elkind, A.~Fanelli, and M.~Flammini.
\newblock Price of pareto optimality in hedonic games.
\newblock In \emph{AAAI Conference on Artificial Intelligence}, 2016.

\bibitem[Esmaeili et~al.(2020)Esmaeili, Brubach, Tsepenekas, and
  Dickerson]{esmaeili2020}
S.~Esmaeili, B.~Brubach, L.~Tsepenekas, and J.~Dickerson.
\newblock Probabilistic fair clustering.
\newblock In \emph{Neural Information Processing Systems (NeurIPS)}, 2020.

\bibitem[Fakcharoenphol et~al.(2004)Fakcharoenphol, Rao, and Talwar]{FRT04}
J.~Fakcharoenphol, S.~Rao, and K.~Talwar.
\newblock A tight bound on approximating arbitrary metrics by tree metrics.
\newblock \emph{Journal of Computer and System Sciences}, 69\penalty0
  (3):\penalty0 485--497, 2004.

\bibitem[Fehrman et~al.(2015)Fehrman, Muhammad, Mirkes, Egan, and
  Gorban]{drug_consumption_data}
E.~Fehrman, A.~K. Muhammad, E.~M. Mirkes, V.~Egan, and A.~N. Gorban.
\newblock The five factor model of personality and evaluation of drug
  consumption risk.
\newblock arXiv:1506.06297 [stat.AP], 2015.
\newblock Data set available on
  \url{https://archive.ics.uci.edu/ml/datasets/Drug+consumption+(quantified)}.

\bibitem[Feldman et~al.(2015{\natexlab{a}})Feldman, Friedler, Moeller,
  Scheidegger, and Venkatasubramanian]{feldman2015}
M.~Feldman, S.~A. Friedler, J.~Moeller, C.~Scheidegger, and
  S.~Venkatasubramanian.
\newblock Certifying and removing disparate impact.
\newblock In \emph{ACM International Conference on Knowledge Discovery and Data
  Mining (KDD)}, 2015{\natexlab{a}}.

\bibitem[Feldman et~al.(2015{\natexlab{b}})Feldman, Lewin-Eytan, and
  Naor]{feldman-hedonic-clustering}
M.~Feldman, L.~Lewin-Eytan, and J.~Naor.
\newblock Hedonic clustering games.
\newblock \emph{ACM Trans. Parallel Comput.}, 2\penalty0 (1),
  2015{\natexlab{b}}.

\bibitem[Gale and Sotomayor(1985)]{gale1985some}
D.~Gale and M.~Sotomayor.
\newblock Some remarks on the stable matching problem.
\newblock \emph{Discrete Applied Mathematics}, 11\penalty0 (3):\penalty0
  223--232, 1985.

\bibitem[Garey and Johnson(1979)]{garey_comp_and_intractability}
M.~R. Garey and D.~S. Johnson.
\newblock \emph{Computers and Intractability: A Guide to the Theory of
  NP-Completeness}.
\newblock W. H. Freeman and Company, 1979.

\bibitem[Ghadiri et~al.(2021)Ghadiri, Samadi, and Vempala]{ghadiri2021}
M.~Ghadiri, S.~Samadi, and S.~Vempala.
\newblock Socially fair $k$-means clustering.
\newblock In \emph{ACM Conference on Fairness, Accountability, and Transparency
  (ACM FAccT)}, 2021.

\bibitem[Gillen et~al.(2018)Gillen, Jung, Kearns, and Roth]{gillen2018}
S.~Gillen, C.~Jung, M.~Kearns, and A.~Roth.
\newblock Online learning with an unknown fairness metric.
\newblock In \emph{Neural Information Processing Systems (NeurIPS)}, 2018.

\bibitem[Gonzalez(1985)]{gonzalez1985}
T.~F. Gonzalez.
\newblock Clustering to minimize the maximum intercluster distance.
\newblock \emph{Theoretical Computer Science}, 38:\penalty0 293--306, 1985.

\bibitem[Haeupler et~al.(2021)Haeupler, Hershkowitz, and
  Zuzic]{haeupler2021tree}
B.~Haeupler, E.~Hershkowitz, and G.~Zuzic.
\newblock Tree embeddings for hop-constrained network design.
\newblock In \emph{Annual ACM SIGACT Symposium on Theory of Computing (STOC)},
  2021.

\bibitem[Harb and Lam(2020)]{harb2020}
E.~Harb and H.~Lam.
\newblock {KFC:} a scalable approximation algorithm for {$k$}-center fair
  clustering.
\newblock In \emph{Neural Information Processing Systems (NeurIPS)}, 2020.

\bibitem[Huang et~al.(2019)Huang, Jiang, and Vishnoi]{huang2019}
L.~Huang, S.~H.-C. Jiang, and N.~K. Vishnoi.
\newblock Coresets for clustering with fairness constraints.
\newblock In \emph{Neural Information Processing Systems (NeurIPS)}, 2019.

\bibitem[Ilvento(2019)]{ilvento2019}
C.~Ilvento.
\newblock Metric learning for individual fairness.
\newblock arXiv:1906.00250 [cs.LG], 2019.

\bibitem[Joseph et~al.(2016)Joseph, Kearns, Morgenstern, and Roth]{joseph2016}
M.~Joseph, M.~Kearns, J.~Morgenstern, and A.~Roth.
\newblock Fairness in learning: Classic and contextual bandits.
\newblock In \emph{Neural Information Processing Systems (NIPS)}, 2016.

\bibitem[Joseph et~al.(2018)Joseph, Kearns, Morgenstern, Neel, and
  Roth]{joseph2018}
M.~Joseph, M.~Kearns, J.~Morgenstern, S.~Neel, and A.~Roth.
\newblock Meritocratic fairness for infinite and contextual bandits.
\newblock In \emph{AAAI / ACM Conference on Artificial Intelligence, Ethics,
  and Society}, 2018.

\bibitem[Jung et~al.(2020)Jung, Kannan, and Lutz]{jung2020}
C.~Jung, S.~Kannan, and N.~Lutz.
\newblock A center in your neighborhood: Fairness in facility location.
\newblock In \emph{Symposium on Foundations of Responsible Computing (FORC)},
  2020.

\bibitem[Kleindessner et~al.(2019{\natexlab{a}})Kleindessner, Awasthi, and
  Morgenstern]{fair_k_center_2019}
M.~Kleindessner, P.~Awasthi, and J.~Morgenstern.
\newblock Fair $k$-center clustering for data summarization.
\newblock In \emph{International Conference on Machine Learning (ICML)},
  2019{\natexlab{a}}.

\bibitem[Kleindessner et~al.(2019{\natexlab{b}})Kleindessner, Samadi, Awasthi,
  and Morgenstern]{fair_SC_2019}
M.~Kleindessner, S.~Samadi, P.~Awasthi, and J.~Morgenstern.
\newblock Guarantees for spectral clustering with fairness constraints.
\newblock In \emph{International Conference on Machine Learning (ICML)},
  2019{\natexlab{b}}.

\bibitem[Lloyd(1982)]{lloyd_algorithm}
S.~Lloyd.
\newblock Least squares quantization in pcm.
\newblock \emph{EEE Transactions on Information Theory}, 28\penalty0
  (2):\penalty0 129--137, 1982.

\bibitem[Mahabadi and Vakilian(2020)]{mahabadi2020}
S.~Mahabadi and A.~Vakilian.
\newblock Individual fairness for $k$-clustering.
\newblock In \emph{International Conference on Machine Learning (ICML)}, 2020.

\bibitem[Makarychev and Makarychev(2016)]{makarychev2016metric}
K.~Makarychev and Y.~Makarychev.
\newblock Metric perturbation resilience.
\newblock \emph{arXiv preprint arXiv:1607.06442}, 2016.

\bibitem[Makarychev and Vakilian(2021)]{makarychev2021approximation}
Y.~Makarychev and A.~Vakilian.
\newblock Approximation algorithms for socially fair clustering.
\newblock In \emph{Conference on Learning Theory (COLT)}, 2021.

\bibitem[Micha and Shah(2020)]{micha2020proportionally}
E.~Micha and N.~Shah.
\newblock Proportionally fair clustering revisited.
\newblock In \emph{47th International Colloquium on Automata, Languages, and
  Programming (ICALP 2020)}, 2020.

\bibitem[Ostrovsky et~al.(2013)Ostrovsky, Rabani, Schulman, and
  Swamy]{ostrovsky2013effectiveness}
R.~Ostrovsky, Y.~Rabani, L.~Schulman, and C.~Swamy.
\newblock The effectiveness of {Lloyd-type} methods for the $k$-means problem.
\newblock \emph{Journal of the ACM (JACM)}, 59\penalty0 (6):\penalty0 1--22,
  2013.

\bibitem[R\"{o}sner and Schmidt(2018)]{roesner2018}
C.~R\"{o}sner and M.~Schmidt.
\newblock Privacy preserving clustering with constraints.
\newblock In \emph{International Colloquium on Automata, Languages, and
  Programming (ICALP)}, 2018.

\bibitem[Roth(1984)]{roth1984evolution}
A.~Roth.
\newblock The evolution of the labor market for medical interns and residents:
  a case study in game theory.
\newblock \emph{Journal of political Economy}, 92\penalty0 (6):\penalty0
  991--1016, 1984.

\bibitem[Schmidt et~al.(2018)Schmidt, Schwiegelshohn, and
  Sohler]{sohler_kmeans}
M.~Schmidt, C.~Schwiegelshohn, and C.~Sohler.
\newblock Fair coresets and streaming algorithms for fair k-means clustering.
\newblock arXiv:1812.10854 [cs.DS], 2018.

\bibitem[Shalev-Shwartz and Ben-David(2014)]{shalev2014understanding}
S.~Shalev-Shwartz and S.~Ben-David.
\newblock \emph{Understanding machine learning: From theory to algorithms}.
\newblock Cambridge University Press, 2014.

\bibitem[Vakilian and Yalciner(2022)]{vakilian2022improved}
A.~Vakilian and M.~Yalciner.
\newblock Improved approximation algorithms for individually fair clustering.
\newblock In \emph{International Conference on Artificial Intelligence and
  Statistics (AISTATS)}, 2022.

\bibitem[von Luxburg(2007)]{Luxburg_tutorial}
U.~von Luxburg.
\newblock A tutorial on spectral clustering.
\newblock \emph{Statistics and Computing}, 17\penalty0 (4):\penalty0 395--416,
  2007.

\bibitem[Wagstaff et~al.(2001)Wagstaff, Cardie, Rogers, and
  Schr{\"o}dl]{wagstaff2001constrained}
K.~Wagstaff, C.~Cardie, S.~Rogers, and S.~Schr{\"o}dl.
\newblock Constrained $k$-means clustering with background knowledge.
\newblock In \emph{International Conference on Machine Learning (ICML)}, 2001.

\end{thebibliography}
\bibliographystyle{plainnat}

\clearpage
\onecolumn
\appendix

\section*{Appendix}\label{appendix}

\section{Other Related Work}\label{sec:other_related_work}
\paragraph{Other notions of fairness for clustering}

Fair clustering was first studied in the seminal work of \citet{fair_clustering_Nips2017}. 
It is based on the fairness notion of disparate impact
\citep{feldman2015}, which says that the output of a 
machine learning 
algorithm should be independent of a sensitive attribute, and 
asks that
each cluster has proportional representation from different
demographic groups.
\citeauthor{fair_clustering_Nips2017}
provide approximation algorithms that incorporate
their
notion 
into
$k$-median and $k$-center  clustering, assuming that there are only two 
demographic 
groups. 
Several follow-up works extend this line of work to other 
clustering objectives such as $k$-means or spectral clustering, 
multiple or non-disjoint groups, 
some variations of the fairness notion, 
 to address scalability issues 
 or to improve approximation guarantees   \citep{roesner2018,sohler_kmeans,ahmadian2019,anagnostopoulos2019,backurs2019,bera2019,bercea2019,huang2019,fair_SC_2019,
 ahmadian_fair_hierarchical_clustering,ahmadian2020, ahmadi2020fair,esmaeili2020,harb2020}. 
The recent work of \citet{davidson2020} shows that for two groups, when given any clustering, one can efficiently compute the fair clustering 
(fair according to the notion of \citeauthor{fair_clustering_Nips2017})
that is most similar to the given clustering  
using 
linear programming.
\citeauthor{davidson2020} also show that it is NP-hard to decide whether a data set allows for a fair clustering that 
additionally 
satisfies some given must-link constraints.
They mention that such must-link constraints could be used for encoding individual level fairness constraints of the form ``similar data points must go to the same cluster''. 
However, for such a notion of individual fairness it remains unclear which pairs of data points exactly should be subject to a must-link~constraint.

Alternative fairness notions for clustering 
are
tied to centroid-based clustering such as $k$-means,  $k$-median and $k$-center,  
 where one chooses $k$ centers and then forms clusters by assigning every data point to its closest center: 
(i) Motivated 
by the 
application of 
data summarization, \citet{fair_k_center_2019} propose 
that the various demographic groups should be 
proportionally represented~among 
the chosen centers.  
(ii) \citet{chen2019} propose a notion of proportionality that requires that 
no sufficiently large subset of data points could jointly reduce 
their distances from their 
closest~centers by choosing a new center.   
This 
notion is similar to our notion of IP-stability  in that it 
assumes that an individual data point 
strives 
to be well represented (in the notion of \citeauthor{chen2019} by being close to a center) 
and it  also 
does not rely on demographic group information. 
However, while our notion aims at ensuring stability for every single data point, the notion of \citeauthor{chen2019} only looks at sufficiently large subsets.~\citet{micha2020proportionally} further studied the proportionally fair clustering in Euclidean space and on graph metrics.
(iii)~\citet{Brubach2020,brubach2021fairness} introduce the notions of pairwise fairness and community preserving fairness and incorporate them into $k$-center clustering. 
(iv)~\citet{ghadiri2021} and \citet{abbasi2021} study a fair version of $k$-means clustering where the goal is to minimize the maximum average cost that a demographic group incurs (the maximum 
is over the demographic groups and the average is over the various data points within a group).~\citet{makarychev2021approximation} and 
\citet{chlamtavc2022approximating} 
designed optimal algorithms for minimizing the global clustering cost of a clustering (e.g., $k$-median and $k$-means cost) with respect to this notion of fairness and its generalization known as clustering with cascaded norms.  

\section{Proof of Theorem~\ref{theorem_hardness}}\label{proof_hardness}
We show NP-hardness 
of the 
{\ipste} clustering
decision problem 
(with $k=2$ 
and $d$ required to be a metric) 
via a reduction from a variant of 3-SAT.   
It is well known that deciding whether a Boolean formula in conjunctive normal form, where each clause 
comprises at most three literals, is satisfiable is NP-hard. NP-hardness also holds for a restricted version of 3-SAT, where each variable 
occurs in at most three clauses \citep[][page 259]{garey_comp_and_intractability}. Furthermore, we can require the formula to have the 
same number of clauses as number of variables as the following transformation shows:
let $\Phi$ be a formula with $m$ clauses and $n$ variables. If $n>m$, we introduce $l=\lfloor\frac{n-m+1}{2}\rfloor$ new variables $x_1,\ldots,x_l$ 
and for each of them add three clauses $(x_i)$ to $\Phi$ (if $n-m$ is odd, we add only two clauses $(x_l)$). The resulting formula has the 
same number of clauses as number of variables and is satisfiable if and only if $\Phi$ is satisfiable. Similarly, if $n<m$, we introduce $l=\lfloor3\cdot\frac{m-n}{2}+\frac{1}{2}\rfloor$ 
new variables $x_1,\ldots,x_l$ and add to $\Phi$ the clauses $(x_1\vee x_2\vee x_3),(x_4\vee x_5\vee x_6),\ldots,(x_{l-2}\vee x_{l-1}\vee x_l)$ (if $m-n$ is odd, the last clause is 
$(x_{l-1}\vee x_l)$ instead of $(x_{l-2}\vee x_{l-1}\vee x_l)$). As before, the resulting formula has the 
same number of clauses as number of variables and is satisfiable if and only if $\Phi$ is satisfiable.

So let $\Phi=C_1 \wedge C_2\wedge \ldots \wedge C_n$ be a formula in conjunctive normal form over variables $x_1,\ldots,x_n$ 
such that each clause~$C_i$ comprises at most three literals 
$x_j$ or $\neg x_j$ and each variable occurs in at most three clauses (as either $x_j$ or $\neg x_j$). 
We construct a metric space $(\dataset,d)$ in time polynomial in $n$ such that $\dataset$ has an 
{\ipste} 
2-clustering with respect to $d$ if and only if $\Phi$ is satisfiable 
(for $n$ sufficiently large). 
We set 
\begin{align*}
\dataset=\{True, False,\star,\infty,C_1,\ldots,C_n,x_1,\neg x_1,\ldots,x_n,\neg x_n\} 
\end{align*}
and 
\begin{align*}
d(x,y)=\left[d'(x,y)+\indi\{x\neq y\}\right]+\indi\{x\neq y\}\cdot\max_{x,y\in\dataset}\left[d'(x,y)+1\right],\quad x,y\in\dataset,
\end{align*}
for some symmetric function $d':\dataset\times\dataset\rightarrow\R_{\geq 0}$ with $d'(x,x)=0$, $x\in\dataset$, that we specify in the next paragraph. 
It is straightforward to see that $d$ is a metric. Importantly, note that for any $x\in\dataset$, inequality \eqref{def_ip_stable_ineq} holds with respect to $d$ if and only 
if it holds with respect to $d'$.

We set $d'(x,y)=0$ for all $x,y\in\dataset$ except for the following:
\begin{align*}
d'(True,False)&=A,\\
d'(True,\star)&=B,\\
d'(\star,False)&=C,\\
d'(C_i,False)&=D,\quad i=1,\ldots,n,\\
d'(C_i,\star)&=E,\quad i=1,\ldots,n,\\
d'(\infty,True)&=F,\\
d'(\infty,False)&=G,\\
d'(\infty,\star)&=H,\\
d'(C_i,\infty)&=J,\quad i=1,\ldots,n,\\
d'(x_i,\neg x_i)&=S,\quad i=1,\ldots,n,\\
d'(C_i,\neg x_j)&=U, \quad  (i,j)\in\{(i,j)\in\{1,\ldots,n\}^2:x_j\text{~appears in~}C_i\},\\
d'(C_i, x_j)&=U, \quad (i,j)\in\{(i,j)\in\{1,\ldots,n\}^2:\neg x_j\text{~appears in~}C_i\},
\end{align*}
where we set
\begin{align}\label{choice_numbers}
\begin{split}
 A &=  n, \qquad F = n^2, \qquad B = 2F=2n^2, \qquad E = \frac{5}{2}F = \frac{5}{2} n^2, \\
 J &= E+\log n=\frac{5}{2} n^2+\log n, \qquad D = J + \log^2 n=\frac{5}{2} n^2+\log n+\log^2 n,\\
 U &=3J=\frac{15}{2} n^2+3\log n, \qquad H= n D +E=\frac{5}{2} n^3+\frac{5}{2} n^2+n\log n+n\log^2 n,\\
G &= H+2n^2-n-2n\log^2 n=\frac{5}{2} n^3+\frac{9}{2} n^2-n+n\log n-n\log^2 n,\\
S &= (3n+3)U=\frac{45}{2} n^3+\frac{45}{2} n^2+9n\log n+9\log n,\\ 
C &= \frac{A + G + n  D}{2}=\frac{5}{2} n^3+\frac{9}{4} n^2+n\log n.
\end{split}
\end{align}

We 
show that 
for $n\geq 160$ 
there 
is 
a satisfying assignment 
for $\Phi$
if and only if there 
is 
an 
{\ipste}
2-clustering of  $\dataset$.

\begin{itemize}[leftmargin=*]
\item  \emph{``Satisfying assignment $\Rightarrow$ 
{\ipste}
2-clustering''}

Let us assume we are given a satisfying assignment of $\Phi$. We may assume that if $x_i$ only appears as $x_i$ in $\Phi$ and not as $\neg x_i$, then $x_i$ is true; 
similarly, if $x_i$ only appears as $\neg x_i$, then $x_i$ is false.
We construct a clustering of $\dataset$ into two clusters $V_1$ and $V_2$ as follows:
\begin{align*}
 V_1&=\{True,\infty,C_1,\ldots,C_n\}\cup\{x_i: x_i\text{~is true in sat. ass.}\}\cup\{\neg x_i: \neg x_i\text{~is true in sat. ass.}\},\\
 V_2&=\{False,\star\}\cup\{x_i: x_i\text{~is false in satisfying assignment}\}\cup\{\neg x_i: \neg x_i\text{~is false in sat. ass.}\}.
\end{align*}
It is $|V_1|=2+2n$ and $|V_2|=2+n$. We need show that every data point in $\dataset$ is stable.
This is equivalent to verifying that the following inequalities are true:

\vspace{2mm}
Points in $V_1$:
\begin{align}\label{cond_points_X1}
 True:~~~~&\frac{1}{1+2n}\sum_{v\in V_1}d'(True,v)=\frac{F}{1+2n}\leq\frac{A+B}{2+n}=\frac{1}{2+n}\sum_{v\in V_2}d'(True,v)\\
\infty:~~~~&\frac{1}{1+2n}\sum_{v\in V_1}d'(\infty,v)=\frac{F+nJ}{1+2n}\leq \frac{G+H}{2+n}= \frac{1}{2+n}\sum_{v\in V_2}d'(\infty,v)\\
C_i:~~~~& \frac{1}{1+2n}\sum_{v\in V_1}d'(C_i,v)\leq\frac{J+2U}{1+2n}\leq \frac{U+D+E}{2+n}\leq \frac{1}{2+n}\sum_{v\in V_2}d'(C_i,v)\\
x_i:~~~~&\frac{1}{1+2n}\sum_{v\in V_1}d'(x_i,v)\leq\frac{2U}{1+2n}\leq \frac{S}{2+n}\leq \frac{1}{2+n}\sum_{v\in V_2}d'(x_i,v)\\
\neg  x_i:~~~~&\frac{1}{1+2n}\sum_{v\in V_1}d'(\neg x_i,v)\leq\frac{2U}{1+2n}\leq \frac{S}{2+n}\leq \frac{1}{2+n}\sum_{v\in V_2}d'(\neg x_i,v)\label{cond_points_X1_end}
\end{align}

\vspace{2mm}
Points in $V_2$:
\begin{align}
False:~~~~&\frac{1}{1+n}\sum_{v\in V_2}d'(False,v)=\frac{C}{1+n}\leq\frac{A+G+nD}{2+2n}=\frac{1}{2+2n}\sum_{v\in V_1}d'(False,v)\label{cond_False_is_fair}\\
\star:~~~~&\frac{1}{1+n}\sum_{v\in V_2}d'(\star,v)=\frac{C}{1+n}\leq\frac{B+H+nE}{2+2n}=\frac{1}{2+2n}\sum_{v\in V_1}d'(\star,v)\\
x_i:~~~~&\frac{1}{1+n}\sum_{v\in V_2}d'(x_i,v)=0\leq \frac{S}{2+2n}\leq \frac{1}{2+2n}\sum_{v\in V_1}d'(x_i,v)\\
\neg  x_i:~~~~&\frac{1}{1+n}\sum_{v\in V_2}d'(\neg x_i,v)=0\leq \frac{S}{2+2n}\leq \frac{1}{2+2n}\sum_{v\in V_1}d'(\neg x_i,v)\label{cond_points_X2_end}
\end{align}

It is straightforward to check that for our choice of $A,B,C,D,E,F,G,H,J,S,U$ as specified 
in \eqref{choice_numbers} all inequalities~\eqref{cond_points_X1} to \eqref{cond_points_X2_end} are true.

\vspace{2mm}
\item \emph{``{\ipste} 2-clustering $\Rightarrow$ satisfying assignment''}

Let us assume that there is an {\ipste} clustering of $\dataset$ with two clusters $V_1$ and $V_2$.
For any partitioning of $\{C_1,\ldots,C_n\}$ into two sets of size $l$ and $n-l$ ($0\leq l\leq n$) we denote the two sets by $\mathcal{C}_l$ and $\widetilde{\mathcal{C}}_{n-l}$.

We first show that  $x_i$ and $\neg x_i$ cannot be contained in the same cluster (say in $V_1$). This is because if we assume that $x_i,\neg x_i \in V_1$, 
for our choice of $S$ and $U$ in \eqref{choice_numbers} we have
\begin{align*}
\frac{1}{|V_2|}\sum_{v\in V_2}d'(x_i,v)\leq U< \frac{S}{3n+2}\leq \frac{1}{|V_1|-1}\sum_{v\in V_1}d'(x_i,v)
\end{align*}
in contradiction to $x_i$ being 
stable.
As a consequence we have $n\leq |V_1|,|V_2|\leq 2n+4$.

Next, we show that due to our choice of $A,B,C,D,E,F,G,H,J$ in \eqref{choice_numbers} none of the following cases can be true:

\vspace{1mm}
\begin{enumerate}
\setlength\itemsep{5mm}

 \item $\{True,\infty\}\cup \mathcal{C}_l\subset V_1$ and $\widetilde{\mathcal{C}}_{n-l}\cup\{False,\star\}\subset V_2$ 
 for any $0\leq l< n$

 \vspace{2.4mm}
In this case, $False$ would not be 
stable since for all $0\leq l< n$,
\begin{align*}
\frac{1}{|V_1|}\sum_{v\in V_1}d'(False,v)= \frac{A+G+lD}{l+2+n}<\frac{C+(n-l)D}{n-l+1+n}=\frac{1}{|V_2|-1}\sum_{v\in V_2}d'(False,v).
\end{align*}

 \item $\{True\}\cup \mathcal{C}_l\subset V_1$ and $\widetilde{\mathcal{C}}_{n-l}\cup\{False,\star,\infty\}\subset V_2$ for any $0\leq l\leq n$

 \vspace{2.4mm}
 In this case, $False$ would not be 
 stable since for all $0\leq l\leq n$,
\begin{align*}
 \frac{1}{|V_1|}\sum_{v\in V_1}d'(False,v)=\frac{A+lD}{l+1+n}<\frac{C+G+(n-l)D}{n-l+2+n}=\frac{1}{|V_2|-1}\sum_{v\in V_2}d'(False,v). 
\end{align*}

\item $\{False,\infty\}\cup\mathcal{C}_l\subset V_1$ and $\widetilde{\mathcal{C}}_{n-l}\cup\{True,\star\}\subset V_2$ for any $0\leq l\leq n$

\vspace{2.4mm}
In this case, $True$ would not be 
stable since for all $0\leq l\leq  n$,
 \begin{align*}
\frac{1}{|V_1|}\sum_{v\in V_1}d'(True,v)= \frac{A+F}{l+2+n}<\frac{B}{n-l+1+n}=\frac{1}{|V_2|-1}\sum_{v\in V_2}d'(True,v).
\end{align*}

\item $\{False\}\cup\mathcal{C}_l\subset V_1$ and $\widetilde{\mathcal{C}}_{n-l}\cup\{True,\star,\infty\}\subset V_2$ for any $0\leq l\leq n$

\vspace{2.4mm}
In this case, $True$ would not be 
stable since for all $0\leq l\leq  n$, 
\begin{align*}
 \frac{1}{|V_1|}\sum_{v\in V_1}d'(True,v)=\frac{A}{l+1+n}<\frac{B+F}{n-l+2+n}=\frac{1}{|V_2|-1}\sum_{v\in V_2}d'(True,v).
\end{align*}

\item $\{\star,\infty\}\cup\mathcal{C}_l\subset V_1$ and $\widetilde{\mathcal{C}}_{n-l}\cup\{False,True\}\subset V_2$ for any $0\leq l\leq n$

\vspace{2.4mm}
In this case, $\star$ would not be 
stable since for all $0\leq l\leq  n$,
\begin{align*}
\frac{1}{|V_2|}\sum_{v\in V_2}d'(\star,v)= \frac{B+C+(n-l)E}{n-l+2+n}<\frac{H+lE}{l+1+n}=\frac{1}{|V_1|-1}\sum_{v\in V_1}d'(\star,v). 
\end{align*}

\item $\{\star\}\cup\mathcal{C}_l\subset V_1$ and $\widetilde{\mathcal{C}}_{n-l}\cup\{False,True,\infty\}\subset V_2$ for any $0\leq l\leq n$

\vspace{2.4mm}
In this case, $\infty$ would not be 
stable since for all $0\leq l\leq n$,
\begin{align*}
\frac{1}{|V_1|}\sum_{v\in V_1}d'(\infty,v)=\frac{H+lJ}{l+1+n}<\frac{F+G+(n-l)J}{n-l+2+n}=\frac{1}{|V_2|-1}\sum_{v\in V_2}d'(\infty,v).
\end{align*}

\item $\mathcal{C}_l\subseteq V_1$ and $\widetilde{\mathcal{C}}_{n-l}\cup\{True,False,\star,\infty\}\subseteq V_2$ for any $0\leq l\leq n$
 
 \vspace{2.4mm}
 In this case, $True$ would not be 
 stable since for all $0\leq l\leq  n$,
\begin{align*}
  \frac{1}{|V_1|}\sum_{v\in V_1}d'(True,v)=0<\frac{A+B+F}{3+(n-l)+n}=\frac{1}{|V_2|-1}\sum_{v\in V_2}d'(True,v).  
\end{align*}

\item $\{\infty\}\cup\mathcal{C}_l\subseteq V_1$ and $\widetilde{\mathcal{C}}_{n-l}\cup\{True,False,\star\}\subseteq V_2$ for any $0\leq l\leq n$

\vspace{2.4mm}
In this case, $True$ would not be 
stable since for all $0\leq l\leq  n$,
\begin{align*}
  \frac{1}{|V_1|}\sum_{v\in V_1}d'(True,v)=\frac{F}{1+l+n}<\frac{A+B}{2+(n-l)+n}=\frac{1}{|V_2|-1}\sum_{v\in V_2}d'(True,v).
\end{align*}

\end{enumerate}

\vspace{2mm}
Of course, in all these cases we can exchange the role of $V_1$ and $V_2$. Hence, $True,\infty, C_1,\ldots,C_n$ must be contained in one cluster and $\star,False$ must be 
contained in the other cluster. W.l.o.g., let us assume $True,\infty, C_1,\ldots,C_n\in V_1$ and  $\star,False\in V_2$ and hence 
$|V_1|=2n+2$ and $|V_2|=n+2$.

Finally, we show that for the clause $C_i=(l_j)$ or $C_i=(l_j\vee l_{j'})$ or $C_i=(l_j\vee l_{j'} \vee l_{j''})$, with the literal $l_j$ equaling $x_j$ or $\neg x_j$,  
it cannot be the case that $C_i,\neg l_j$ or  $C_i,\neg l_j$, $\neg l_{j'}$ or  $C_i,\neg l_j$, $\neg l_{j'}$, $\neg l_{j''}$ are all contained in $V_1$. 
This is because otherwise
\begin{align}\label{wider8}
 \frac{1}{|V_2|}\sum_{v\in V_2}d'(C_i,v)=\frac{D+E}{n+2}<\frac{U+J}{2n+1}\leq\frac{1}{|V_1|-1}\sum_{v\in V_1}d'(C_i,v)
\end{align}
for our choice of $D,E,J,U$ in \eqref{choice_numbers} and $C_i$ would not be 
stable. Consequently, since $x_j$ and $\neg x_j$ are not in the same cluster, 
for each clause $C_i$ at least one of its literals must be in~$V_1$. 

Hence, if we set every literal $x_i$ or $\neg x_i$ that is contained in $V_1$ to a true logical value 
and every literal $x_i$ or $\neg x_i$ that is contained in~$V_2$ to a false logical value, we obtain a valid assignment 
that makes $\Phi$ true.  
\hfill $\square$
\end{itemize}

\section{Missing Proofs from \Cref{sec:approx-stability}}
\label{sec:appendix-proof-approx}

\subsection{Proof of \Cref{clm:hst-extend}}
\begin{proof}
    The idea is to extend nodes corresponding to actual points in $V$ so that they are the leaves and that they are at the same depth. 
    The distortion of the modified tree will be worse by a constant factor, but that is fine for our purpose as the construction has $\mathcal O(polylog(n))$ distortion.
    Let $(V',d_T)$ be a given partial tree metric from the construction in \Cref{thm:main_partial_tree} and let $\ell = \depth(d_T)$ be the depth of the tree. 
    If a node~$v \in V$ is not a leaf of $T$, 
    we augment $T$ by replacing $v$ with $v'$ and connecting $v'$ to $v$ with a path  $v', v_1, v_2, \ldots, v_{\ell-\depth_T(v) -1}, v$.
    Let $T'$ be the tree after our modification.
    By this augmentation, $v$ is now a leaf in $T'$ and $\depth_{T'}(v) = \ell - 1$.
    Let $(V'',d_{T'})$ be the new tree metric.
    For any $u\in V$, we show that $d_{T'}(u,v) \leq 3 d_{T}(u,v)$.
    In words, this is because the weight of the added path between $v'$ and $v$ is a telescopic sum, and this sums up to at most the weight of the edge connecting $v$ to its parent in $T$, i.e. $p_T(v)$.
    That is,
    $d_{T'}(u,v) = d_{T'}(u,v') + d_{T'}(v',v) = d_{T}(u,v) + d_{T'}(v',v) \leq d_{T}(u,v) + d_{T'}(p_{T'}(v'),v') = d_{T}(u,v) + d_{T}(p_{T}(v),v)\leq  3d_{T}(u,v)$. 
    Since we can apply this argument to any node, the claim holds true.
\end{proof}

\subsection{Proof of \Cref{claim:main_clustering_hst}}
\begin{proof}
    Let $T$ be our HST tree rooted at $r$. Let $S = \{v_1, v_2 \ldots, v_{k}\}$ be a set of nodes that we are going to pick. Suppose they are sorted in such a way that for every $1\leq i<k$, $\depth(v_i) \geq \depth(v_{i+1})$. We let $\depth(r)=0$, and for every node~$u$, $\depth(u) = \depth(p(u)) + 1$ where $p(u)$ is the parent of $u$ in $T$. 
    
    To find $S$, consider the lowest depth, i.e., furthest from the root, $\ell$ such that the number of nodes at depth $\ell$ is at most $k$. Let $V_\ell$ be the set of nodes at depth $\ell$. Initially, $S = \emptyset$.
    For each $v \in V_\ell$, if $|S| + |\textrm{children}(v)| + |V_\ell| - 1 < k$, then add $\textrm{children}(v)$ to $S$ and remove $v$ from $V_\ell$ and continue.
    Otherwise, if $|S| + |\textrm{children}(v)| + |V_\ell| - 1 = k$, then add $\textrm{children}(v)$ and $V_\ell - \{v\}$ to $S$. Since $|S| = k$ after this operation, we are done. In the last case where $|S| + |\textrm{children}(v)| + |V_\ell| - 1 > k$, select a subset of $\textrm{children}(v)$ of size $ k - |S| + |V_\ell|$. Let this subset be $V_v$. Let $S = S \cup V_\ell \cup V_v$. 
    An example is given in Figure~\ref{fig:FRT-stable-clustering} for $k=4$. 
    
\begin{figure}[!ht]
    \centering
    \includegraphics[width=0.6\textwidth]{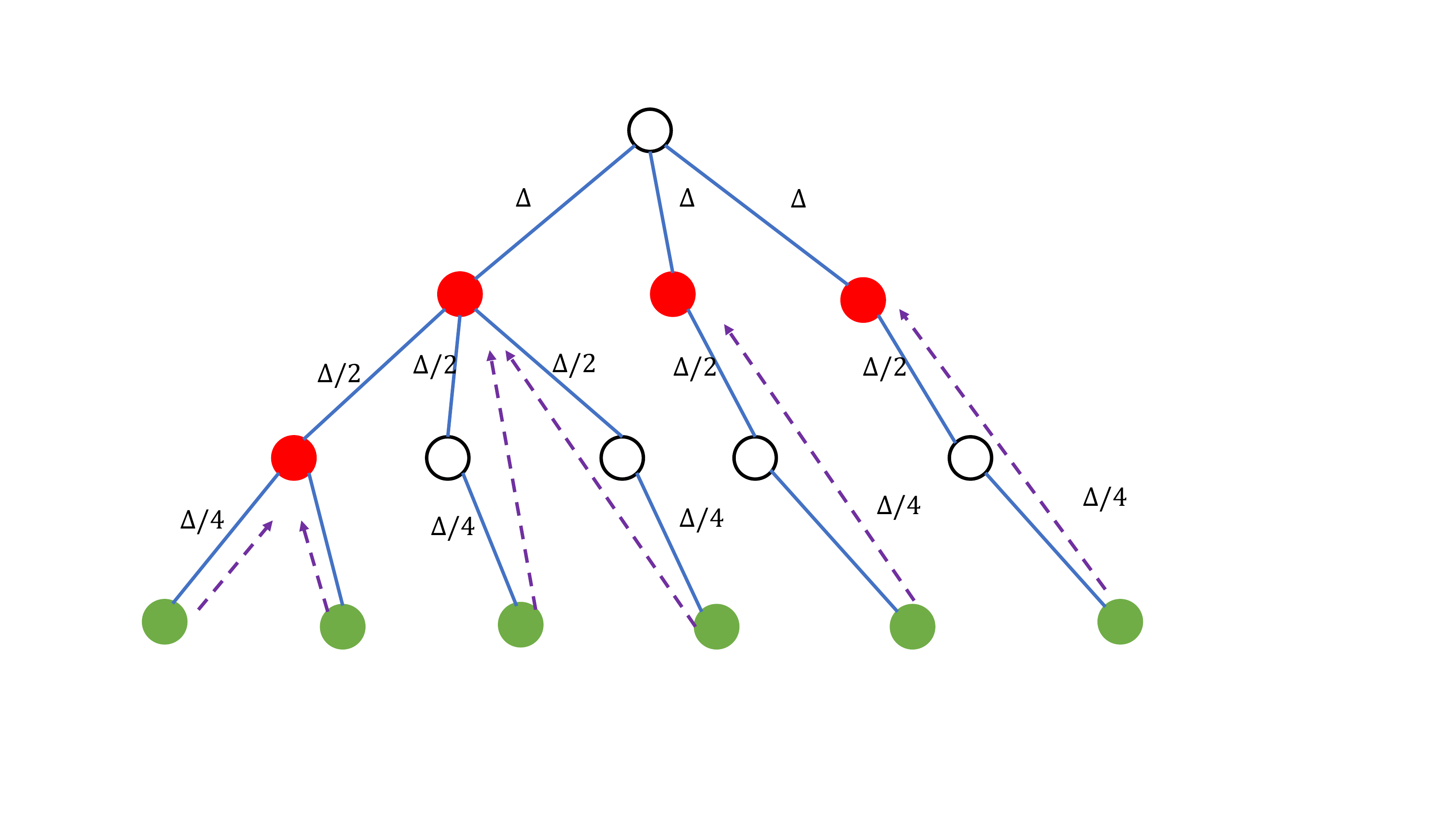}
    \caption{IP-stable $4$-clustering on a $2$-HST. $S$ is the set of the red nodes.}\label{fig:FRT-stable-clustering}
\end{figure}


    For $i \leq k$, let $T_i$ to be the subtree of $T \setminus \left ( \bigcup_{j<i} T_j \right )$ rooted at $v_i$.  
    Let $C_i$ denote the $i^{th}$ cluster consisting of the leaves of $T_i$.
    We argue that $\mathcal C = (C_1, C_2, \ldots, C_k)$ is an IP-stable clustering of $T$.
    Let $u,v$ be two leaves in $T$. 
    Let $x$
    be the lowest common ancestor (lca) of $u$ and $v$. Since $T$ is an HST tree, and since all leaves are at the same depth, $d_T(u,v) = d_T(u,x) + d_T(x,v) = 2 d(u,x)$. 
    For any pair of $C_i\in \mathcal{C}, u\in C_i$ where $u$ is a leaf, as we show above, the distance from $u$ to any other leaf node~$v\in C_i$ is at most $2 d_T(u,v_i)$. This is because the lca of $u$ and $v$ can either be $v_i$ or its descendant. 
    Hence, we can say that $
    \frac{1}{|C_i| -1} \sum_{w \in C_i}d_T(u, w) \leq 2 d_T(u, v_i)$. 
    Moreover, for any leaf node~$v' \in C_{j \neq i}$,
    we have that 
    $d_T(u, v') \geq 2 d_T(u, v_i)$.
    This is because the lca of $u$ and $v'$ can either be $v_i$ or its ancestor.
    Hence, $\frac{1}{|C_j|}\sum_{w \in C_j}  d_T(u, w) \geq 2 d_T(u, v_i)$. 
    Therefore, $\frac{1}{|C_i| -1} \sum_{w \in C_i}d(u, w)\leq \frac{1}{|C_j|}\sum_{w \in C_j}  d_T(u, w)$, and $\mathcal{C}$ is a stable clustering for $T$.
    
    We end the proof by noting that for any $u,v \in C_i$ and $w \in C_j$, $d_T(u,v) \leq 2 d_T(u,v_i) \leq d_T(u, v_j) \leq d_T(u,w) $.
\end{proof}

\subsection{Proof of \Cref{clm:main_good_edge_property}}
\begin{proof}
    \begin{align*}
        d(x,y)
        &\leq \frac{2}{\gamma-1} \ol{d}(y,C_j) & \text{\big (By (2) in \Cref{lem:main_danielystructure} \big)} \\
        &\leq \frac{\gamma-1}{\gamma} \ol{d}(y,C_j) & \big (\gamma \geq 2+\sqrt{3} \big )\\
        &\leq d(y,z) & \text{\big (By (1) in \Cref{lem:main_danielystructure}\big)}
    \end{align*}
\end{proof}

\subsection{Proof of \Cref{thm:main_exact_bf}}
\begin{proof}
Let us consider a modification to \emph{single-linkage} algorithm:
We consider edges in non-decreasing order and only merge two clusters if one (or both) size is consisting of at most $\alpha n$ points.
\Cref{clm:main_good_edge_property} suggests that the edges within an underlying cluster will be considered before edges between two clusters. Since we know that any cluster size is at least $\alpha n$, when considering an edge, it is safe to use this condition to merge them. If it is not the case, we can ignore the edge. 

By this process, we will end up with a clustering where each cluster has size at least $\alpha n$. There are at most $\mathcal O(1/\alpha)$ such clusters. We then can enumerate over all possible clusterings, as there are at most $\mathcal O( \frac{1}{\alpha^k})$ such clusterings, the running time follows.
\end{proof}

\subsection{Proof of \Cref{clm:main_good_edge_property_2}}
\begin{proof}
From \Cref{lem:main_danielystructure}, $\frac{\gamma-1}{\gamma} \ol{d}(x,C_j) \leq d(x,y),d(x,y') \leq \frac{\gamma^2+1}{\gamma(\gamma-1)} \ol{d}(x,C_j)$. 
Hence, $$\frac{d(x,y)}{d(x,y')} \leq \frac{\gamma^2+1}{\gamma(\gamma-1)} \cdot \frac{\gamma}{\gamma-1} = \frac{\gamma^2+1}{(\gamma-1)^2}.$$
\end{proof}

\subsection{Proof of \Cref{clm:main_merge_criteria}}
\begin{proof}
    $(1)$ follows from \Cref{clm:main_good_edge_property} and the fact that we consider edges in non-decreasing order, hence, when we consider $e$, if $D,D'$ belong to different underlying clusters, and if $|D| < \alpha n$, then an edge $e'$ within the cluster $C_D$ must already be considered. At that point, $D$ should be merged into another cluster, hence a contradiction.
    
    $(2)$ follows from \Cref{col:main_bounded_length}. If $\frac{\max_{x\in D, y\in D'} d(x,y)}{\min_{x\in D, y\in D'} d(x,y) } > {\left ( \frac{\gamma^2+1}{(\gamma-1)^2} \right )}^2$, then there exists $x,x' \in D$ and $y,y' \in D'$ that violate the inequality in \Cref{col:main_bounded_length}, so $D$ and $D'$ must belong to the same underlying cluster.
    
    It remains to show $(3)$. Suppose $D$ belongs to the underlying cluster $C_D$ and $D'$ belongs to another underlying cluster $C_{D'}$, then
    $d(x,y) \geq \frac{\gamma-1}{\gamma} \ol{d}(x,C_{D'})$ ($(1)$ in \Cref{lem:main_danielystructure}). However,
    \begin{align*} 
        d(x,x') &> \frac{2 \gamma}{(\gamma-1)^2} d(x,y) 
        > \frac{2 \gamma}{(\gamma-1)^2}  \cdot \frac{\gamma-1}{\gamma} \ol{d}(x,C_{D'})\\
        &= \frac{2}{\gamma-1} \ol{d}(x,C_{D'}).
    \end{align*}
    Since we know that $x,x'$ belong to the same underlying cluster, this is a contradiction.
\end{proof} 
\section{Missing Proofs of Section~\ref{sec:special-metrics}}

\subsection{Proof of Theorem~\ref{thm:1d-stable}}
\label{proof:thm-interval-clustering}
\begin{proof}
Initially, start with the leftmost cluster containing $n-k+1$ points, and all remaining $k-1$ clusters having one single point. Imagine clusters are separated with some separators. An example is given in~\Cref{fig:initialization_interval_clustering} for $k=4$. We show that by only moving the separators to the left, an IP-stable $k$-clustering can be found.
\begin{figure}[!ht]
    \centering
    \includegraphics[width=0.5\textwidth]{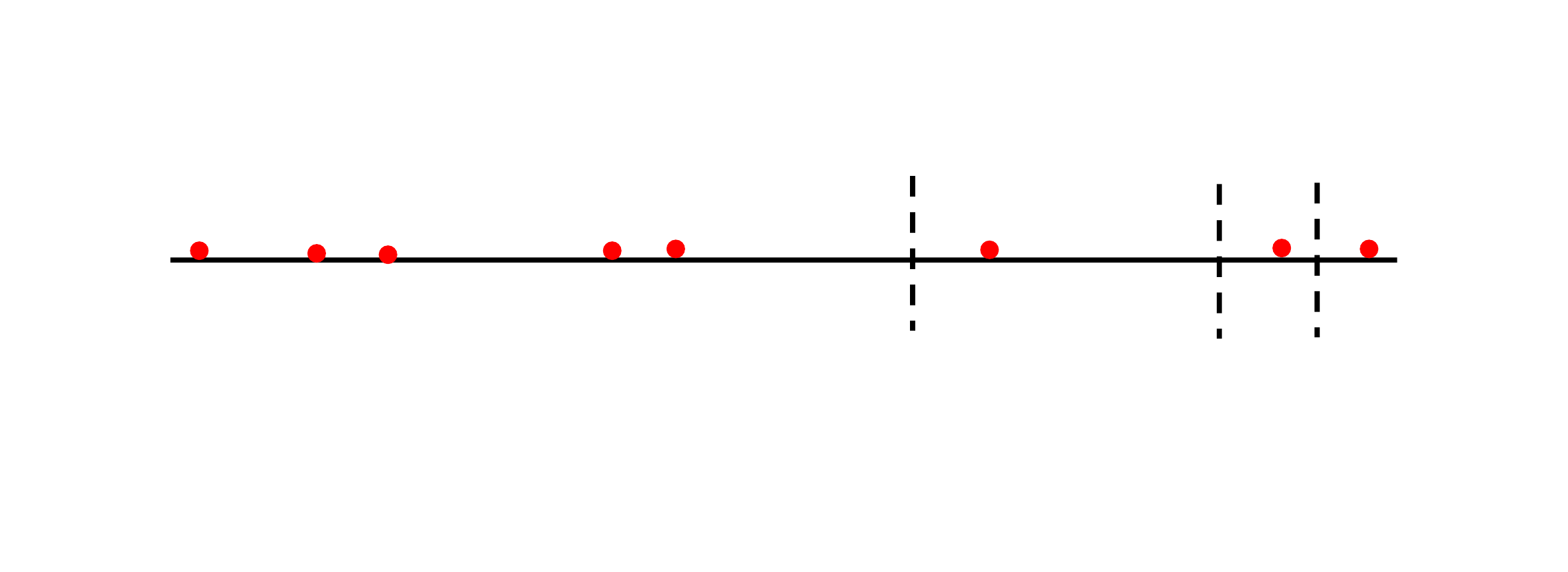}
    \vspace{-0.5in}
    \caption{Initialization for $k=4$ and $n=8$.}
    \label{fig:initialization_interval_clustering}
\end{figure}
In this proof, we mainly argue about the stability of boundary nodes of clusters; however, Lemma~\ref{lemma_boundary_points} in Appendix~\ref{appendix_p_equals_infty} shows that stability of boundary nodes implies stability of all the nodes.
Since all the clusters but the first one are singletons, and singleton clusters are {\ipste}, the only separator that might want to move is the first separator. 
The first separator moves to the left until the first cluster becomes {\ipste}. 
In the next step, the only cluster that is not {\ipste} is the second cluster since the first cluster has just been made {\ipste}, and the rest are singletons. The second cluster is unstable due to some nodes on its right hand side that want to join the third cluster by moving the second separator to the left (note that the leftmost node of the second cluster is stable since the first separator has just moved).
If the second cluster shrinks, some rightmost nodes of the first cluster might get unstable, and want to join the second cluster by moving the first separator to the left. The same scenario repeats over and over again until the first two clusters get stable by moving the first two separators only to the left. Using the same argument, at each step $i$, the only cluster that is not {\ipste} is the $i^{th}$ cluster, and the first $i$ clusters can be made stable by moving the first $i$ separators only to the left. 

Consider the last, i.e. $k^{th}$ step. If the $k^{th}$ cluster is a singleton, then it is {\ipste}, and our algorithm has converged to an {\ipste} $k$-clustering.
Assume the $k^{th}$ cluster is not a singleton. The reason that it is not a singleton is that the points on its left hand side had preferred to move from the $(k-1)^{th}$ cluster to the $k^{th}$ cluster. If the $k^{th}$ cluster is a better cluster for its leftmost point, so is a better cluster for all other points inside it. Therefore, the $k^{th}$ cluster is an {\ipste} cluster, and our algorithm converges to an {\ipste} $k$-clustering.

By using the fact that if a separator moves it only moves to the left, and therefore, each separator moves at most $n$ times, in Lemma~\ref{lem:running-time-1d} in Appendix~\ref{Appendix-running-time-1d}, 
we show the running time of the algorithm is $\mathcal{O}(kn)$. 
\end{proof}

\subsection{Running Time of the Algorithm for $1$-dimensional Euclidean Metric}
\label{Appendix-running-time-1d}

\begin{lemma}
\label{lem:running-time-1d}
The running time of the algorithm described for the $1$-d Euclidean metric case is $\mathcal{O}(kn)$.
\end{lemma}
\begin{proof}
First, we argue that if the following information are maintained for each cluster, the stability of each boundary node can be checked in $\mathcal{O}(1)$.
\begin{itemize}
 \item sum of distances from the left-most and right-most nodes to all other nodes in the cluster
 \item the number of nodes in the cluster
\end{itemize}

Suppose that $u$ is the right-most node of a cluster $C_i$. In order to check if $u$ is stable, we compare the average distance from $u$ to its own cluster ($C_i$) with the average distance from $u$ to the adjacent cluster, say ($C_{i+1}$). Let $v$ be the node adjacent to $u$ in $C_{i+1}$, then $\overline{d}(u, C_{i+1}) = d(u,v) + \overline{d}(v, C_{i+1})$. Since $d(u,C_i), d(v,C_{i+1}), |C_i|, |C_{i+1}|$ are all maintained, we can check if $\overline{d}(u,C_i) \leq \overline{d}(u,C_{i+1})$ in $\mathcal O(1)$. The argument for the case that $u$ is the leftmost node of a cluster $C_i$ holds in a similar way.
At the beginning of the algorithm, since our initialization has only one big cluster, coming up with the needed information for two boundary nodes of the big cluster takes $\mathcal O(n)$. Since the only operation that our algorithm does is to move a separator to the left, we show whenever such an operation happens, all the information stored that need to get updated, can get updated in $\mathcal{O}(1)$.
Suppose when a separator moves to the left, the right-most node of $C_i$, let's say $u$, moves to $C_{i+1}$.
Let $s,t,v,w$ be the left-most node of $C_i$, the node next to $u$ in $C_i$, the node next to $u$ in $C_{i+1}$, and the right-most node of $C_{i+1}$ before the move, respectively.
Also, let $\tilde C_i = C_i \setminus \{u\}$ and $\tilde C_{i+1} = C_{i+1} \cup \{u\}$ denote the updated clusters $C_i$ and $C_{i+1}$ after $u$ joins the cluster on its right.
As $|\tilde C_i| = |C_i| -1$ and $|\tilde C_{i+1}| = |C_{i+1}| + 1$, updating the numbers of nodes in two clusters is trivial. Now we show how to update other information, namely, $d(s,\tilde C_i), d(t, \tilde C_i), d(u, \tilde C_{i+1}), d(w, \tilde C_{i+1})$. 
\begin{align*}
    d(s,\tilde C_i) & = d(s, C_i) - d(s,u) \\
    d(t, \tilde C_i) & = d(u, C_i) - |\tilde C_i| d(t,u) \\
    d(u, \tilde C_{i+1}) & = d(u, C_{i+1}) = d(v,C_{i+1}) + |\tilde C_{i+1}|d(u,v) \\
    d(w, \tilde C_{i+1}) & = d(w, C_{i+1}) + d(u,w)
\end{align*}
 
As the right-hand sides of the relationships above are maintained each update can be performed in $\mathcal O(1)$.
Since in the algorithm each separator only moves to the left, the total number of times that separators are moved is at most $kn$. Each time that a separator wants to move, the stability of the node on its right is checked in $\mathcal{O}(1)$. Also, if a separator moves, updating the information stored takes $\mathcal{O}(1)$. Hence, the total time complexity of the algorithm is $\mathcal{O}(kn)$.
\end{proof}

\subsection{Proof of \Cref{thm:boundary-nodes-stability-suffices}}
\label{appendix:proof-boundary-nodes-stability-suffices}

\begin{proof}
    Let $C_1,C_2$ be the clusters containing $u$ and $v$, respectively.
    Wlog, suppose $|C_1| > 1$ and $x \in C_1$. Our goal is to show that
    $\overline{d}(x,C_1) \leq \overline{d}(x,C_2)$.
    
    Since the path from $x$ to any $y \in C_2$ has to go through $u$, we have that $\overline{d}(x,C_2) = d(x,u) + \overline{d}(u,C_2)$.
    
    As $u$ is stable, $\overline{d}(u,C_1) \leq \overline{d}(u,C_2)$.
    
    If we can show that $\overline{d}(x,C_1) \leq d(x,u) + \overline{d}(u,C_1)$, then we are done. This is true as
    \begin{align*}
        \overline{d}(x,C_1)
        &=\frac{1}{|C_1|-1} \sum_{y \in C_1, y\neq x } d(x,y) \\
        &=\frac{1}{|C_1|-1} \sum_{y \in C_1, y \neq x } d(x,y) \\
        &\leq \frac{1}{|C_1|-1} \sum_{y \in C_1, y \neq x }\left( d(x,u) + d(u,y) \right)\\
        &=d(x,u) + \frac{1}{|C_1|-1} \sum_{y \in C_1, y \neq x }  d(u,y) \\
        &\leq d(x,u) + \frac{1}{|C_1|-1} \sum_{y \in C_1  }  d(u,y) \\
        &= d(x,u) + \overline{d}(u,C_1).
    \end{align*}
    
\end{proof}
\section{Hard Instances for $k$-means++, $k$-center, and single linkage}
\label{sec:hard_instances}
In this section, we prove Theorem~\ref{thm:hard_instances} by showing hard instances for $k$-means++, $k$-center, and single linkage clustering algorithms.

\subsection{Hard Instances for $k$-means++}
\label{sec:kmeans++}
Here, we show that for any given approximation factor $\alpha>1$, there exists  an instance $\sI_\alpha$ and a target number of clusters~$k_\alpha$ such that with constant probability $k$-means++ outputs a $k_\alpha$-clustering of $\sI_\alpha$ that violates the requirement of {\ipste} clustering by a factor of $\alpha$.

\paragraph{High-level description of $k$-means++.} $k$-means++ is an algorithm for choosing the initial seeds for the Lloyd's heuristic that provides a provable approximation guarantee~\citep{kmeans_plusplus}. The high-level intuition is to spread out the initial set of $k$ centers. The first center is picked uniformly at random, after which each subsequent center is picked from the remaining set points according to the probability distribution proportional to the squared distance of the points to the so-far-selected set of centers. 
Once the seeding phase picks $k$ centers, $k$-means++ performs Lloyd's heuristic.   

Before stating the argument formally, we show how to construct the hard instance $\sI_\alpha$ for any given approximation parameter~$\alpha>1$ for the IP-stability requirement. 
\paragraph{Structure of the hard instance.} Given the approximation parameter $\alpha$, $\sI_\alpha$ is constructed as follows. The instance consists of $n_\alpha$ copies of block $I$ such that each adjacent blocks are at distance $D_\alpha$ from each other; see Figure~\ref{fig:hard_instance}. As we explain in more detail, the solution $\sol$ returned by $k$-means++ violates the IP-stability by a factor of $\alpha$ if there exists a block~$I_j\in \sI_\alpha$ such that $\sol\cap I_j = \{v_j, u_j\}$. In the rest of this section, our goal is to show that with constant probability this event happens when $k_\alpha =\frac{13}{12}n_{\alpha}$.
\begin{figure}[h!]
 \centering
 \includegraphics[width = 0.7\textwidth]{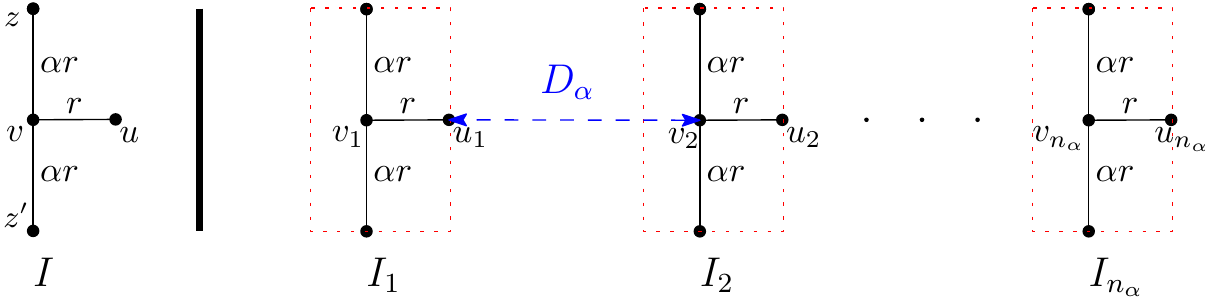}
 \caption{The construction of hard instance for a given parameter $\alpha$. (a) shows the structure of a building block $I$ and (b) shows a hard instance which constitutes of $n_\alpha$ blocks $I_1, \ldots, I_{n_\alpha}$.}\label{fig:hard_instance}
\end{figure}
\begin{claim}\label{clm:unstable-factor}
Suppose that $v$ and $u$ are picked as the initial set of centers for $2$-clustering of $I$. Then, the solution returned by $k$-means++ violates the IP-stability requirement by a factor of $\alpha$.  
\end{claim}
\begin{proof}
It is straightforward to verify that once $v$ and $u$ are picked as the initial set of centers, after a round of Lloyd's heuristic $v$ and $u$ remain the cluster centers and $k$-means++ stops. This is the case since when $v$ and $u$ are centers initially the clusters are $\{z, z', v\}$ and $\{u\}$. The centroids of these two clusters are respectively $v$ and $u$. Hence, $k$-means++ outputs $\{z, z', v\}$ and $\{u\}$ as the $2$-clustering of $I$.
The maximum violation of the {\ipste} requirement for this $2$-clustering corresponds to $v$ which is equal to $\frac{(d(v, z) + d(v, z'))/2}{d(v,u)} = \alpha$.  
\end{proof}

In this section, the squared distance function is denoted by $\dist$. Here is the main theorem of this section.
\begin{theorem}\label{thm:hard-instance}
For any given parameter $\alpha$, there exists an instance $\sI_\alpha$ such that with constant probability the solution returned by $k$-means++ on $(\sI_\alpha, k_\alpha = \frac{13}{12}n_\alpha)$ is violating the requirement of {\ipste} clustering by a factor of $\alpha$---the maximum violation over the points in $\sI_\alpha$ is at least $\alpha$. 
\end{theorem}
To prove the above theorem, we show that by setting the values of $D_\alpha$ and $n_{\alpha}$ properly, $k$-means++ with constant probability picks the points corresponding to $u,v$ from at least one of the copies $I_j$ (and picks no other points from $I_j$). 

\begin{lemma}\label{lem:iter-1}
Let $\sS_1$ denote the set of centers picked by $k$-means++ on $(\sI_\alpha, \frac{13}{12}n_\alpha)$ after the first $n_\alpha$ iterations in the seeding phase where $n_\alpha \ge 18700$. If $D_\alpha > \alpha r \sqrt{\frac{3n_{\alpha}}{\eps}}$ where $\eps<\frac{1}{100 n_\alpha}$, then with probability at least $0.98$, 
\begin{enumerate}[leftmargin=*]
    \item $\sS_1$ contains exactly one point from each block; $\forall I\in \sI_\alpha, |\sS_1 \cap I| = 1$, and
    \item In at least $\frac{n_\alpha}{10}$ blocks, copies of $v$ are picked; $|j\in [n_\alpha]\;|\; I_j\cap \sS_1 = \{v_j\}| \ge \frac{n_\alpha}{10}$.
\end{enumerate}
\end{lemma}
\begin{proof}
We start with the first property. Consider iteration~$i\leq n_\alpha$ of $k$-means++ and let $S_{i-1}$ denote the set of points that are selected as centers in the first $i-1$ iterations. 
Let $\sI^{+}_{i-1} := \{I_j \;|\; I_j \cap S_{i-1} \neq \emptyset\}$ and $\sI^{-}_{i-1} :=\{I_j \;|\; I_j \cap S_{i-1} = \emptyset\}$.
Since for any pair of points $p\in I, p'\in I'$ where $I\neq I'$, $\dist(p,p') \geq D^2_\alpha$, for any point $p\in \sI^{-}_{i-1}$, $\dist(p, S_{ i-1}) \geq D_\alpha^2$. Moreover, by the construction of the building block (see Figure~\ref{fig:hard_instance}-(a)), for any point $p'\in \sI^{+}_{i-1}$, $\dist(p', S_{i-1}) \leq 4\alpha^2 r^2\leq \frac{4\eps}{3n_{\alpha}} D_\alpha^2$. Hence, the probability that in iteration~$i$ the algorithm picks a point from  $\sI^{-}_{i-1}$ is
\begin{align*}
\frac{\sum_{p\in \sI^{-}_{i-1}}\dist(p, S_{i-1})}{\sum_{p\in \sI^{-}_{i-1}}\dist(p, S_{i-1}) + \sum_{p\in \sI^{+}_{i-1}}\dist(p, S_{i-1})} \ge \frac{4D^2_\alpha}{4D^2_\alpha + 3n_\alpha \cdot (4\alpha^2 r^2)} \ge \frac{1}{1+\eps} \geq 1-\eps
\end{align*}
Thus, the probability that $i$-th point is picked from $\sI^+_{i-1}$ is at most $\eps$. By union bound over the first $n_\alpha$ iterations of $k$-means++, the probability that $\sS_1$ picks exactly one point from each block in $\sI_{\alpha}$ is at least
\begin{align}\label{eq:cover-block}
    1 - \sum_{i=1}^{n_\alpha} \Pr[\text{$i$-th point is picked from $\sI^+_{i-1}$}] \ge 1 - \eps \cdot n_\alpha \ge 0.99
\end{align}
Next, we show that the second property also holds with high probability. Consider iteration~$i$ of the algorithm. By the approximate triangle inequality for $\dist$ function\footnote{$\forall u,v,w\in P, \dist(v,w) \leq 2(\dist(v,u) + \dist(u,w))$.} and since the distance of any pair of points within any block is $2\alpha r$, for a point $p\in I_j \in \sI^-_{i-1}$, $\eta^2 \leq \dist(p, S_{i-1}) \leq 2(\eta^2 + 4\alpha^2r^2)$ where $\eta \ge D_\alpha$. In particular, for each block $I_j\in \sI^{-}_{i-1}$, 
\begin{align}\label{eq:v-pick}
\frac{\dist(v_j, S_{i-1})}{\sum_{p\in I_j}\dist(p, S_{i-1})} \geq \frac{\eta^2}{7\eta^2 + 24\alpha^2 r^2}\geq 1/(8+\frac{\eps}{n_\alpha}) 
\end{align}
Let $X_i$ be a random variable which indicates that the $i$-th point belongs to $\{v_j | j \in [n_\alpha]\}$; $X_i = 1$ if the $i$-th point belongs to $\{v_j | j \in [n_\alpha]\}$ and is equal to zero otherwise.
\begin{align*}
    \Pr[X_i = 1] 
    &\ge \frac{\sum_{I_j\in \sI^{-}_{i-1}} \dist(v_j, S_{i-1})}{\sum_{I_j\in \sI^{-}_{i-1}} \sum_{p\in I_j}\dist(p, S_{i-1}) + \sum_{p\in \sI^{+}_{i-1}}\dist(p, S_{i-1})} \\ 
    &\ge \frac{D^2_\alpha}{(8+\eps/n_\alpha) \cdot D^2_\alpha + \sum_{p\in \sI^{+}_{i-1}}\dist(p, S_{i-1})} \quad\rhd\text{by Eq.~\eqref{eq:v-pick}}
    \\
    &\ge \frac{D^2_\alpha}{(8+\eps/n_\alpha) \cdot D^2_\alpha + 3n_\alpha \cdot (4\alpha^2 r^2)} \quad \rhd\forall p\in \sI^+_{i-1}, \dist(p, S_{i-1})\leq 4\alpha^2 r^2
    \\
    &\ge \frac{D^2_\alpha}{(8+\eps/n_\alpha) \cdot D^2_\alpha + \eps \cdot D^2_\alpha} > 1/9
\end{align*}
Note that we showed that $\Pr[X_i=1] \geq 1/9$ independent of the algorithm's choices in the first $i-1$ iterations of the algorithm. In particular, random variables $X_1, \ldots, X_{n_\alpha}$ are stochastically dominated by {\em independent} random variable $Y_1, \ldots, Y_{n_{\alpha}}$ where each $Y_i =1$ with probability $1/9$ and is zero otherwise. Hence, the random variable $X:=X_1 + \ldots + X_{n_{\alpha}}$ is stochasitcally dominated by $Y$ which is $\B(n_\alpha, 1/9)$, the binomial distribution with $n_\alpha$ trials and success probability $1/9$.  
Hence, by an application of Hoeffding's inequality on $Y$,
\begin{align}
    \Pr[X \leq \frac{n_\alpha}{10}] \leq \Pr[Y\leq \frac{n_\alpha}{10}] 
    &\leq \exp(-2n_{\alpha}(\frac{1}{9}-\frac{1}{10})^2) \leq 0.01 &&\rhd\text{for $n_\alpha>18700$}\label{eq:witness-block}
\end{align}

Thus, by Eq.~\eqref{eq:cover-block} and~\eqref{eq:witness-block}, with probability at least $0.98$ both properties hold.
\end{proof}

\begin{lemma}\label{lem:iter-2}
Let $\sS$ be the set of centers picked at the end of the seeding phase of $k$-means++ on $(\sI_\alpha, k_\alpha = \frac{13}{12}n_\alpha)$ where $n_\alpha>\max(9000, 240\alpha^2)$. Then, with probability at least $0.33$, there exists a block $I_j\in \sI_\alpha$ such that $\sS \cap I_j = \{v_j, u_j\}$. 
\end{lemma}
\begin{proof}

Let $\sT_1$ denote the event that $\sS_1$, the set of centers picked in the first $n_\alpha$ iterations, picks exactly one point from each block and in at least $n_{\alpha}/10$ blocks the $v$-type points are picked in $\sS_1$.
Note that by an application of Lemma~\ref{lem:iter-1}, $\Pr[\sT_1] \geq 0.98$---with probability at least $0.98$, $\sT_1$ holds.

First we show the following. Let $m_\alpha = n_{\alpha}/20$. With constant probability, there exists an iteration~$i^*\in [n_\alpha+1, n_\alpha + m_\alpha]$ in which $k$-means++ picks a $u$-type point form $I_j$ such that $I_j \cap S_{i^*-1} = \{v_j\}$. Let $\mathcal{E}_i$ denote the event that in iteration~$i+n_\alpha$ where $i \in [m_\alpha]$, $k$-means++ picks the point $u_j$ from $I_j$ such that $S_{i-1} \cap I_j = \{v_j\}$. 
Note that for each block in $\{I_j \; | \; \sS_1 \cap I_j = \{v_j\}\}$, $\sum_{p\in I_j} \dist(p, \sS_1) = (2\alpha^2+1) r^2$; otherwise, $\sum_{p\in I_j} \dist(p, \sS_1) \le (5\alpha^2+1) r^2$. 
Hence, for each $i\in [m_{\alpha}]$, conditioned on $\sT_1$ and independent of the prior choices of the algorithm in iterations $n_\alpha+1, \ldots, i^*:=n_{\alpha}+i-1$,
\begin{align*}
    \Pr[\mathcal{E}_i | \sT_1] 
    &= \frac{\sum_{\{I_j \; | \; I_j\cap S_{i^*} = \{v_j\}\}} \dist(u_j, S_{i^*})}{\sum_{\{I_j \; | \; I_j\cap S_{i^*} = \{v_j\}\}} \sum_{p\in I_j}\dist(p, S_{i^*}) + \sum_{\{I_j \; | \; I_j\cap S_{i^*} \neq \{v_j\}\}} \sum_{p\in I_j}\dist(p, S_{i^*})} \\
    &\ge\frac{(\frac{n_\alpha}{10} - m_\alpha)\cdot r^2}{(\frac{n_\alpha}{10} - m_\alpha)\cdot (2\alpha^2 + 1) r^2 + (n_\alpha - (\frac{n_\alpha}{10} -m_\alpha))\cdot (5\alpha^2+1)r^2} \\
    &\ge \frac{m_\alpha}{ m_\alpha (2\alpha^2+1) + (n_\alpha-m_\alpha)(5\alpha^2+1)} \\
    &\ge \frac{1}{2\alpha^2 + 1 + 20(5\alpha^2 +1)}\\
    &\ge \frac{1}{123\alpha^2}
\end{align*}
For each $i\in [m_\alpha]$, let random variable $X_i := \boldsymbol{1}_{\sE_i|\sT_1}$.
Since $X_1, \ldots, X_{m_\alpha}$ are stochastically dominated by independent random variables $Y_i$ such that $Y_i = 1$ with probability $1/(123\alpha^2)$ and zero otherwise, we can bound the probability that none of $\sE_1, \ldots, \sE_{m_{\alpha}}$ happens as follows. 
\begin{align}
    \Pr[\neg\sE_1 \wedge \ldots \wedge \neg\sE_{m_\alpha} | \sT_1] 
    &= \Pr[\sum_{i\leq m_\alpha} X_i < 1 | \sT_1] \nonumber\\
    &\leq \Pr[\sum_{i\leq m_\alpha} Y_i < 1] \nonumber\\
    &= (1- \frac{1}{123\alpha^2})^{m_\alpha} \leq \exp(-\frac{m_\alpha}{123\alpha^2}) \label{eq:pick-u-v}
\end{align}
Hence, with probability at least $1- \exp(-\frac{m_\alpha}{123\alpha^2})$, there exists an iteration~$i^*\in [n_\alpha+1, n_\alpha + m_\alpha]$ and block $I^*_j$ such that $I^*_j\cap S_{i^*} = \{v_j, u_j\}$. Next, we show that with constant probability no more points is picked from $I^*_j$ in the remaining iterations of the seeding phase $k$-means++.
Let $\sT_2$ denote the event that there exists $(I^*_j, i^*)$ such that $I^*_j \cap S_{i^*} = \{v_j, u_j\}$; $\sT_2 := \sE_1 \vee \ldots \vee \sE_{m_{\alpha}}$. For any $i\in[1, n_\alpha+m_\alpha - i^*]$, let $\sF_i$ denote the event that in iteration~$i^* + i$, $k$-means++ picks another point from $I^*_j$.
\begin{align}\label{eq:bad-event}
    \Pr[\sF_i | \sT_1 \wedge \sT_2] 
    &< \frac{2\alpha^2 r^2}{\sum_{\{I_j \; | \; I_j\cap S_{i-1} = \{v_j\}\}} \sum_{p\in I_j}\dist(p, S_{i-1})} \nonumber \\
    &\leq \frac{2\alpha^2 r^2}{(\frac{n_\alpha}{10}-m_\alpha)\cdot (2\alpha^2 + 1) r^2} = \frac{2}{3m_{\alpha}}
\end{align}
Let $t_\alpha := n_\alpha + m_\alpha - i^*$ and define $\sT_3$ to denote the event that none of $\sF_1, \ldots, \sF_{t_\alpha}$ happens.
For each $i\in [t_\alpha]$, let random variable $\tilde{X}_i := \boldsymbol{1}_{\sF_i|\sT_1 \wedge \sT_2}$.
Since $\tilde{X}_1, \ldots, \tilde{X}_{m_\alpha}$ are stochastically dominated by independent random variables $\tilde{Y}_i$ such that $\tilde{Y}_i = 1$ with probability $1/(123\alpha^2)$ and zero otherwise, we can lower bound the probability that event $\sT_3$ holds as follows. 
\begin{align}
    \Pr[\sT_3 | \sT_1 \wedge \sT_2] = \Pr[\neg\sF_1 \wedge \ldots \wedge \neg \sF_{t_{\alpha}} | \sT_1 \wedge \sT_2] 
    &\ge \Pr[\sum_{i\leq t_\alpha} \tilde{Y}_i < 1 | \sT_1 \wedge \sT_2] \nonumber \\
    &\ge (1-\frac{2}{3m_\alpha})^{t_\alpha} &&\rhd\text{by Eq.~\eqref{eq:bad-event}} \nonumber \\
    &\ge (1-\frac{2}{3m_\alpha}) \cdot \exp(-\frac{2}{3}) &&\rhd t_\alpha \leq m_\alpha \label{eq:not-violate-u-v}
\end{align}

Hence, Eq.~\eqref{eq:pick-u-v} and Eq.~\eqref{eq:not-violate-u-v} together imply that the probability that there exists no block $I^*_j$ such that $|I^*_j \cap \sS| = \{v_j, u_j\}$ is at most 
\begin{align*}
    \Pr[\neg \sT_1] + \Pr[\neg \sT_2 | \sT_1] + \Pr[\neg \sT_3 | \sT_1 \wedge \sT_2]
    &\leq 0.02 + e^{-\frac{m_\alpha}{123\alpha^2}} + 1 - (1-\frac{2}{3m_\alpha})e^{-\frac{2}{3}}\\
    &\leq 0.02 + e^{-\frac{m_\alpha}{123\alpha^2}} + \frac{1}{2}+\frac{1}{3m_\alpha} \;\rhd m_\alpha \ge 240\alpha^2 \\
    &\leq 0.02 + 0.145 + 0.5 + 0.005 < 0.67
\end{align*}
Thus, with probability at least $0.33$, at the end of the seeding phase, there exists a block $I^*_j$ such that $I^*_j \cap \sS = \{v_j, u_j\}$.
\end{proof}

Now, we are ready to prove Theorem~\ref{thm:hard-instance}.
\begin{proof}[Proof of Theorem~\ref{thm:hard-instance}]
Note that Lemma~\ref{lem:iter-1} together with Lemma~\ref{lem:iter-2} implies that with constant probability, at the end of the seeding phase, at least one point is picked from each block in $\sI_{\alpha}$ and there exist a block $I_j$ such that $I_j\cap S = \{v_j, u_j\}$.

Given the structure of $\sI_{\alpha}$, after each step of the Lloyd's heuristic, no points from two different blocks will be in the same cluster. Furthermore, by Claim~\ref{clm:unstable-factor}, the clusters of $I_j$ which are initially $\{z_j, z'_j, v_j\}, \{u_j\}$ remain unchanged in the final solution of $k$-means++. However, such clustering is violating the requirement of {\ipste} clustering for $v_j$ by a factor of $\alpha$.

Thus, with constant probability, the solution of $k$-means++ on $(\sI_\alpha, k_\alpha)$ violates the requirement of {\ipste} clustering by a factor of at least $\alpha$.
\end{proof}

\subsection{Hard Instances for $k$-center}

\label{sec:hard_instance_kcenter}
\begin{figure}[!ht]
    \centering
    \includegraphics[width=0.5\textwidth]{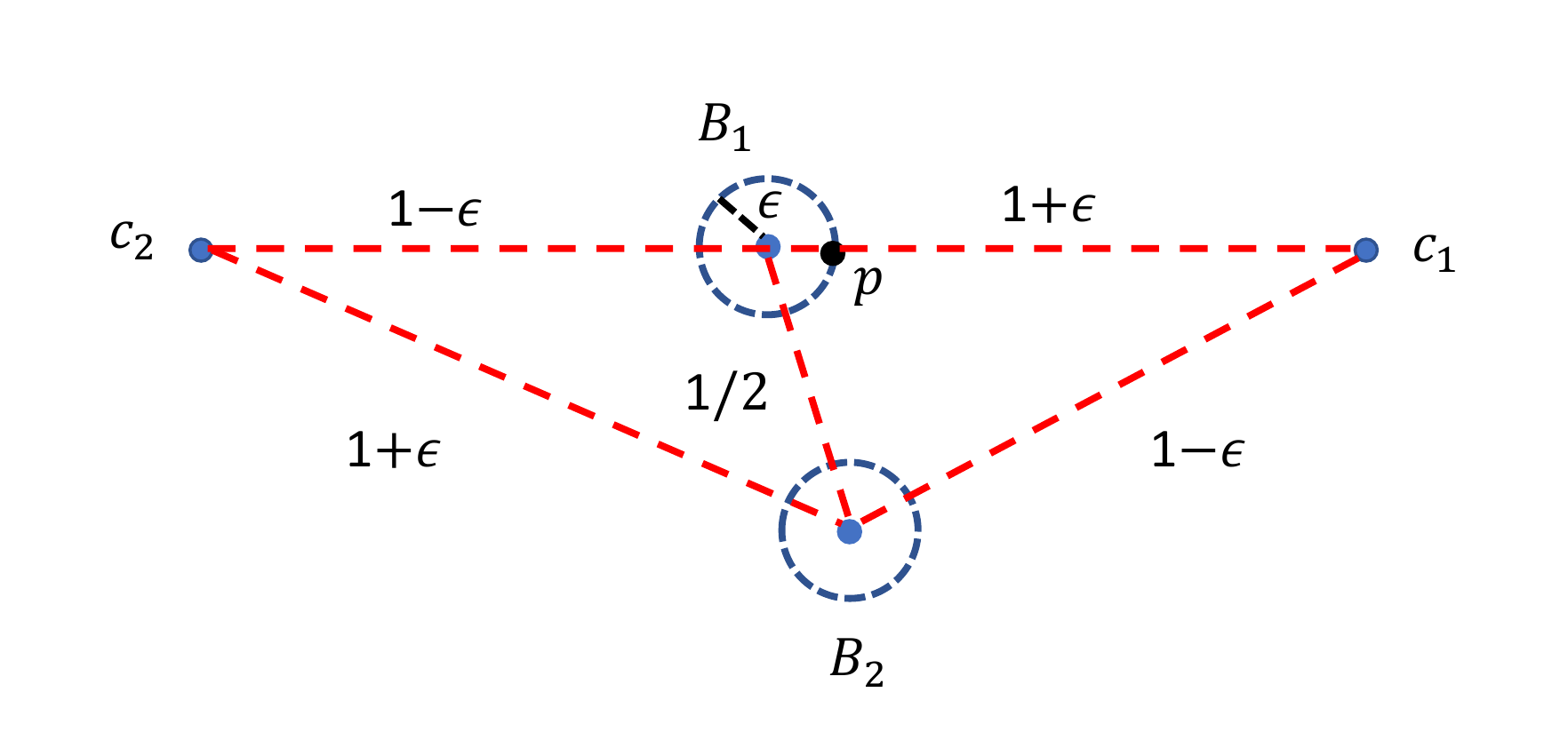}
    \caption{A hard instance for $k$-center.}
    \label{fig:k-center-bad-example}
\end{figure}

Figure~\ref{fig:k-center-bad-example} shows a hard instance for $k$-center. Each of the balls $B_1$ and $B_2$ have radius $\epsilon$. The center of $B_1$ has a distance of $1+\epsilon$ from $c_1$, and a distance of $1-\epsilon$ from $c_2$. The center of $B_2$ has a distance of $1-\epsilon$ from $c_1$, and a distance of $1+\epsilon$ from $c_2$. The nodes of the graph are $c_1$, $c_2$, $n$ points on the surface of $B_1$, and $n$ points on the surface of $B_2$.

Consider $k$-center using the greedy strategy by~\cite{gonzalez1985}. Suppose $k=2$, and the first center picked is $c_1$. The furthest node from $c_1$ is $c_2$, and hence it is the second center picked.  Since all the points get assigned to the closest center, all the points on $B_1$ get assigned to $c_2$, except the point $p$ that has the same distance from $c_1$ and $c_2$ and gets assigned to $c_1$. All the points on $B_2$ get assigned to $c_1$. 
We show such a clustering is unstable for $p$ within a factor of $n/8$. The average distance of $p$ to the nodes in its own cluster is at least $\frac{n(1/2-2\epsilon)+1}{n+1}$. The average distance of $p$ to the nodes assigned to $c_2$ is at most $\frac{2{\epsilon}(n-1)+1}{n}$. For $\epsilon \leq 1/2n$, the instability for $p$ is at least within a multiplicative factor of $n/8$. By setting $n>8\alpha$, $p$ is not $t$-approximately {\ipste} for any value $t<\alpha$.

\subsection{Hard Instances for Single Linkage}
\label{sec:hard_instance_single_linkage}
Figure~\ref{fig:single-linkage-bad-example} shows a bad example for single linkage clustering when $k=2$. A possible implementation of single linkage first merges $v_2$ and $v_3$. Next, it merges $v_4$ and $\{v_2, v_3\}$, and repeatedly at each iteration~$i$ adds $v_{i+2}$ to the set $\{v_2, v_3, \ldots, v_{i+1}\}$. When there are only two clusters $\{v_2, v_3, \ldots, v_n\}$ and $v_1$ left, the algorithm terminates. By setting $\epsilon = 1/2$, $v_2$ gets unstable by a factor of $(n-1)/4$. 
By setting $(n-1)/4 > \alpha$, $v_2$ is not $t$-approximately {\ipste} for any value of $t<\alpha$.

\begin{figure}[!ht]
    \centering
    \includegraphics[width=0.5\textwidth]{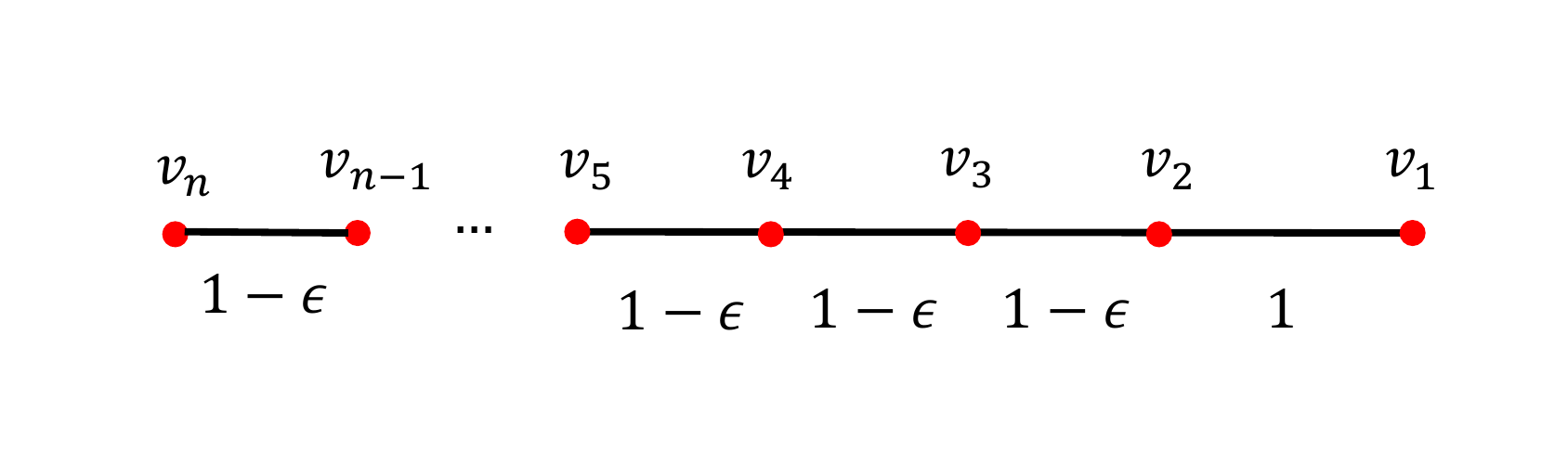}
    \vspace{-0.3cm}
    \caption{A hard instance for single linkage clustering.}
    \label{fig:single-linkage-bad-example}
\end{figure}

\section{Dynamic Programming Approach for Solving \eqref{1dim-problem}}\label{appendix_p_equals_infty}

In the following, we propose an efficient DP approach to find 
a 
solution~to~\eqref{1dim-problem}. 
Let 
$\dataset=\{x_1,\ldots,x_n\}$ with $x_1\leq \ldots\leq x_n$. 
Our approach builds a 
table~$T\in(\N\cup\{\infty\})^{n\times n\times k}$
with  
\begin{align}\label{definition_table_T}
T(i,j,l)=\min_{(C_1,\ldots,C_l)\in\mathcal{H}_{i,j,l}} \|(|C_1|-t_1,\ldots,|C_l|-t_l)\|_p^p
\end{align}
for $i\in[n]$, $j\in[n]$, $l\in[k]$, where
\begin{align*}
&\mathcal{H}_{i,j,l}=\big\{\mathcal{C}=(C_1,\ldots,C_l):~\text{$\mathcal{C}$ is a stable $l$-clustering of $\{x_1,\ldots,x_i\}$ with $l$ non-empty}\\
&~~~~~~~~~~~~~~~~~~~~~\text{contiguous clusters such that the right-most cluster $C_l$ contains exactly $j$ points}\big\}
\end{align*}
and $T(i,j,l)=\infty$ if $\mathcal{H}_{i,j,l}=\emptyset$. Here, we consider the case $p\neq \infty$. The modifications of our approach 
to 
the case~$p=\infty$ are minimal and 
are 
described later.

The optimal value of \eqref{1dim-problem} is given 
by $\min_{j\in[n]} T(n,j,k)^{1/p}$. Below, we will describe how to use the table~$T$ to compute an {\ipste} $k$-clustering solving \eqref{1dim-problem}. 
First, we explain how to build $T$. 
We have, for $i,j\in[n]$, 
\begin{align}\label{table_T_initial}
\begin{split}
 &T(i,j,1)=\begin{cases}
 |i-t_1|^p,& j=i, \\
 \infty, &j\neq i\\           
           \end{cases},
\qquad
T(i,j,i)=\begin{cases}
 \sum_{s=1}^i|1-t_s|^p,& j=1, \\
 \infty, &j\neq 1\\           
           \end{cases},\\
 &T(i,j,l)=\infty,\quad j+l-1>i,
 \end{split}
\end{align}
and
 the recurrence relation, for $l>1$ and $j+l-1\leq i$,
\begin{align}\label{recurrence_relation}
\begin{split}
T(i,j,l)=|j-t_l|^p+\min\left\{T(i-j,s,l-1):  s\in[i-j-(l-2)],\frac{\sum_{f=1}^{s-1}|x_{i-j}-x_{i-j-f}|}{s-1}\leq\right.\\
 \left.\frac{\sum_{f=1}^{j}|x_{i-j}-x_{i-j+f}|}{j},\frac{\sum_{f=2}^{j}|x_{i-j+1}-x_{i-j+f}|}{j-1}\leq\frac{\sum_{f=0}^{s-1}|x_{i-j+1}-x_{i-j-f}|}{s}\right\},
\end{split}
\end{align}
where we use the convention that $\frac{0}{0}=0$ for the 
fractions on the left sides of the inequalities. 
%
First, we explain relation~\eqref{recurrence_relation}. Before that we need to show the following lemma holds:

\begin{lemma}[Stable boundary points imply stable  clustering]\label{lemma_boundary_points}
 Let $\mathcal{C}=(C_1,\ldots,C_k)$ be a $k$-clustering of $\dataset=\{x_1,\ldots,x_n\}$, where $x_1\leq x_2\leq\ldots\leq x_n$, with contiguous clusters 
 $C_1=\{x_1,\ldots,x_{i_1}\},C_2=\{x_{i_1+1},\ldots,x_{i_2}\},\ldots,C_k=\{x_{i_{k-1}+1},\ldots,x_n\}$, for some 
 $1\leq i_1<\ldots<i_{k-1}<n$. Then 
 $\mathcal{C}$ is 
 {\ipste} if and only if 
all points $x_{i_l}$ and $x_{i_l+1}$, $l\in[k-1]$,
 are stable.
 Furthermore, 
$x_{i_l}$ ($x_{i_l+1}$, resp.) is stable if and only if its average distance to the points in $C_l\setminus\{x_{i_l}\}$ ($C_{l+1}\setminus\{x_{i_l+1}\}$, resp.) 
is not greater than the average distance to the points in $C_{l+1}$ ($C_{l}$, resp.). 
\end{lemma}

\begin{proof}
We assume that $\dataset=\{x_1,\ldots,x_n\}\subseteq \R$ with $x_1\leq x_2\leq \ldots \leq x_n$ and write the Euclidean metric $d(x_i,x_j)$ between two points $x_i$ and $x_j$  
in its usual way $|x_i-x_j|$. 

If $\mathcal{C}$ is  {\ipste}, then all points $x_{i_l}$ and $x_{i_l+1}$, $l\in[k-1]$, are stable. 
Conversely, let us assume that $x_{i_l}$ and $x_{i_l+1}$, $l\in[k-1]$, are stable. We need to show that all points in $\dataset$ are stable.  
Let $\tilde{x}\in C_l=\{x_{i_{l-1}+1},\ldots,x_{i_l}\}$ for some $l\in\{2,\ldots,k-1\}$ 
and $l'\in\{l+1,\ldots,k\}$. Since $x_{i_l}$ is stable, we have
\begin{align*}
 \frac{1}{|C_l|-1}\sum_{y\in{C_{l}}} (x_{i_l}-y)= \frac{1}{|C_l|-1}\sum_{y\in{C_{l}}} |x_{i_l}-y|\leq  
  \frac{1}{|C_{l'}|}\sum_{y\in{C_{l'}}} |x_{i_l}-y|=\frac{1}{|C_{l'}|}\sum_{y\in{C_{l'}}} (y-x_{i_l}) 
\end{align*}
and hence
\begin{align*}
 \frac{1}{|C_l|-1}\sum_{y\in{C_{l}}} |\tilde{x}-y|&\leq  \frac{1}{|C_l|-1}\sum_{y\in C_{l}\setminus\{\tilde{x}\}} (|\tilde{x}-x_{i_l}|+|x_{i_l}-y|)\\
 &= (x_{i_l}-\tilde{x})+\frac{1}{|C_l|-1}\sum_{y\in C_{l}\setminus\{\tilde{x}\}} (x_{i_l}-y)\\
 &\leq (x_{i_l}-\tilde{x})+\frac{1}{|C_{l'}|}\sum_{y\in{C_{l'}}} (y-x_{i_l})\\
 &=\frac{1}{|C_{l'}|}\sum_{y\in{C_{l'}}} (y-\tilde{x})\\
 &=\frac{1}{|C_{l'}|}\sum_{y\in{C_{l'}}} |\tilde{x}-y|. 
\end{align*}
Similarly, we can show for $l'\in\{1,\ldots,l-1\}$ that 
\begin{align*}
 \frac{1}{|C_l|-1}\sum_{y\in{C_{l}}} |\tilde{x}-y|\leq  \frac{1}{|C_{l'}|}\sum_{y\in{C_{l'}}} |\tilde{x}-y|, 
\end{align*}
and hence $\tilde{x}$ is stable. 
Similarly, we can show that all points $x_1,\ldots,x_{{i_1}-1}$ and $x_{i_{k-1}+2},\ldots,x_n$ are stable.

For the second claim observe that for $1\leq s\leq l-1$, the average distance of $x_{i_l}$ to the points in $C_s$ cannot 
 be smaller than the average distance to the points 
in $C_l\setminus\{x_{i_l}\}$
and for $l+2\leq s\leq k$,  
the average distance of $x_{i_l}$ to the points in $C_s$ cannot be smaller than the average distance to the points 
in $C_{l+1}$. A similar argument proves the claim for $x_{i_l+1}$.
\end{proof}

It follows from Lemma~\ref{lemma_boundary_points} that a clustering $(C_1,\ldots,C_l)$ of $\{x_1,\ldots,x_i\}$ with contiguous clusters and $C_l=\{x_{i-j+1},\ldots,x_i\}$ is 
{\ipste} if and only if $(C_1,\ldots,C_{l-1})$ is a stable clustering of  $\{x_1,\ldots,x_{i-j}\}$ and the average distance of $x_{i-j}$ to the points in 
$C_{l-1}\setminus\{x_{i-j}\}$ is not greater 
than  the average distance to the points in $C_l$ and the average distance of $x_{i-j+1}$ to the points in $C_{l}\setminus\{x_{i-j+1}\}$ is not greater 
than  the average distance to the points in $C_{l-1}$. The latter two conditions correspond to the two inequalities in~\eqref{recurrence_relation} 
(when 
$|C_{l-1}|=s$, where $s$ is a variable). 
By 
explicitly enforcing 
these 
two constraints,  
we can utilize the first condition and 
rather than minimizing over  $\mathcal{H}_{i,j,l}$ in \eqref{recurrence_relation_explained}, 
we can minimize over both $s\in[i-j-(l-2)]$ and $\mathcal{H}_{i-j,s,l-1}$ 
(corresponding to minimizing over all {\ipste}
$(l-1)$-clusterings of  $\{x_1,\ldots,x_{i-j}\}$ with non-empty contiguous clusters). It is 
\begin{align*}
\min_{\substack{s\in[i-j-(l-2)]\\ (C_1,\ldots,C_{l-1})\in\mathcal{H}_{i-j,s,l-1}}} 
\|(|C_1|-t_1,\ldots,|C_{l-1}|-t_{l-1})\|_p^p =\min_{s\in[i-j-(l-2)]} T(i-j,s,l-1),
\end{align*}
and hence we end up with the recurrence relation~\eqref{recurrence_relation}.

It is not hard to see that using~\eqref{recurrence_relation}, we can build the table~$T$ in 
time $\mathcal{O}(n^3k)$. Once we have $T$, we 
can compute a solution $(C_1^*,\ldots,C_k^*)$ to \eqref{1dim-problem} by specifying 
$|C_1^*|,\ldots,|C_k^*|$ 
in 
time $\mathcal{O}(nk)$ as follows: let $v^*=\min_{j\in[n]} T(n,j,k)$. We set $|C_k^*|=j_0$ for an arbitrary 
$j_0$ with $v^*=T(n,j_0,k)$. For $l=k-1,\ldots,2$, we then set $|C_l^*|=h_0$ for an arbitrary $h_0$~with 
(i) $T(n-\sum_{r=l+1}^k |C_r^*|,h_0,l)+\sum_{r=l+1}^k ||C_r^*|-t_r|^p=v^*$,
 (ii) the average distance of $x_{n-\sum_{r=l+1}^k |C_r^*|}$ to the closest $h_0-1$ many  points on its left side is not greater than the average 
distance to the points in $C_{l+1}^*$, and 
(iii) the average distance of $x_{n-\sum_{r=l+1}^k |C_r^*|+1}$ to the other points in $C_{l+1}^*$ is 
 not greater than the average distance to the closest $h_0$ many points on its left side. 
Finally, it is $|C_1^*|=n-\sum_{r=2}^k |C_r^*|$. It follows from the definition 
of the table~$T$ in~\eqref{definition_table_T} and Lemma~\ref{lemma_boundary_points} that for $l=k-1,\ldots,2$ we can always find some $h_0$ satisfying (i)~to~(iii)
and that our approach yields an {\ipste} $k$-clustering $(C_1^*,\ldots,C_k^*)$ of $\dataset$. 

Hence we have shown the following theorem:

\begin{theorem}[Efficient 
DP 
approach solves \eqref{1dim-problem}]
By means of the dynamic programming approach \eqref{definition_table_T} to \eqref{recurrence_relation} we can compute 
an {\ipste} clustering 
solving 
\eqref{1dim-problem} in running time $\mathcal{O}(n^3k)$.
\end{theorem}

Let us first explain the recurrence relation~\eqref{recurrence_relation}:  
because of $\|(x_1,\ldots,x_l)\|_p^p=\|(x_1,\ldots,x_{l-1})\|_p^p+|x_l|^p$ and for every clustering~$(C_1,\ldots,C_l)\in\mathcal{H}_{i,j,l}$ it 
is $|C_l|=j$, we have
\begin{align}\label{recurrence_relation_explained}
T(i,j,l)=|j-t_l|^p+\min_{(C_1,\ldots,C_l)\in\mathcal{H}_{i,j,l}} \|(|C_1|-t_1,\ldots,|C_{l-1}|-t_{l-1})\|_p^p.
\end{align}

\subsection{$p=\infty$}
Now we describe how to modify the dynamic programming approach of Section~\ref{section_1dim} to the case $p=\infty$: 
in this case, we replace the definition of the table~$T$ in \eqref{definition_table_T} by  
\begin{align*}
T(i,j,l)=\min_{(C_1,\ldots,C_l)\in\mathcal{H}_{i,j,l}} \|(|C_1|-t_1,\ldots,|C_l|-t_l)\|_{\infty},\quad i\in[n], j\in[n], l\in[k],
\end{align*}
and $T(i,j,l)=\infty$ if $\mathcal{H}_{i,j,l}=\emptyset$ as before.
The optimal value of \eqref{1dim-problem} is now given 
by $\min_{j\in[n]} T(n,j,k)$. Instead of \eqref{table_T_initial}, 
we have, for $i,j\in[n]$, 
\begin{align*}
T(i,j,1)=\begin{cases}
 |i-t_1|,& j=i, \\
 \infty, &j\neq i\\           
           \end{cases},\hspace{7mm}
T(i,j,i)=\begin{cases}
 \max_{s=1,\ldots,i}|1-t_s|,& j=1, \\
 \infty, &j\neq 1\\           
           \end{cases}
\end{align*}
and
\begin{align*}
T(i,j,l)=\infty,\quad j+l-1>i,
\end{align*}
and
 the recurrence relation~\eqref{recurrence_relation} now becomes, for $l>1$ and $j+l-1\leq i$,
\begin{align*}
&T(i,j,l)=\max\Bigg\{|j-t_l|,\min\bigg\{T(i-j,s,l-1):  s\in[i-j-(l-2)],\\
&~~~~~~~~~~~~~~~~~~~~~~~~~~~~~~~~~~~~~~~~~~~~~~~~~~~~~~\frac{1}{s-1}\,{\sum_{f=1}^{s-1}|x_{i-j}-x_{i-j-f}|}\leq \frac{1}{j}\,{\sum_{f=1}^{j}|x_{i-j}-x_{i-j+f}|},\\
&~~~~~~~~~~~~~~~~~~~~~~~~~~~~~~~~~~~~~~~~~~~~~~~~~~~~~~\frac{1}{j-1}\,{\sum_{f=2}^{j}|x_{i-j+1}-x_{i-j+f}|}\leq\frac{1}{s}\,{\sum_{f=0}^{s-1}|x_{i-j+1}-x_{i-j-f}|}\bigg\}\Bigg\}.
\end{align*}

Just like before, we can build the table~$T$ in 
time $\mathcal{O}(n^3k)$. Computing a solution $(C_1^*,\ldots,C_k^*)$ to \eqref{1dim-problem} 
also works similarly as before. The only thing that we have to change is the condition (i) on $h_0$ 
(when setting $|C_l^*|=h_0$ for $l=k-1,\ldots,2$): now $h_0$ must satisfy 
\begin{align*}
\max\left\{T\left(n-\sum_{r=l+1}^k |C_r^*|,h_0,l\right),\max_{r=l+1,\ldots,k} ||C_r^*|-t_r|\right\}=v^*
\end{align*}
or equivalently 
 \begin{align*}
 T\left(n-\sum_{r=l+1}^k |C_r^*|,h_0,l\right)\leq v^*.
 \end{align*}



\vspace{6mm}
\section{
Addendum to Section~\ref{section_experiments_general}
}\label{appendix_exp_general}

In this section, we provide supplements to Section~\ref{section_experiments_general}:
\begin{itemize}

\item
In Section~\ref{example_group_fair_vs_indi_fair}, we present a simple example that shows that it really depends on the data set whether a group-fair clustering is IP-stable or not.

\item
In Section~\ref{example_local_search}, we provide an example illustrating why the local search idea 
of 
Section~\ref{section_experiments_general} does not work.

\item
In Section~\ref{appendix_pruning_strategy}, we 
present a heuristic approach to make linkage clustering more aligned with 
IP-stability.

\item
In Section~\ref{appendix_exp_general_adult}, we present the experiments of Section~\ref{section_experiments_general} on the Adult data set \citep{UCI_all_four_data_sets_vers2}:  we study the performance of  various standard clustering algorithms as a function of the number of clusters~$k$ when the underlying metric is either the
Euclidean (as in the plot of Figure~\ref{exp_gen_standard_alg_Adult}),  Manhattan or Chebyshev metric. We also study our heuristic approach of Appendix~\ref{appendix_pruning_strategy} for making linkage clustering more stable: just as for the standard algorithms, we show $\Nrunf$, $\Maxviol$, $\Meanviol$ and $\cost$ as a function of $k$ for ordinary average / complete / single linkage clustering and their modified versions 
using our heuristic approach in its both variants ($\#$U denotes the variant 
based on
$\Nrunf$ and MV the variant 
based on 
$\Maxviol$). 


\item
In Section~\ref{appendix_exp_general_drug}, we show the same set of experiments  on the 
Drug Consumption data set \citep{drug_consumption_data}. We used all 1885 records in the data set, and we used all 12 features describing 
a record (e.g., age, gender, or education), but did not use the information about the drug consumption of a record (this information is usually used as label when setting up a classification problem on the data set). We normalized the features to zero mean and unit variance. When 
running the standard clustering algorithms on the data set, we refrained  from running spectral clustering since the Scikit-learn implementation occasionally was not able to do the eigenvector computations and aborted with a LinAlgError. Other than that, all results are largely consistent with the results for the Adult data set.

\item
In Section~\ref{appendix_exp_general_liver}, we show the same set of experiments on the Indian Liver Patient data set \citep{UCI_all_four_data_sets_vers2}. Removing four records with missing values, we ended up with 579 records, 
for which we used all 11 available features (e.g., age, gender, or total proteins). We normalized the features to zero mean and unit variance. 
Again, all results are largely consistent with the results for the Adult data set.

 
\end{itemize}

\begin{figure}[t]
\centering
\includegraphics[scale=0.38]{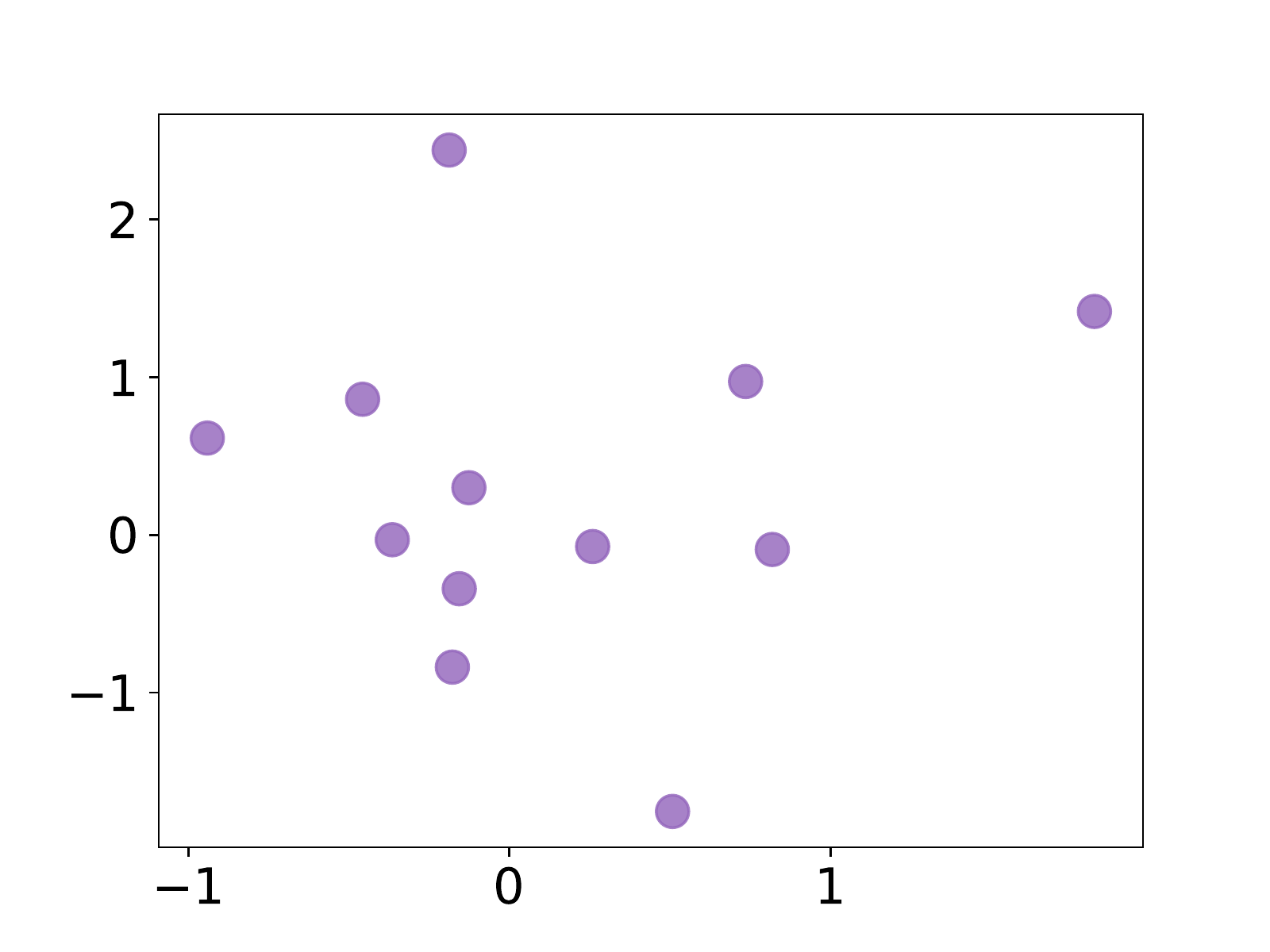}
\hspace{6mm}
\includegraphics[scale=0.38]{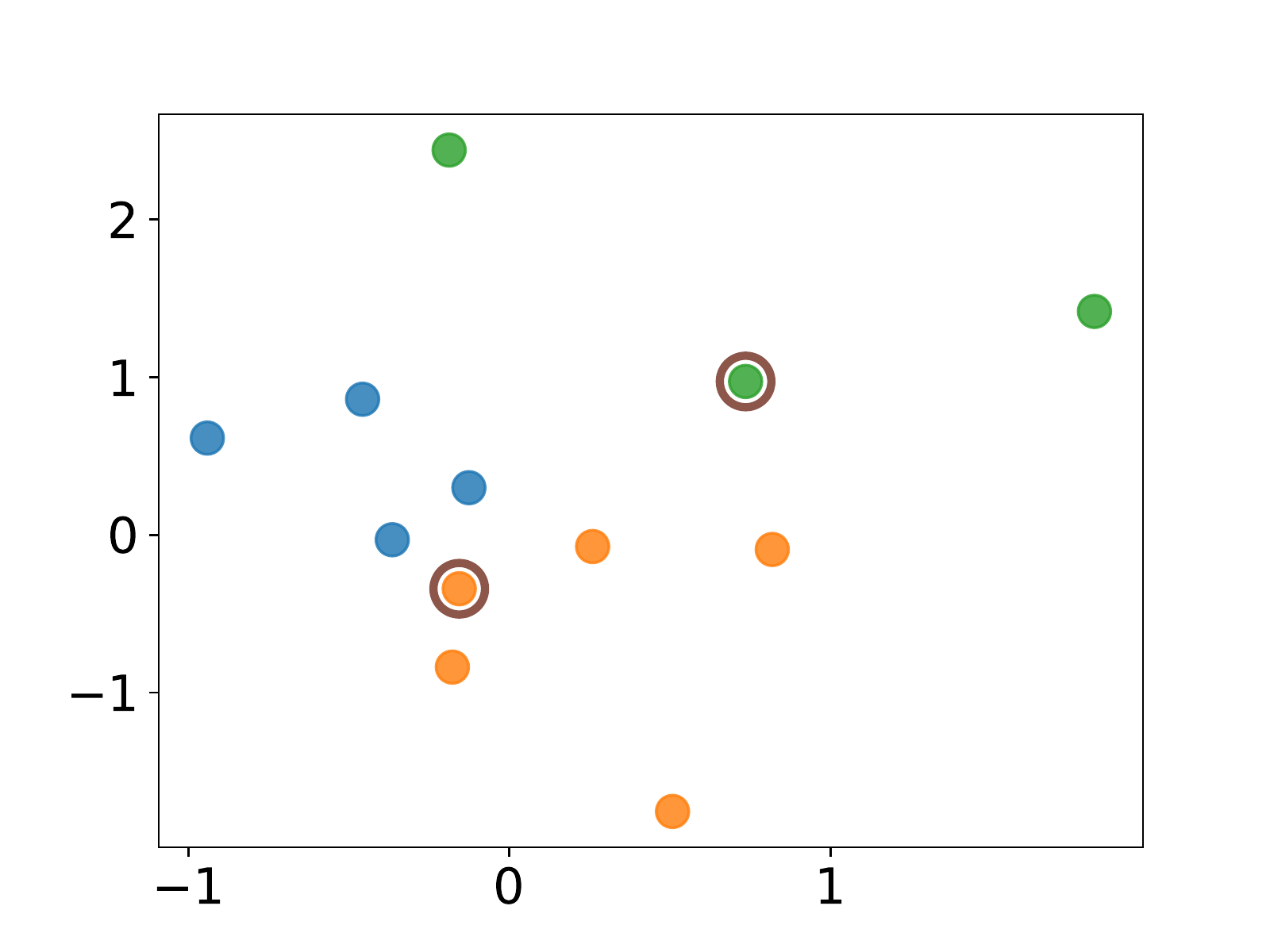}

\includegraphics[scale=0.38]{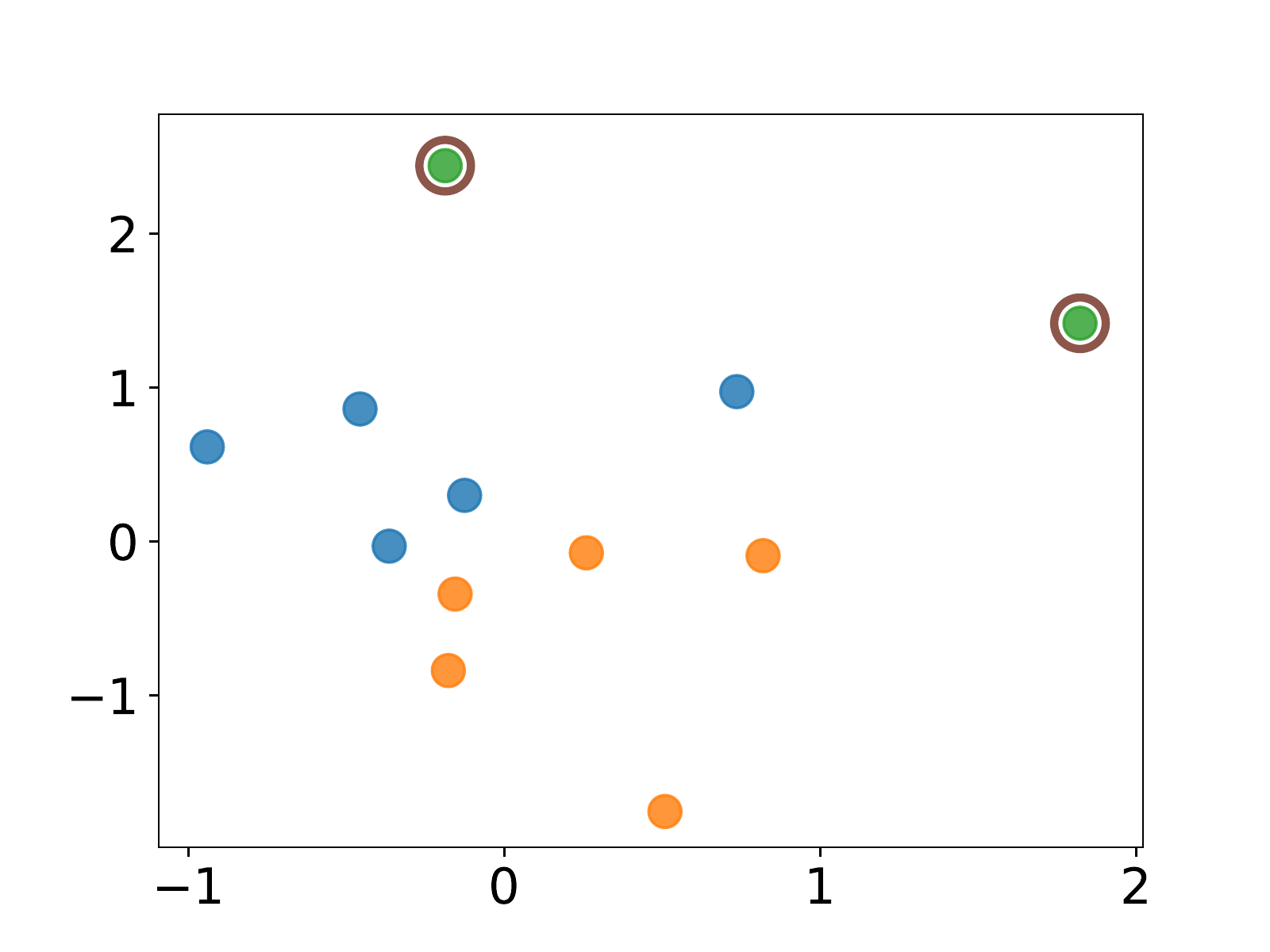}
\hspace{6mm}
\includegraphics[scale=0.38]{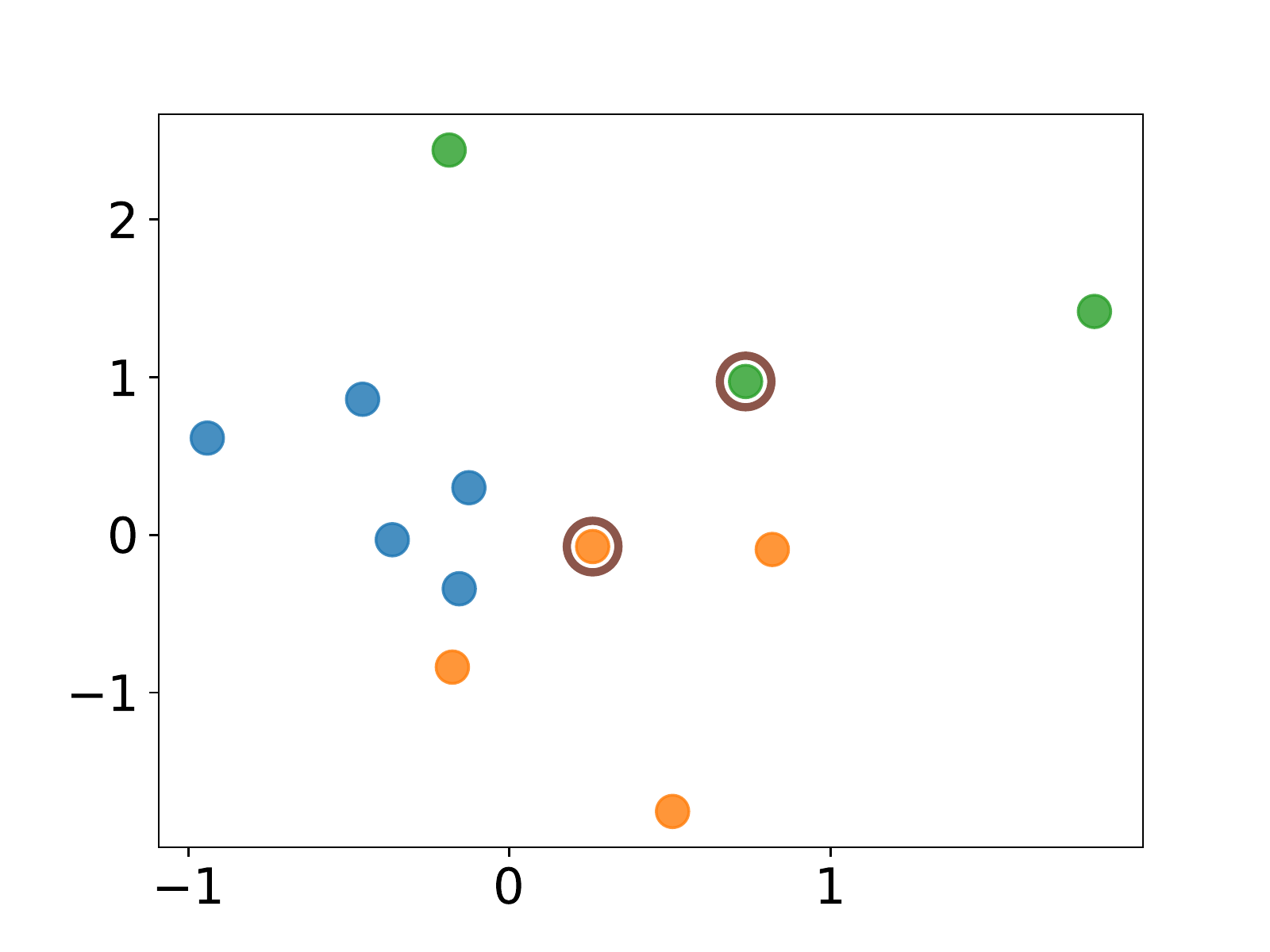}

\caption{An example illustrating why the local search idea outlined in Section~\ref{section_experiments_general} does not work.
\textbf{Top left:} 12 points in $\R^2$. 
\textbf{Top right:} A $k$-means clustering of the 12 points (encoded by color) with two points that are not IP-stable (surrounded by a circle). 
\textbf{Bottom row:} After assigning one of the two points that are not stable in the $k$-means clustering to its closest cluster, that point is stable. However, 
now some points that were initially stable are not stable anymore.}\label{fig_counterexample_local_search}
\end{figure}

\subsection{Compatibility of Group Fairness and IP-Stability}\label{example_group_fair_vs_indi_fair}

By means of a simple example we want to illustrate that it really depends on the data set whether group fairness and IP-stability are compatible 
or at odds with each other. Here we consider the prominent 
group fairness notion for clustering of \citet{fair_clustering_Nips2017}, which 
asks that in each cluster, every demographic group is approximately equally represented. Let us assume that the data set
consists of the four 1-dimensional points 0, 1, 7 and 8 and the distance function $d$ is the ordinary Euclidean metric.
It is easy to see that the only IP-stable 2-clustering is $\mathcal{C}=(\{0,1\}, \{7,8\})$. Now if there are two demographic 
groups $G_1$ and $G_2$ with $G_1=\{0,7\}$ and $G_2=\{1,8\}$, the clustering~$\mathcal{C}$ is perfectly fair according to the notion 
of \citeauthor{fair_clustering_Nips2017}. But if $G_1=\{0,1\}$ and $G_2=\{7,8\}$, the clustering~$\mathcal{C}$
is totally unfair according to the 
their 
notion.

\subsection{Why Local Search Does not Work}\label{example_local_search}

Figure~\ref{fig_counterexample_local_search} presents an example illustrating why the local search idea outlined in 
Section~\ref{section_experiments_general} does not work: assigning a data point that is not IP-stable to its closest cluster (so that that data point becomes stable) 
may cause other data points that are initially stable to not be stable anymore  
after the reassignment.

\subsection{Heuristic Approach to Make Linkage Clustering More Stable}\label{appendix_pruning_strategy}

Linkage clustering 
builds 
a binary tree that represents a hierarchical clustering with the root of the tree corresponding to the whole data set and 
every 
node 
corresponding to a subset such that a parent is the union of its two children.   
The leaves of the tree correspond to singletons comprising one data point 
\citep[e.g.,][Section 22.1]{shalev2014understanding}. If one wants to obtain a $k$-clustering of the data set, 
the output of a linkage clustering 
algorithm is a certain pruning of this tree.

With {\ipst} being our primary goal, 
we propose to construct a $k$-clustering / a pruning of the tree as follows:  
starting with the two children of the root, 
we maintain a set of 
nodes 
that corresponds to a clustering 
and proceed in $k-2$ rounds. In round~$i$, we greedily 
split 
one of the $i+1$ many 
nodes 
that we currently have into its two children such that the resulting $(i+2)$-clustering minimizes, over the $i+1$ many possible splits, 
 $\Nrunf$ (the number of non-stable datapoints; cf. the beginning of Section~\ref{section_experiments}). Alternatively, we can split the node that 
gives rise to a minimum value of 
$\Maxviol$ (cf. the beginning of Section~\ref{section_experiments}). 
Algorithm~\ref{pseudocode_pruning_strategy} provides the pseudocode of our proposed strategy. 
%

\begin{algorithm}[H]
   \caption{Algorithm to greedily prune a hierarchical clustering}
   \label{pseudocode_pruning_strategy}
   
\begin{algorithmic}[1]
\vspace{1mm}
     \STATE {\bfseries Input:}  
       binary tree $T$ representing a hierarchical clustering obtained from running a linkage clustering algorithm; number of clusters~$k\in\{2,\ldots,|\dataset|\}$; 
       measure $meas\in\{\Nrunf,\Maxviol\}$ that one aims to optimize~for

\vspace{1mm}  
   \STATE {\bfseries Output:}
   a $k$-clustering $\mathcal{C}$
   
\vspace{3mm}
\STATE{\emph{\# Conventions:} 
\vspace{0mm}
\begin{itemize}[leftmargin=*]
\setlength{\itemsep}{1pt}
 \item \emph{for a node $v\in T$, we denote the left child of $v$ by $Left(v)$ and the right child by $Right(v)$}
\item 
\emph{for a $j$-clustering~$\mathcal{C}'=(C_1,C_2,\ldots,C_j)$, a cluster $C_l$ and $A,B\subseteq C_l$ with $A\dot{\cup} B =C_l$ we write 
$\mathcal{C}'|_{C_l\hookrightarrow A,B}$ for the $(j+1)-$clustering 
that we obtain by replacing the cluster $C_l$ with two clusters $A$ and $B$ in $\mathcal{C}'$}
\end{itemize}
}   
  
\vspace{4mm}
\STATE{Let $r$ be the root of $T$ and initialize the clustering $\mathcal{C}$ as $\mathcal{C}=(Left(r),Right(r))$}
\vspace{2pt}
\FOR{$i=1$ {\bfseries to} $k-2$ \textbf{by} $1$} 
\vspace{2pt}
\STATE{Set $$v^\star=\argmin_{v: v\text{ is a cluster in $\mathcal{C}$ with }|v|>1} meas(\mathcal{C}|_{v\hookrightarrow Left(v),Right(v)})$$ 
and
$\mathcal{C}=\mathcal{C}|_{v^\star\hookrightarrow Left(v^\star),Right(v^\star)}$
}

\vspace{4pt}
\ENDFOR

\RETURN~~$\mathcal{C}$
\end{algorithmic}
\end{algorithm}

\clearpage
\subsection{
Experiments on the Adult Data Set}\label{appendix_exp_general_adult}

\vspace{5mm}
\begin{figure*}[h]
\centering
\includegraphics[width=0.74\textwidth]{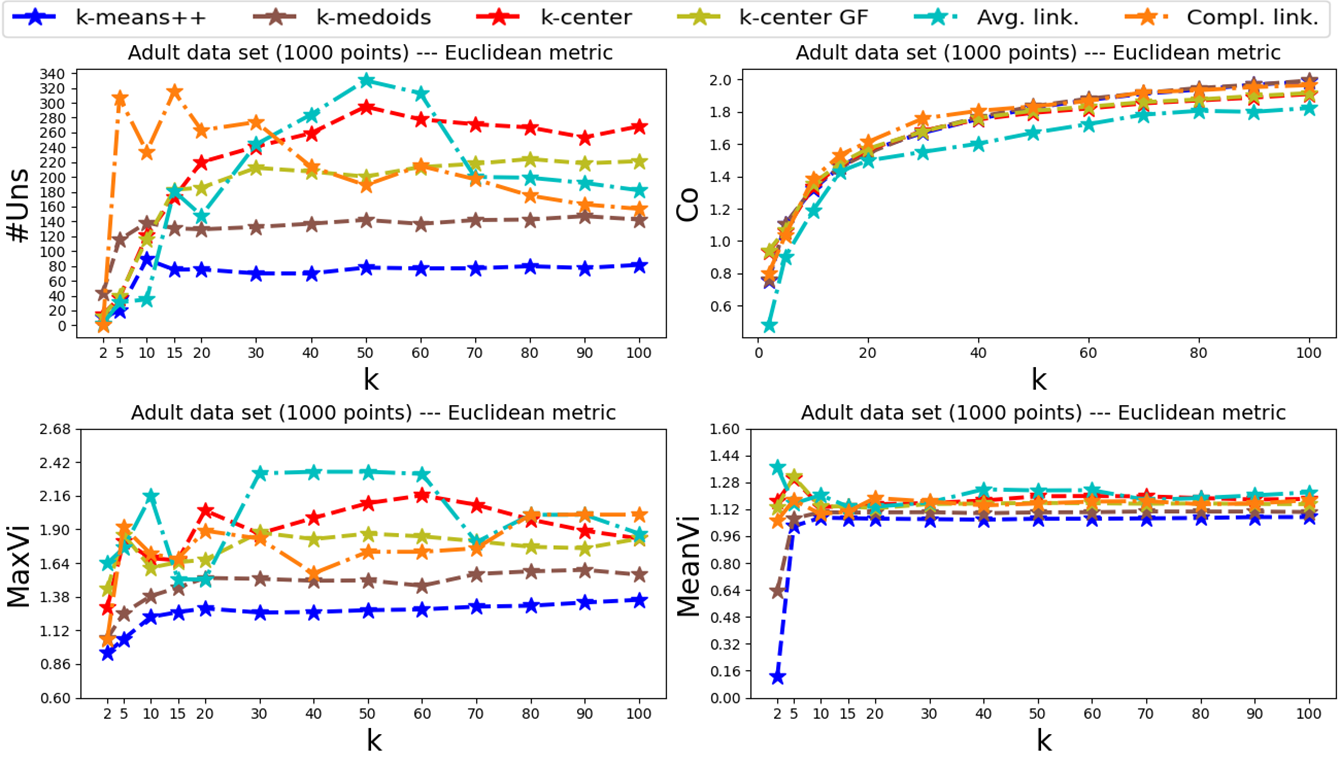}
\caption{Adult data set --- plots of Figure~\ref{exp_gen_standard_alg_Adult}: 
$\Nrunf$ \textbf{(top-left)}, 
$\cost$ \textbf{(top-right)},
$\Maxviol$ \textbf{(bottom-left)} 
and $\Meanviol$ \textbf{(bottom-right)}
for the clusterings produced by the various standard algorithms as a function of 
$k$.
}
\end{figure*}

\vspace{12mm}
\begin{figure*}[h]
\centering
\includegraphics[width=0.74\textwidth]{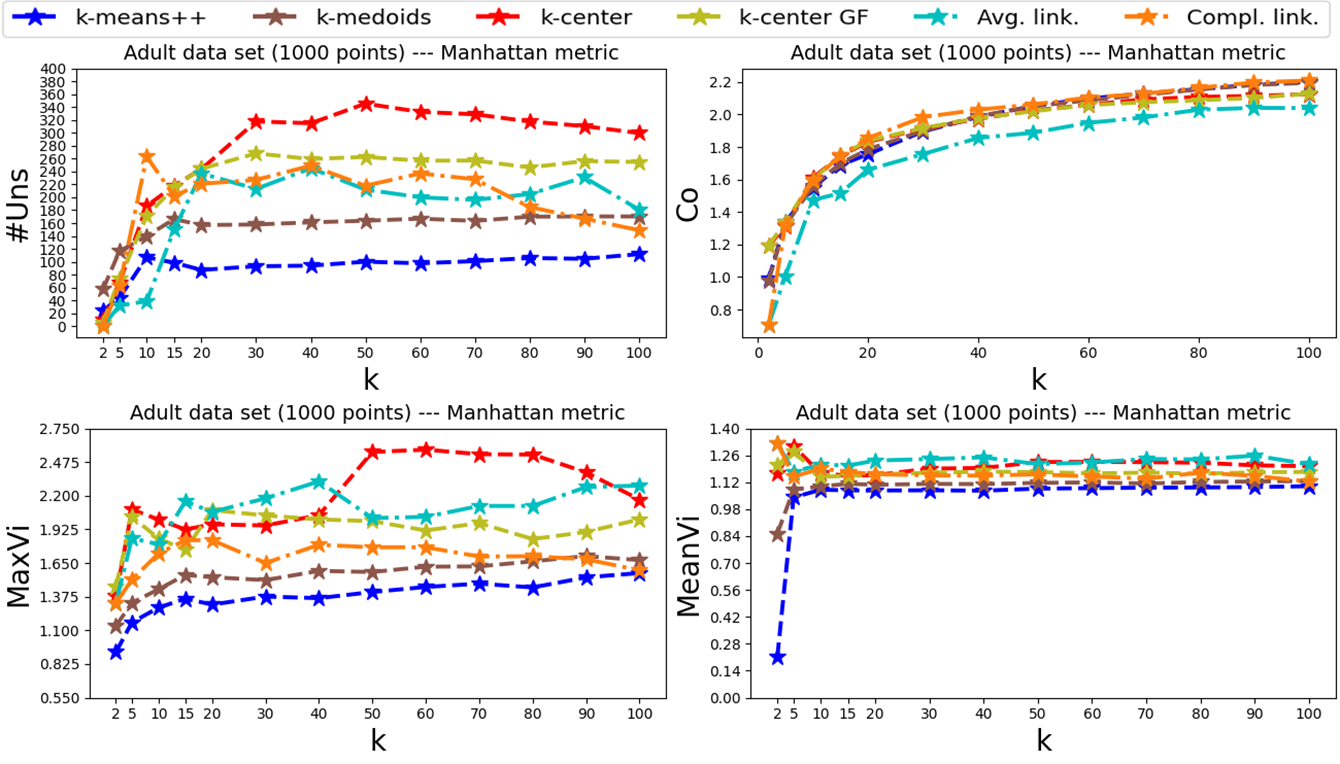}
\caption{Adult data set --- similar plots as in Figure~\ref{exp_gen_standard_alg_Adult}, 
but for the Manhattan metric. 
}
\end{figure*}

\begin{figure*}[h]
\centering
\includegraphics[width=0.9\textwidth]{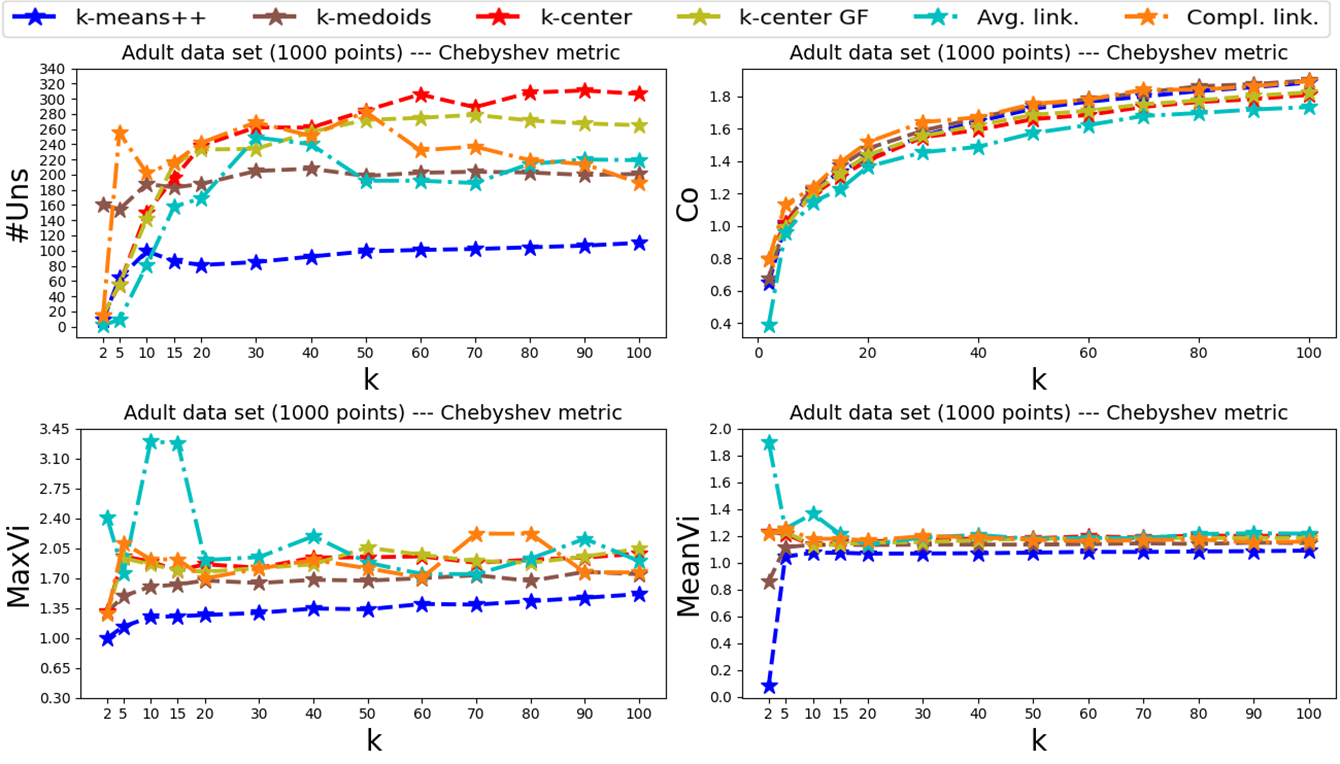}
\caption{Adult data set --- similar plots as in Figure~\ref{exp_gen_standard_alg_Adult}, but for the Chebyshev metric. 
}
\end{figure*}


\begin{figure}[h]
\centering
\includegraphics[width=0.9\textwidth]{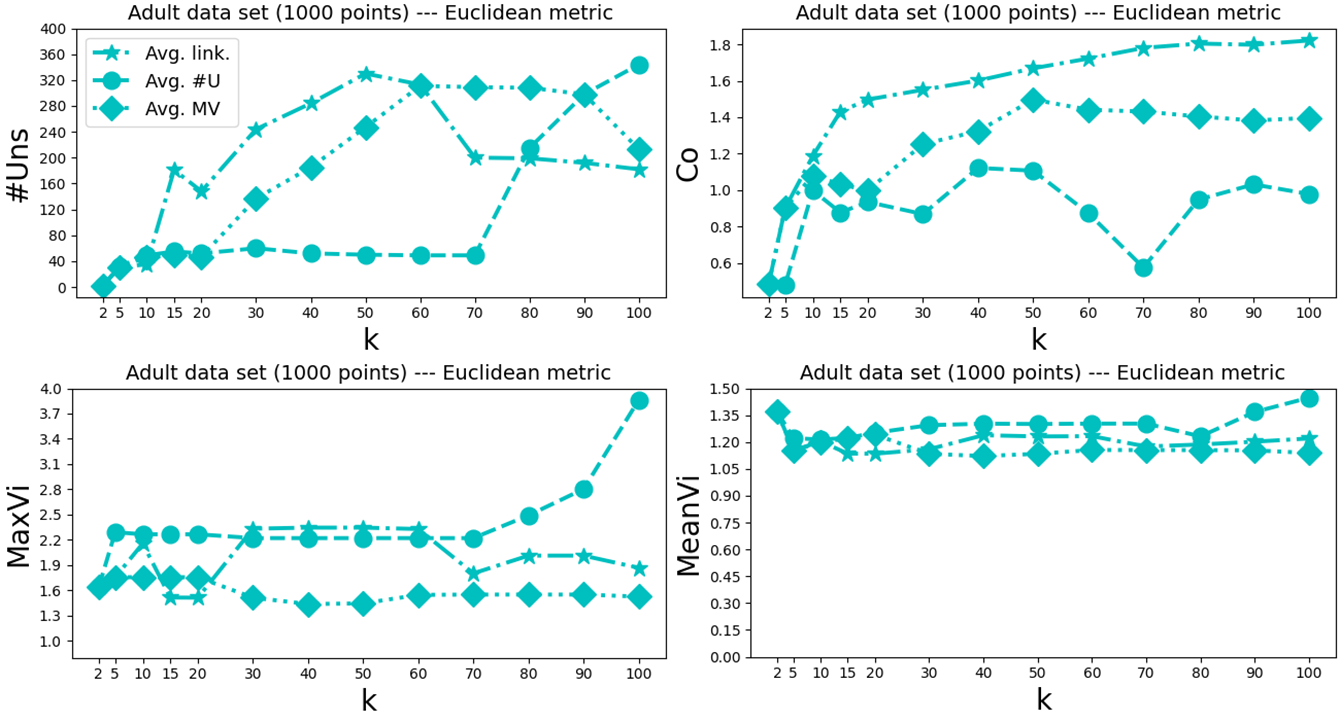}

\caption{Adult data set with Euclidean metric: 
$\Nrunf$  \textbf{(top-left)}, 
$\cost$ \textbf{(top-right)},
$\Maxviol$ \textbf{(bottom-left)}, 
and $\Meanviol$ \textbf{(bottom-right)}
for the clusterings produced by 
average linkage clustering 
and the two variants of our heuristic of Appendix~\ref{appendix_pruning_strategy} 
to improve it: the first ($\#$U in the legend) greedily chooses splits as to minimize $\Nrunf$, the second 
(MV)
as to minimize $\Maxviol$. 
}
\end{figure}

\begin{figure*}[h]
\centering
\includegraphics[width=0.9\textwidth]{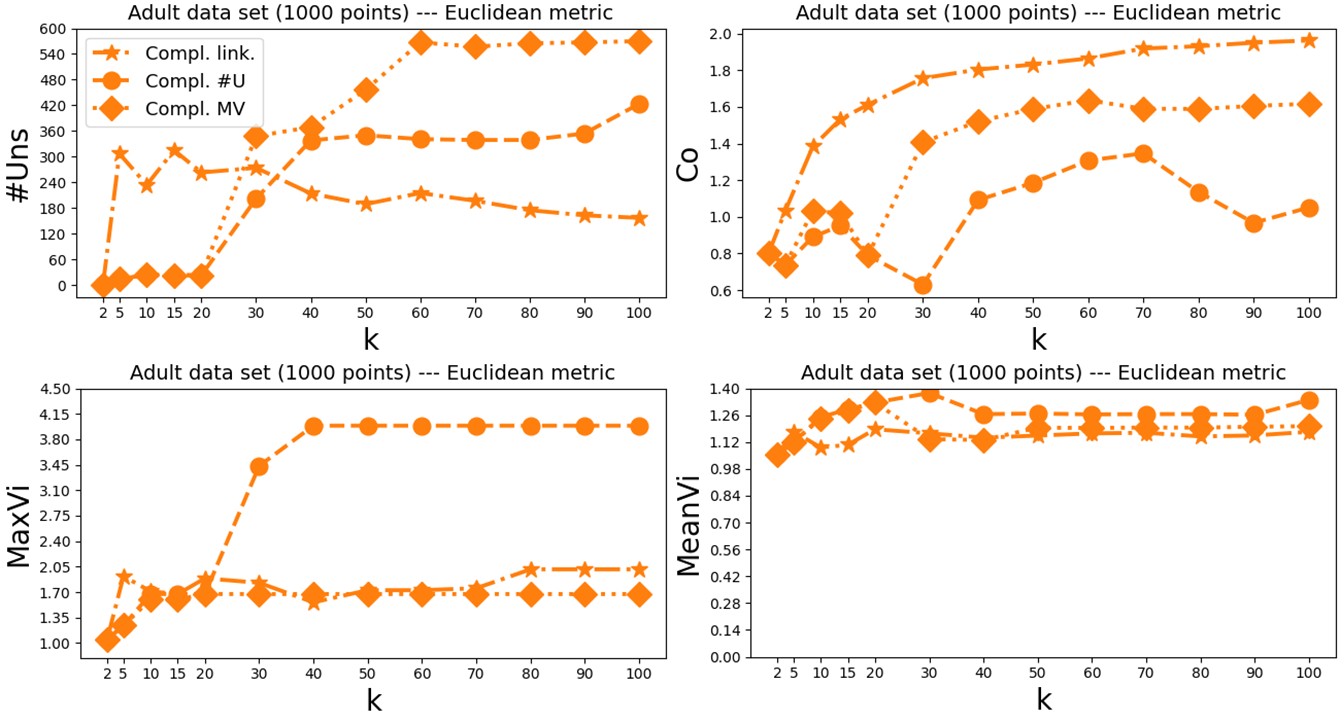}

\caption{Adult data set with Euclidean metric: 
the various measures 
for the clusterings produced by 
complete  linkage clustering 
and the two variants of our heuristic 
approach. 
}
\end{figure*}

\begin{figure*}[h]
\centering
\includegraphics[width=0.9\textwidth]{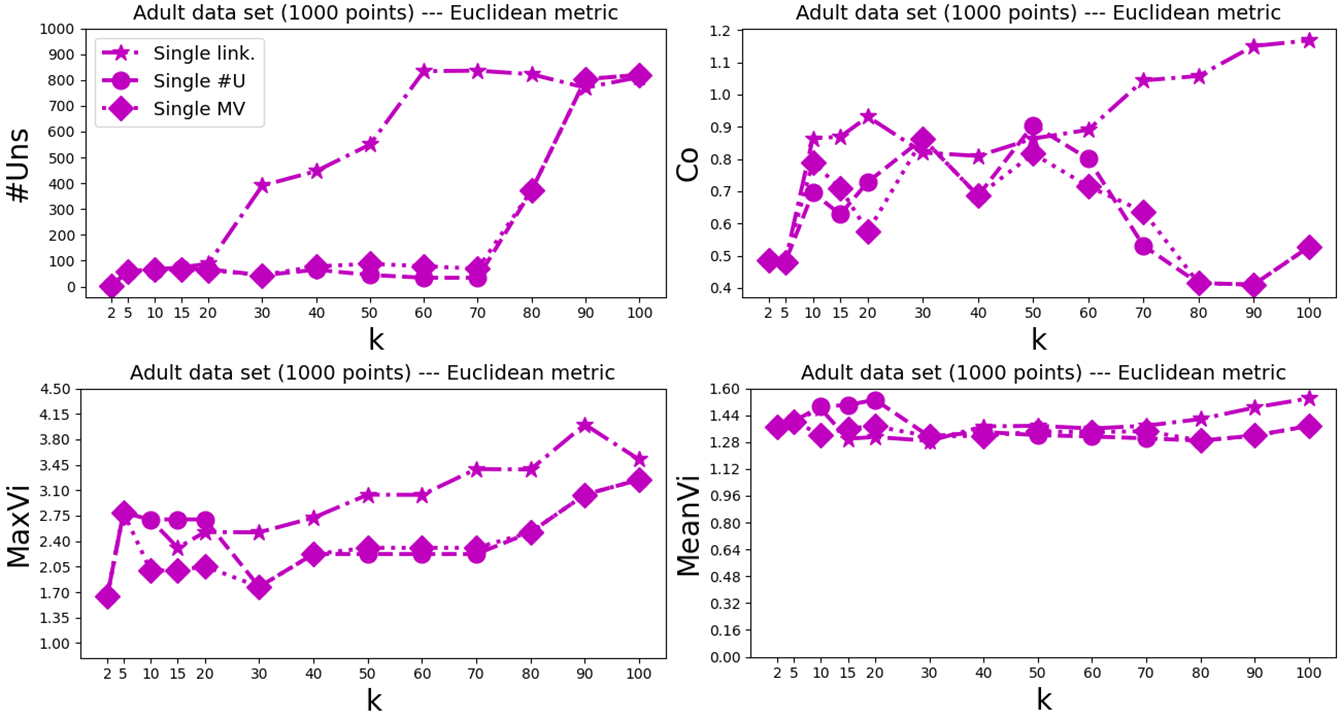}

\caption{Adult data set with Euclidean metric: 
the various measures 
for the clusterings produced by 
single  linkage clustering 
and the two variants of our heuristic 
approach.  
}
\end{figure*}

\begin{figure*}[h]
\centering
\includegraphics[width=0.9\textwidth]{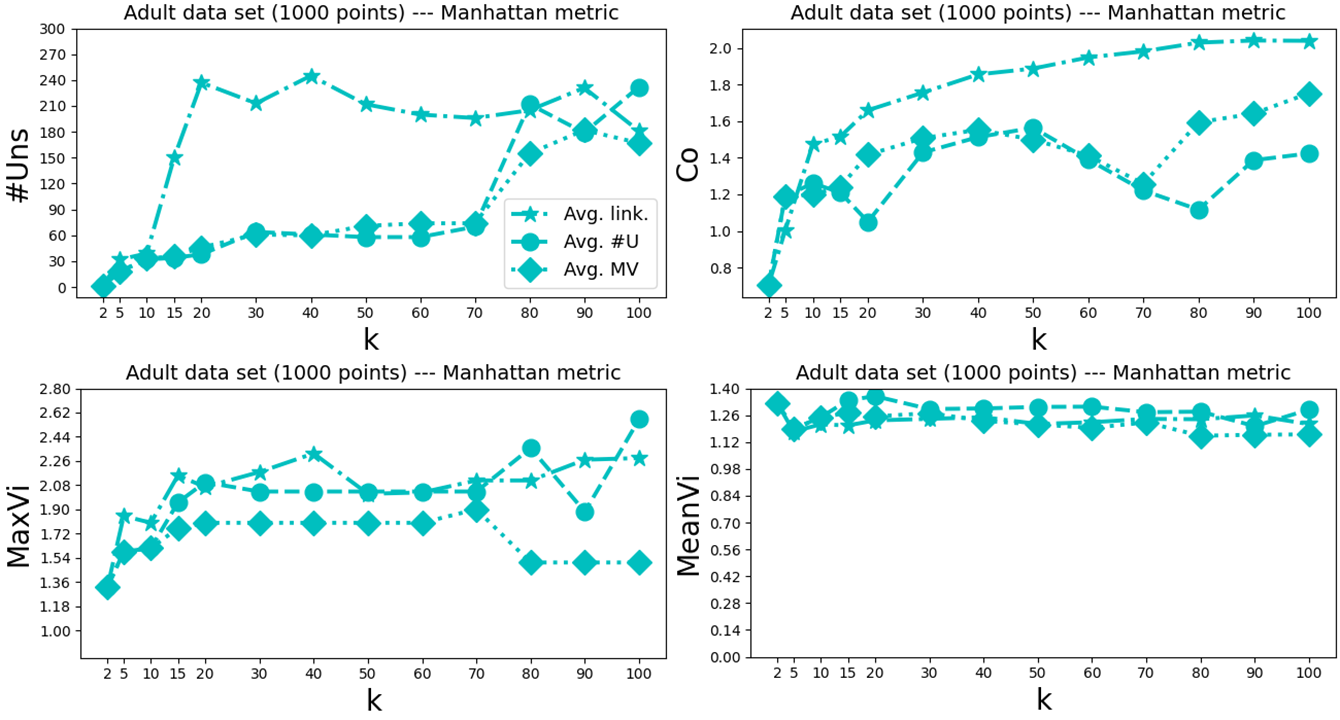}

\caption{Adult data set with Manhattan metric: 
the various measures 
for the clusterings produced by 
average  linkage clustering 
and the two variants of our heuristic 
approach. 
}
\end{figure*}

\begin{figure*}[h]
\centering
\includegraphics[width=0.9\textwidth]{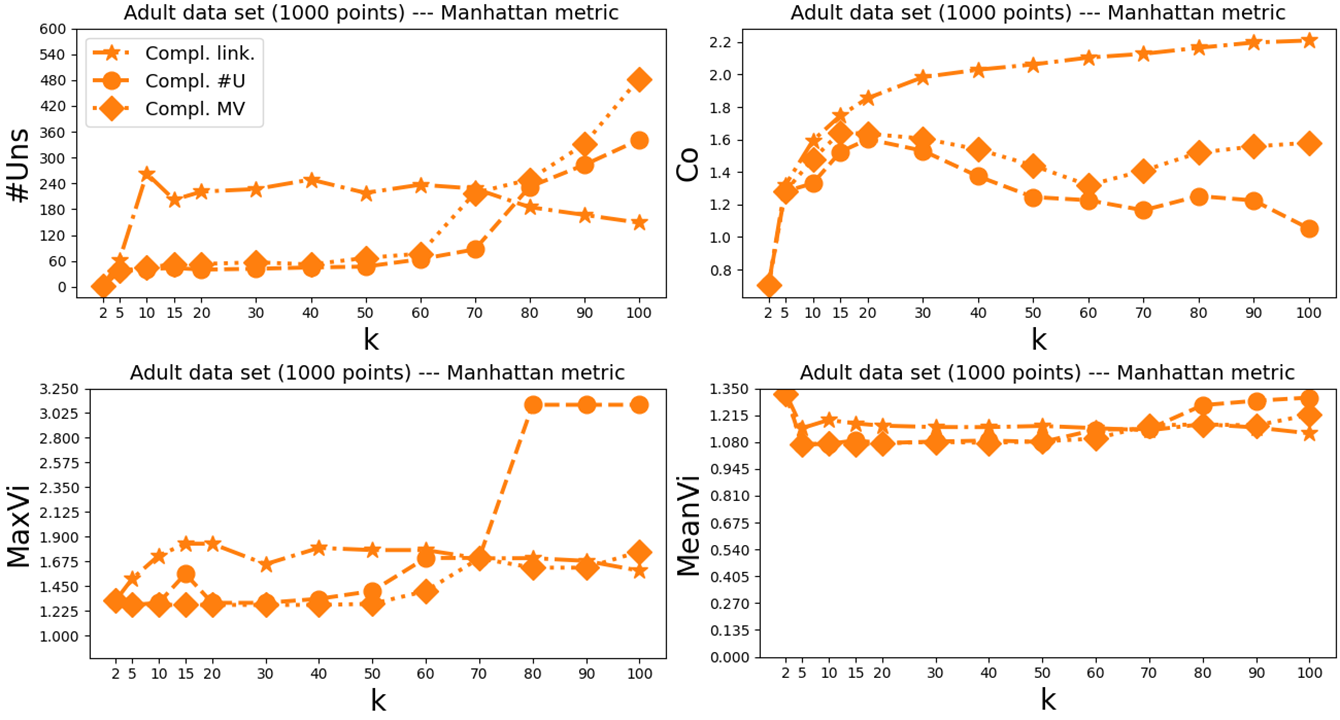}

\caption{Adult data set with Manhattan metric: 
the various measures 
for the clusterings produced by 
complete  linkage clustering 
and the two variants of our heuristic 
approach. 
}
\end{figure*}

\begin{figure*}[h]
\centering
\includegraphics[width=0.9\textwidth]{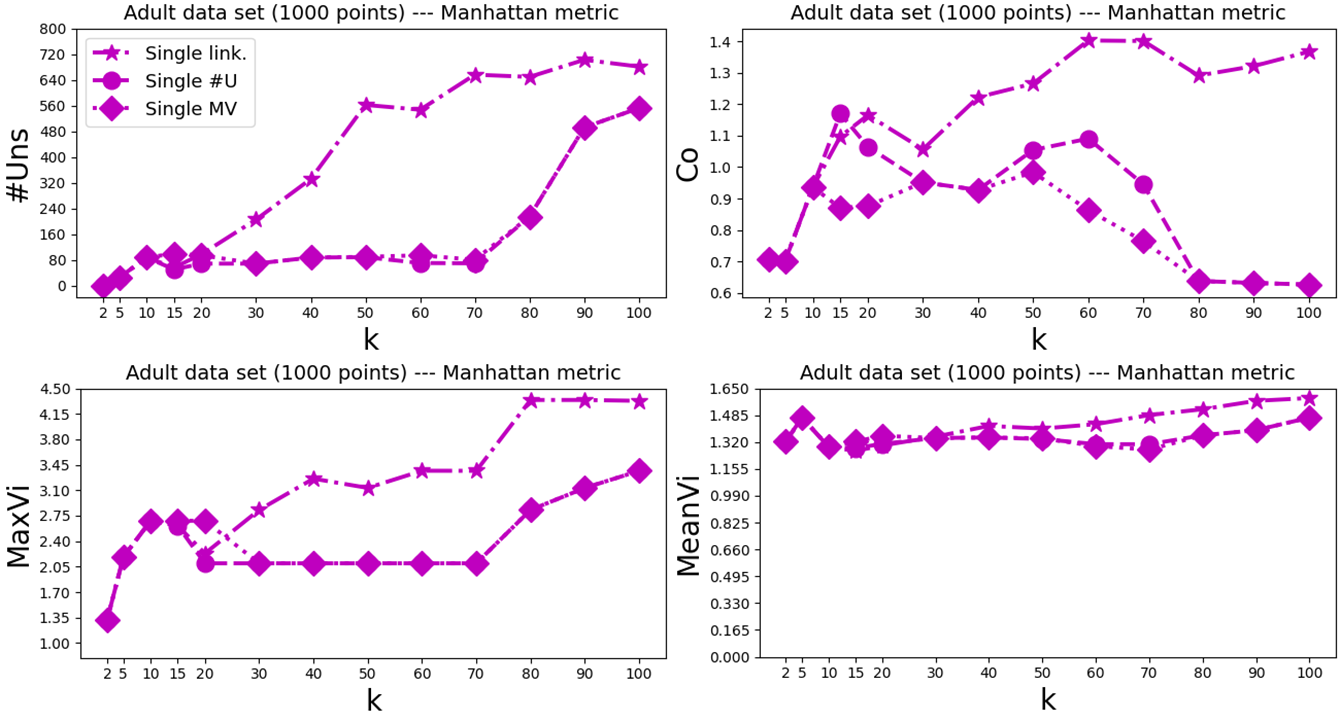}

\caption{Adult data set with Manhattan metric: 
the various measures 
for the clusterings produced by 
single  linkage clustering 
and the two variants of our heuristic 
approach. 
}
\end{figure*}

\begin{figure*}[h]
\centering
\includegraphics[width=0.9\textwidth]{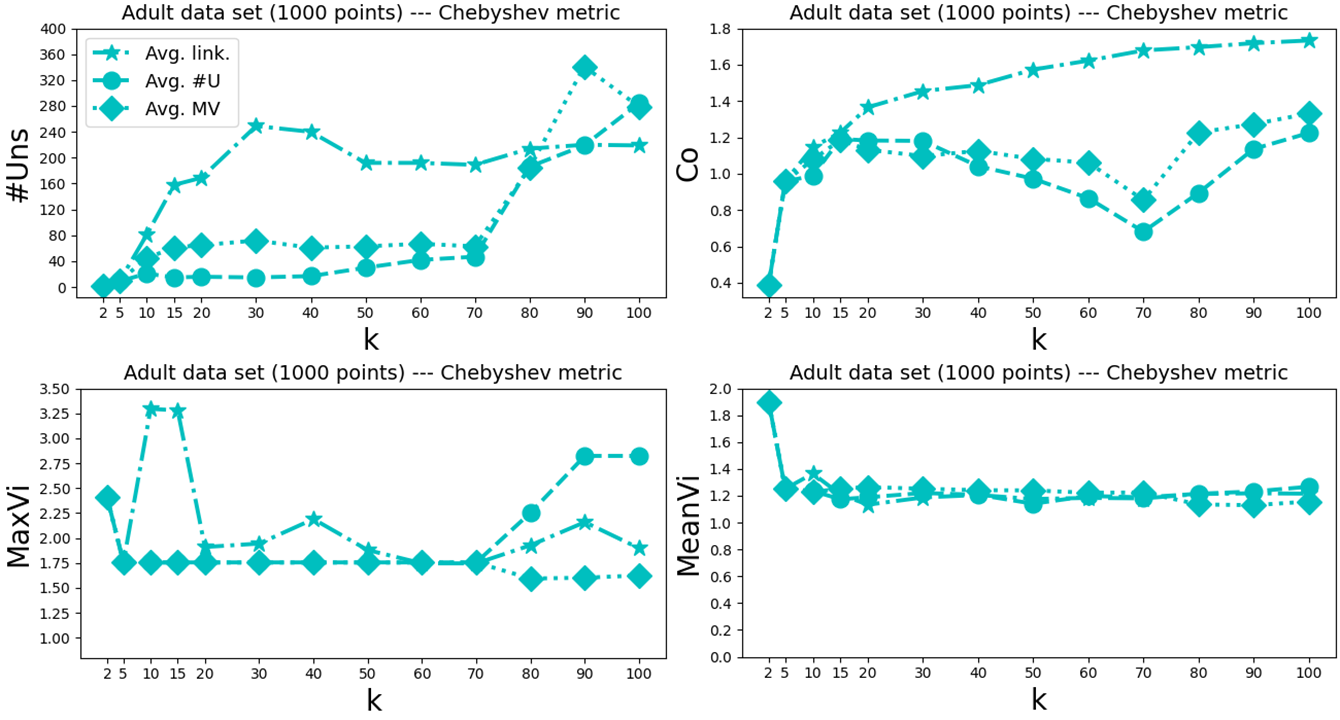}

\caption{Adult data set with Chebyshev metric: 
the various measures 
for the clusterings produced by 
average  linkage clustering 
and the two variants of our heuristic 
approach. 
}
\end{figure*}

\begin{figure*}[h]
\centering
\includegraphics[width=0.9\textwidth]{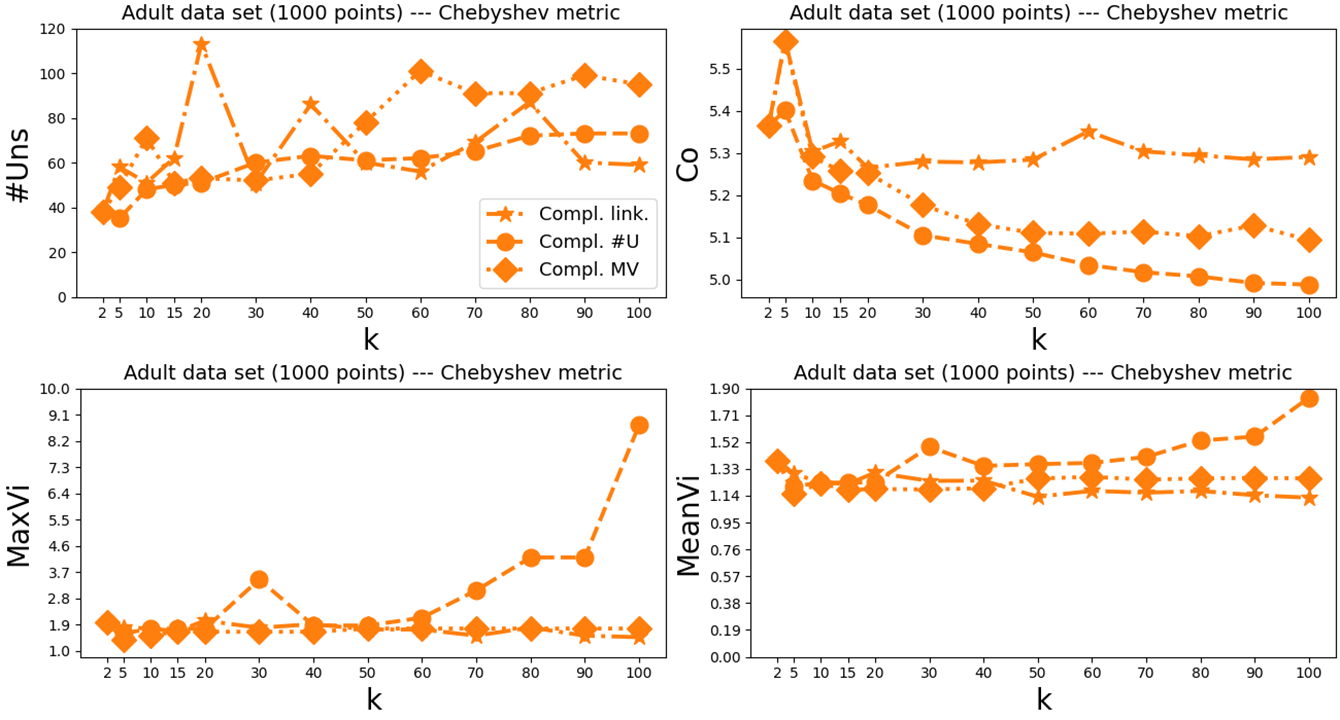}

\caption{Adult data set with Chebyshev metric: 
the various measures 
for the clusterings produced by 
complete  linkage clustering 
and the two variants of our heuristic 
approach. 
}
\end{figure*}

\begin{figure*}[h]
\centering
\includegraphics[width=0.9\textwidth]{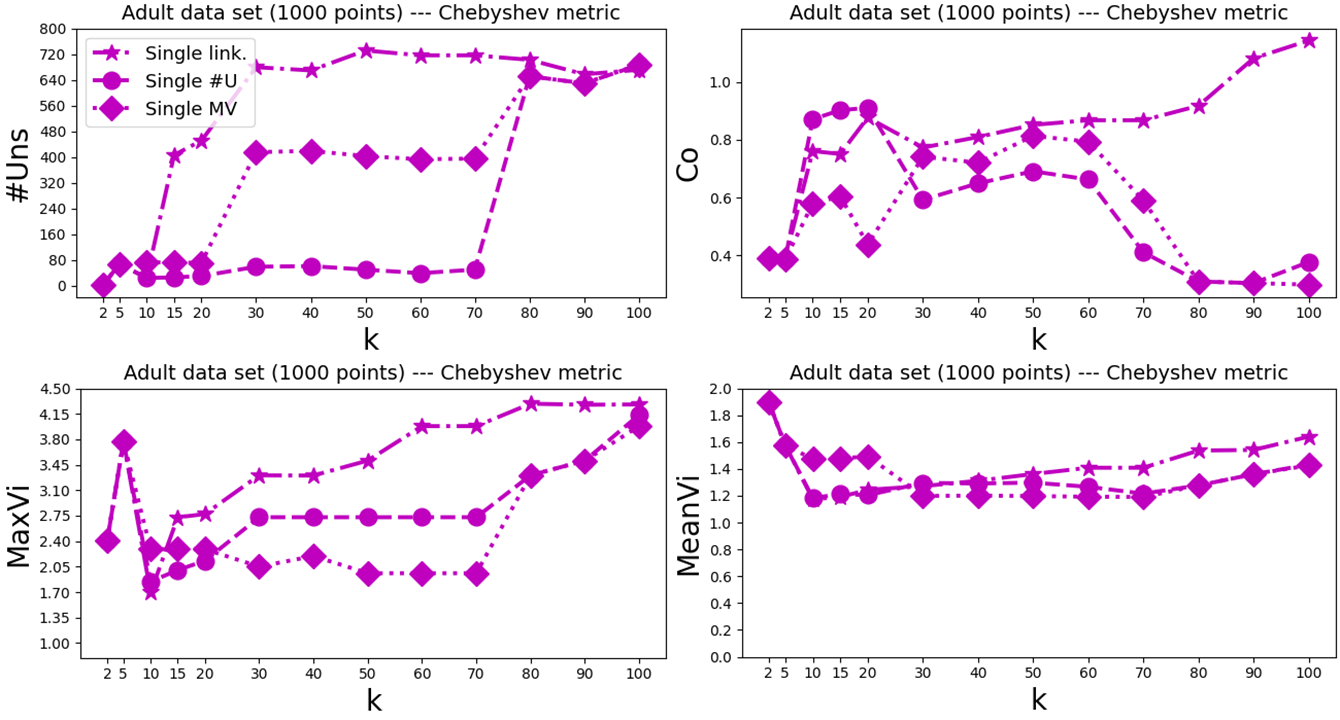}

\caption{Adult data set with Chebyshev metric: 
the various measures 
for the clusterings produced by 
single  linkage clustering 
and the two variants of our heuristic 
approach. 
}
\end{figure*}

\clearpage
\subsection{Experiments on the Drug Consumption Data Set}\label{appendix_exp_general_drug}

\vspace{5mm}
\begin{figure*}[h]
\centering
\includegraphics[width=0.74\textwidth]{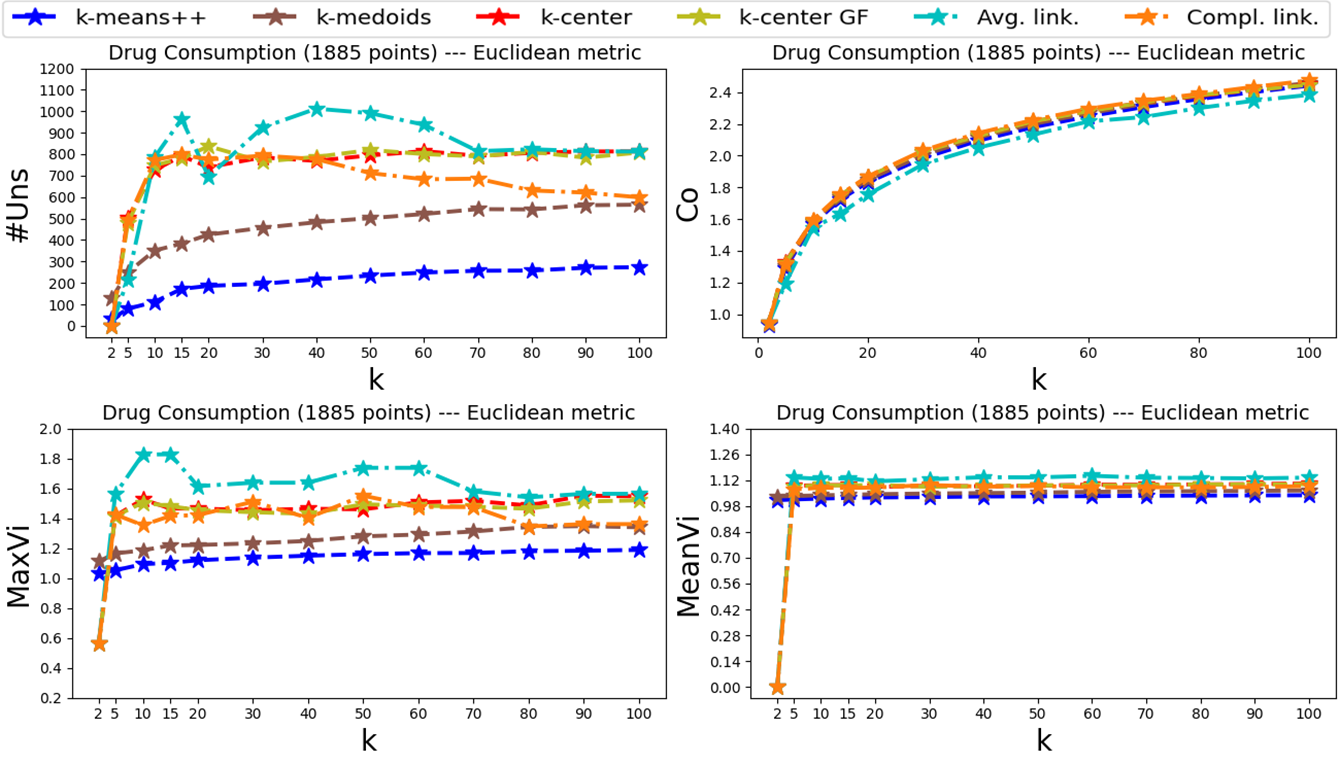}
\caption{Drug Consumption data set with Euclidean metric: 
$\Nrunf$ \textbf{(top-left)}, 
$\cost$ \textbf{(top-right)},
$\Maxviol$ \textbf{(bottom-left)} 
and $\Meanviol$ \textbf{(bottom-right)}
for the clusterings produced by the various standard algorithms as a function of 
 the number of clusters~$k$.
$k$.
}\label{exp_gen_standard_alg_Drug_Consumption_euclidean}
\end{figure*}

\vspace{12mm}
\begin{figure*}[h]
\centering
\includegraphics[width=0.74\textwidth]{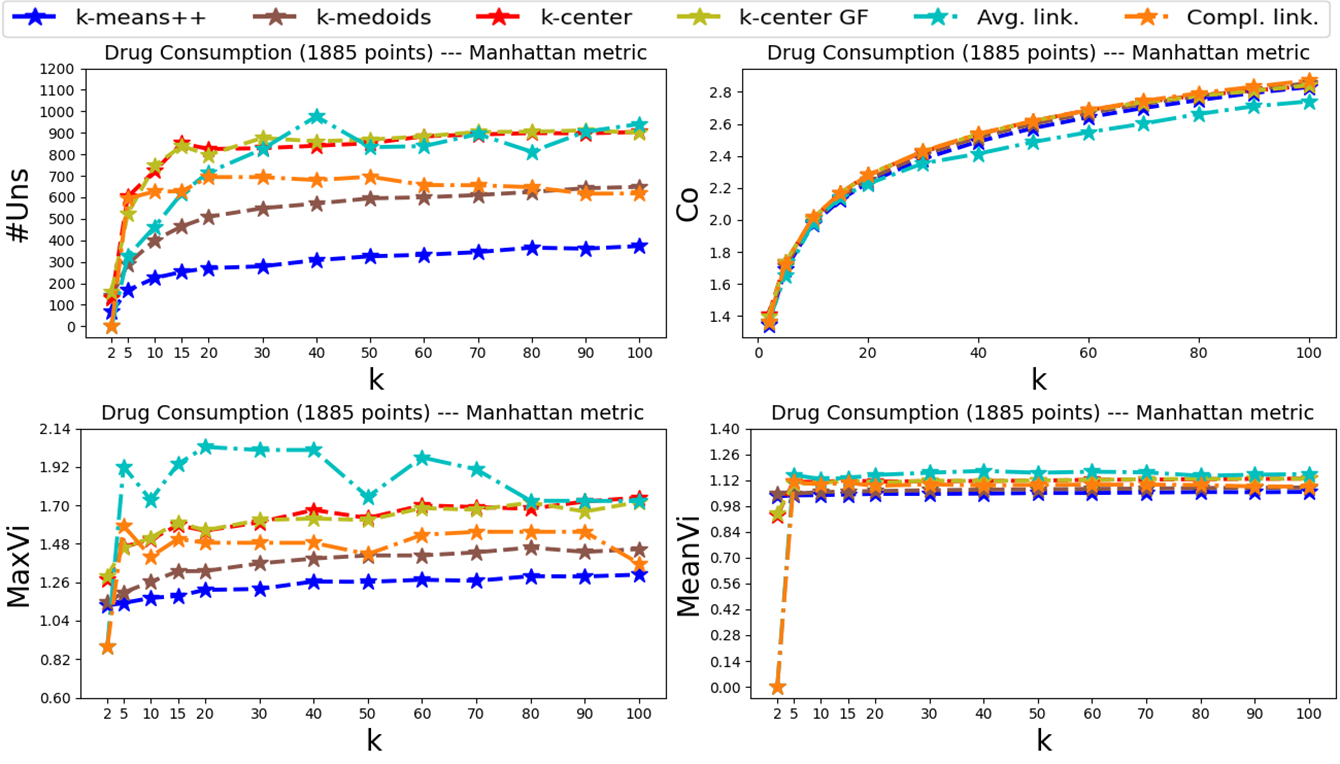}
\caption{Drug Consumption data set --- similar plots as in Figure~\ref{exp_gen_standard_alg_Drug_Consumption_euclidean}, 
but for the Manhattan metric. 
}
\end{figure*}

\begin{figure*}[h]
\centering
\includegraphics[width=0.9\textwidth]{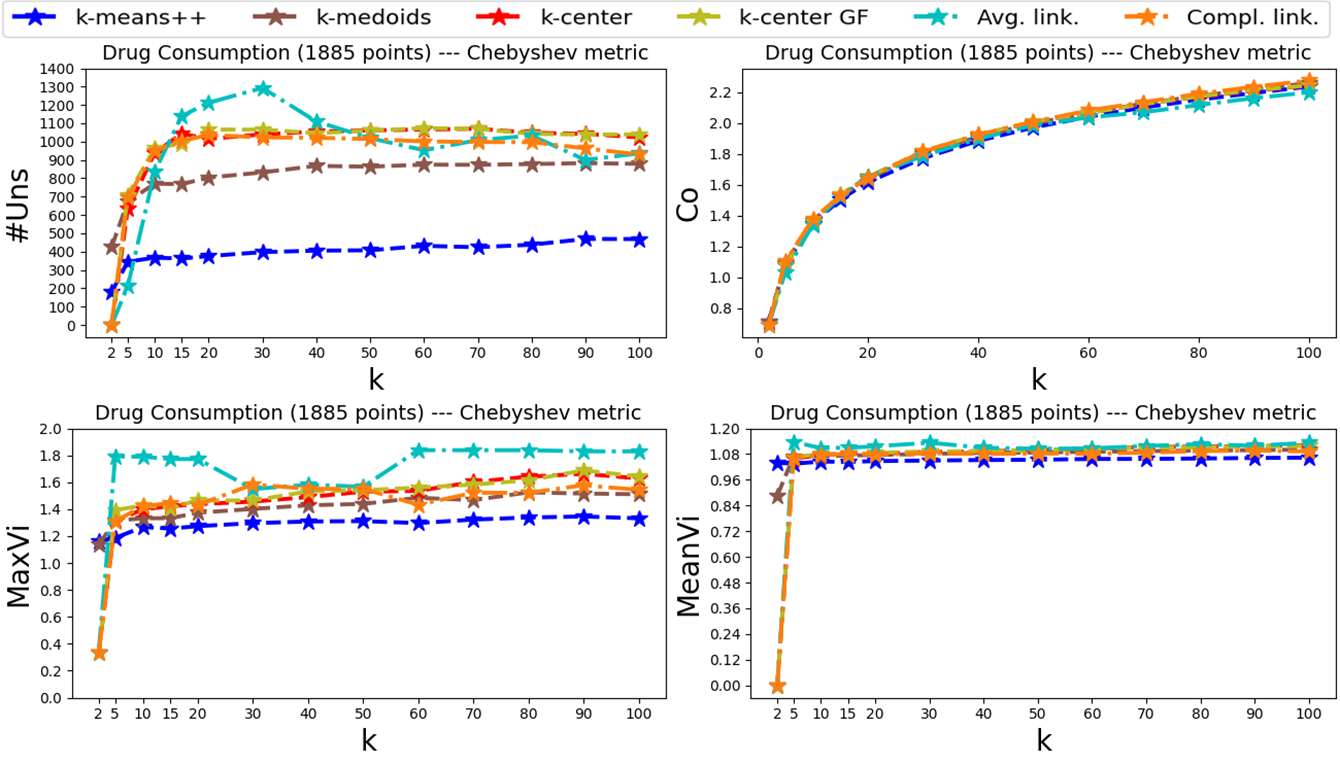}
\caption{Drug Consumption data set --- similar plots as in Figure~\ref{exp_gen_standard_alg_Drug_Consumption_euclidean}, but for the Chebyshev metric. 
}
\end{figure*}


\begin{figure}[h]
\centering
\includegraphics[width=0.9\textwidth]{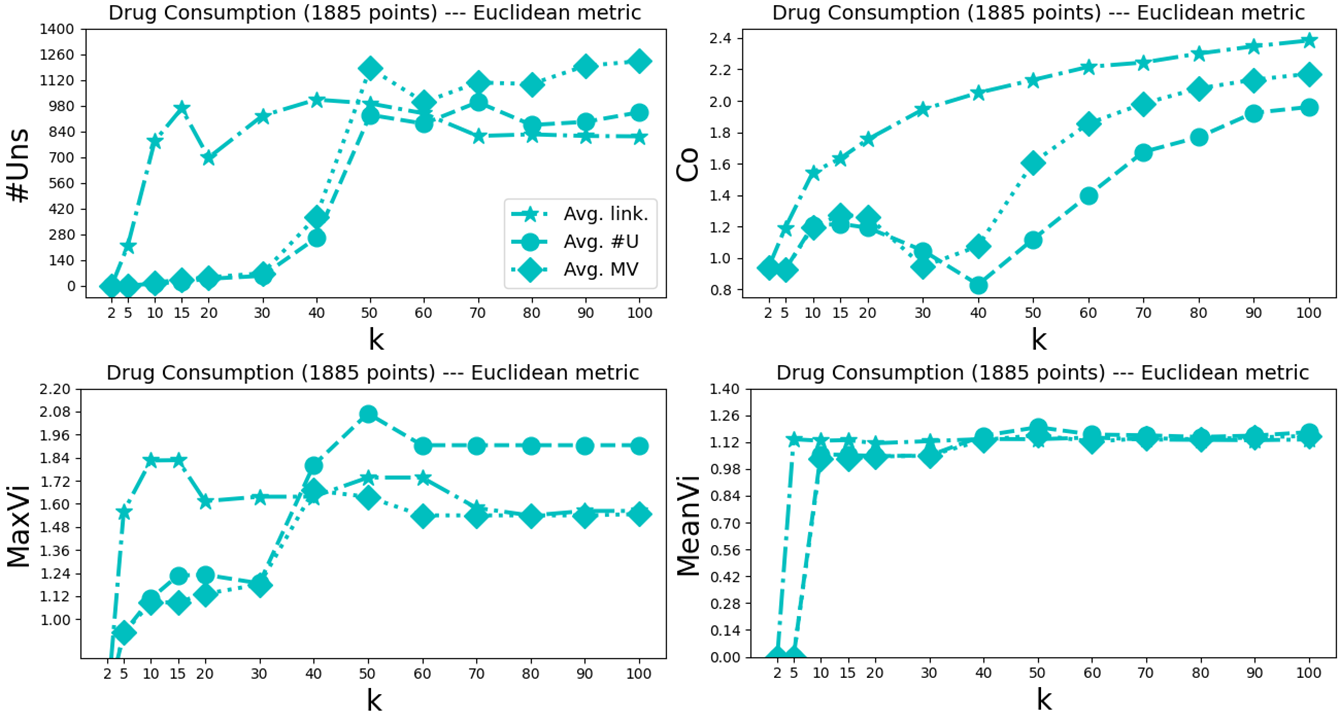}

\caption{Drug Consumption data set with Euclidean metric: 
$\Nrunf$  \textbf{(top-left)}, 
$\cost$ \textbf{(top-right)},
$\Maxviol$ \textbf{(bottom-left)}, 
and $\Meanviol$ \textbf{(bottom-right)}
for the clusterings produced by 
average linkage clustering 
and the two variants of our heuristic of Appendix~\ref{appendix_pruning_strategy} 
to improve it: the first ($\#$U in the legend) greedily chooses splits as to minimize $\Nrunf$, the second 
(MV)
as to minimize $\Maxviol$. 
}
\end{figure}

\begin{figure*}[h]
\centering
\includegraphics[width=0.9\textwidth]{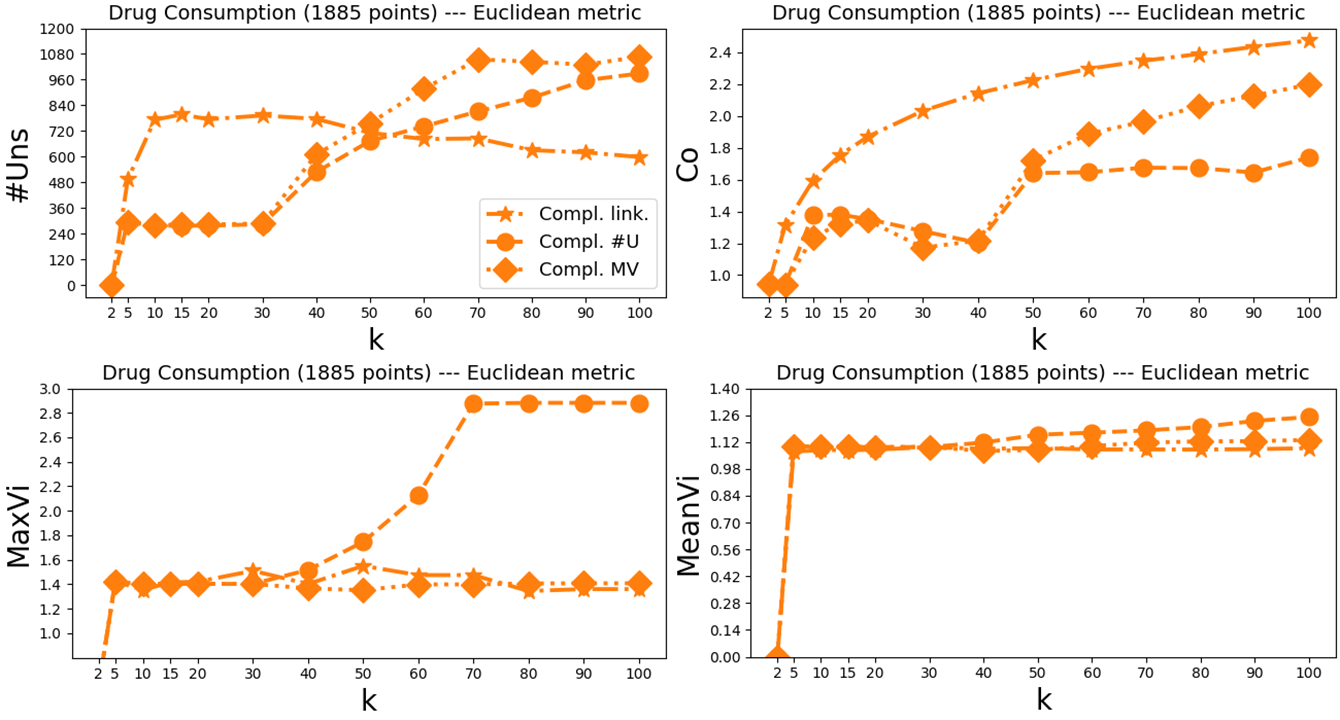}

\caption{Drug Consumption data set with Euclidean metric: 
the various measures 
for the clusterings produced by 
complete  linkage clustering 
and the two variants of our heuristic 
approach. 
}
\end{figure*}

\begin{figure*}[h]
\centering
\includegraphics[width=0.9\textwidth]{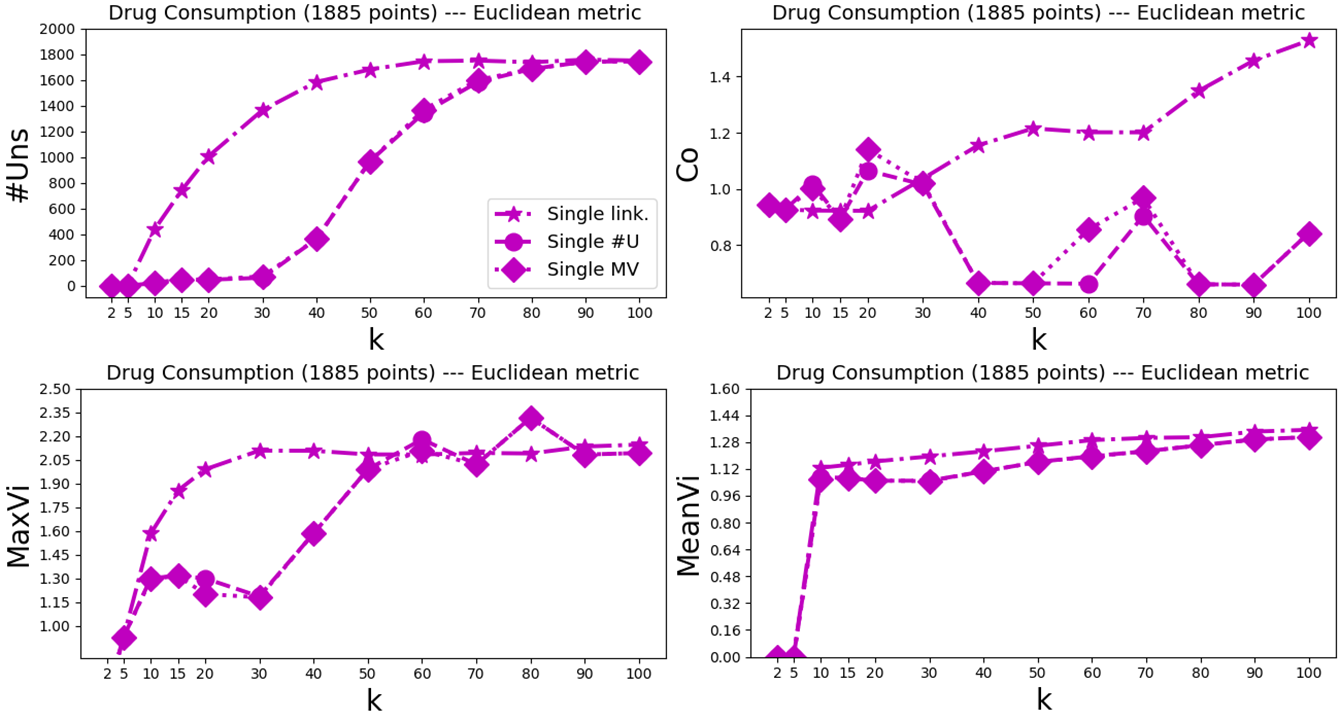}

\caption{Drug Consumption data set with Euclidean metric: 
the various measures 
for the clusterings produced by 
single  linkage clustering 
and the two variants of our heuristic 
approach.  
}
\end{figure*}

\begin{figure*}[h]
\centering
\includegraphics[width=0.9\textwidth]{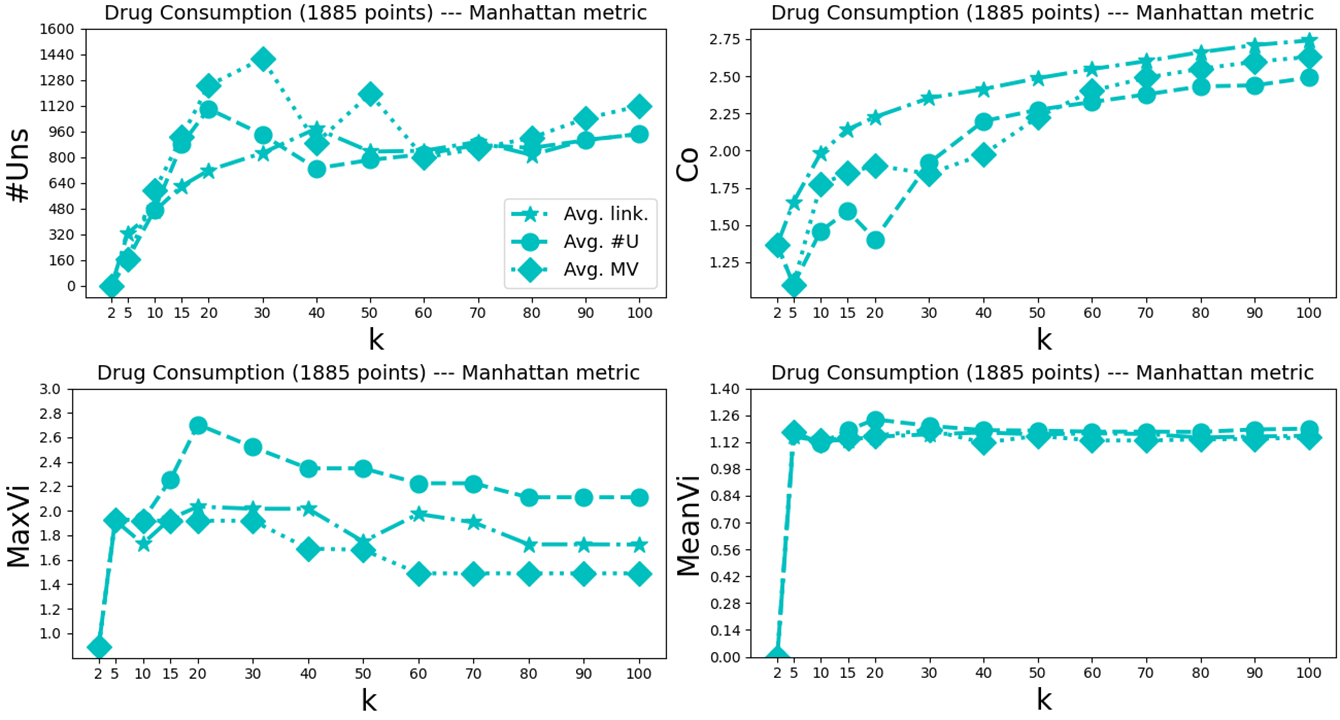}

\caption{Drug Consumption data set with Manhattan metric: 
the various measures 
for the clusterings produced by 
average  linkage clustering 
and the two variants of our heuristic 
approach. 
}
\end{figure*}

\begin{figure*}[h]
\centering
\includegraphics[width=0.9\textwidth]{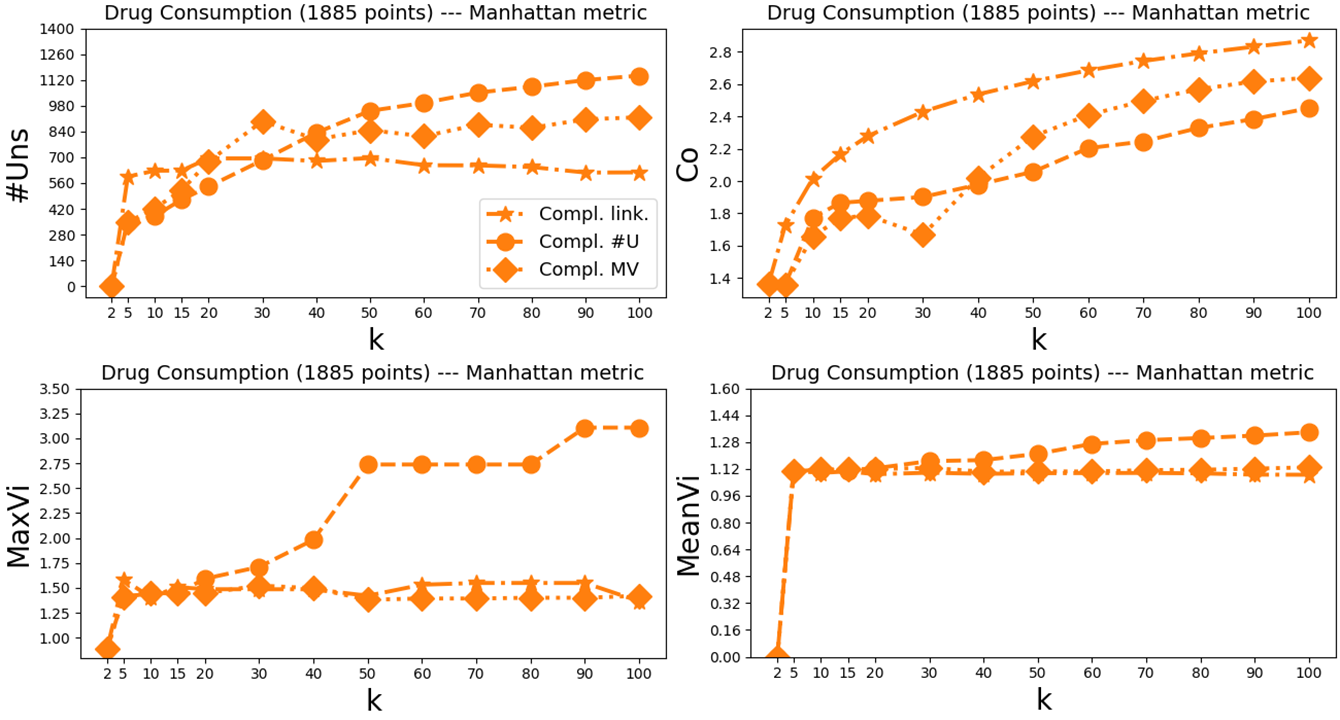}

\caption{Drug Consumption data set with Manhattan metric: 
the various measures 
for the clusterings produced by 
complete  linkage clustering 
and the two variants of our heuristic 
approach. 
}
\end{figure*}

\begin{figure*}[h]
\centering
\includegraphics[width=0.9\textwidth]{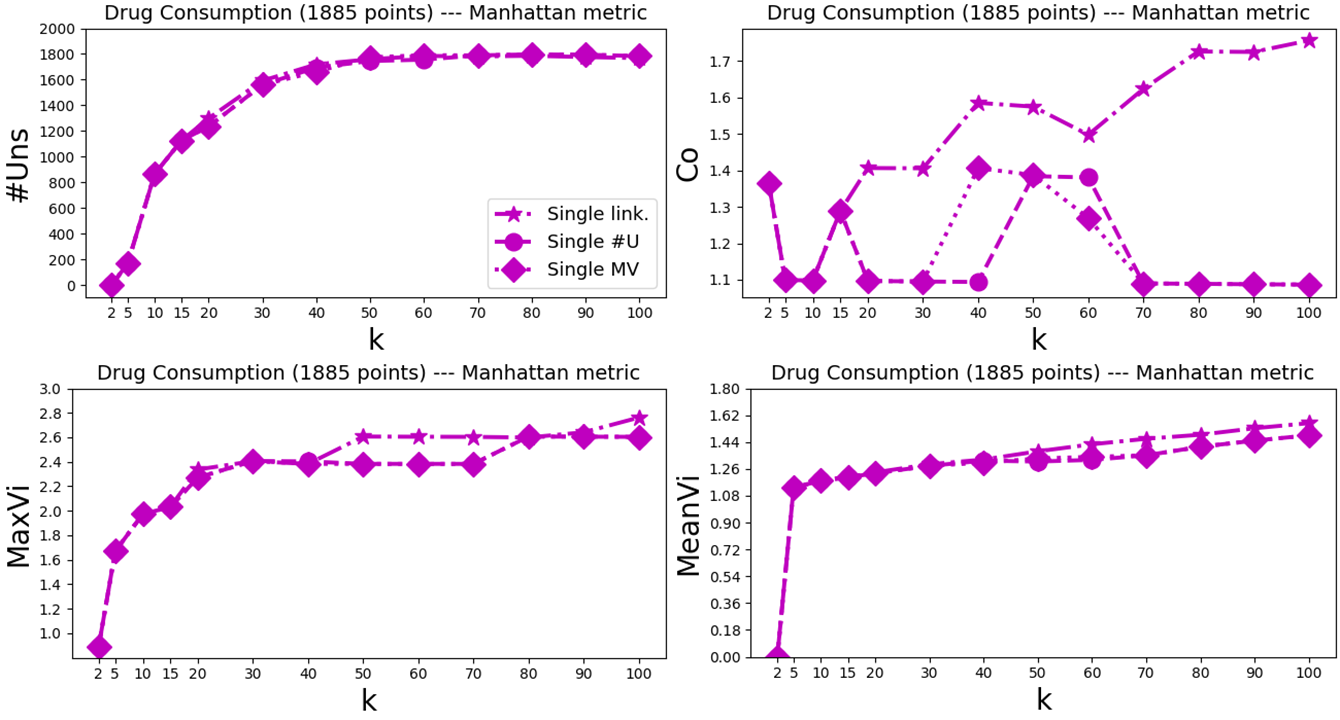}

\caption{Drug Consumption data set with Manhattan metric: 
the various measures 
for the clusterings produced by 
single  linkage clustering 
and the two variants of our heuristic 
approach. 
}
\end{figure*}

\begin{figure*}[h]
\centering
\includegraphics[width=0.9\textwidth]{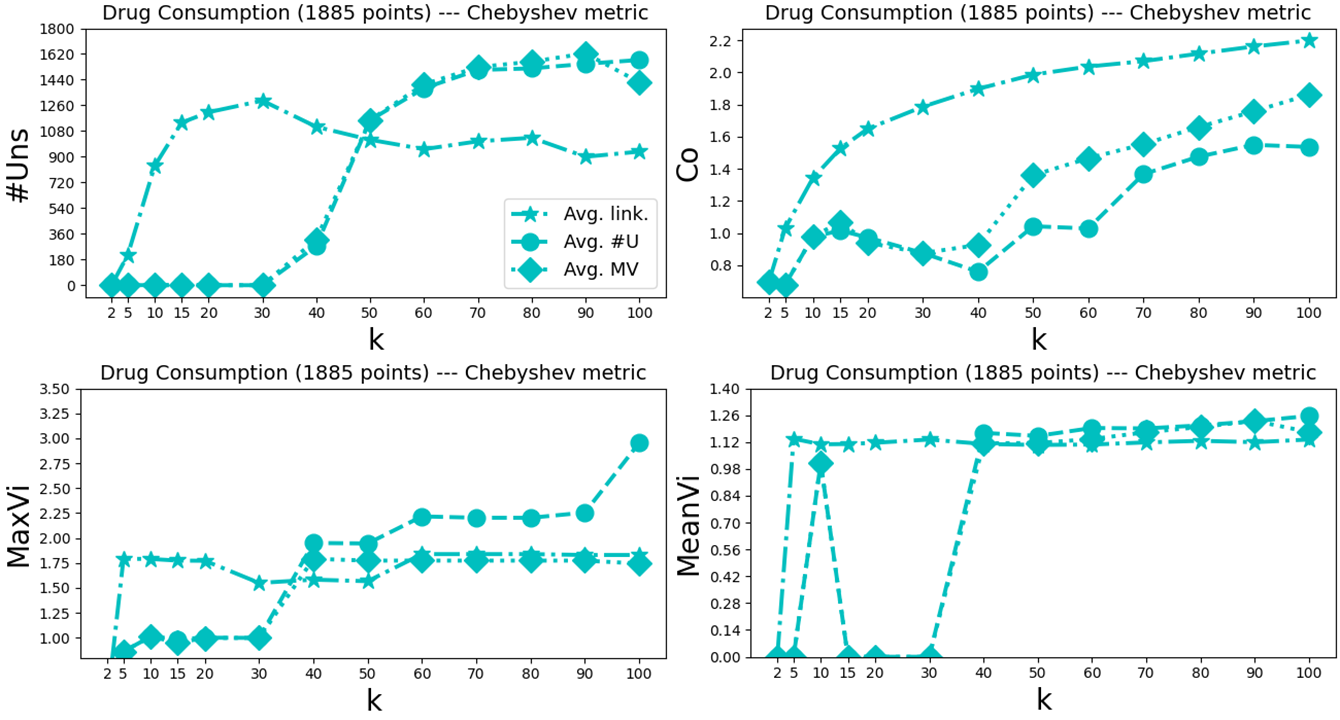}

\caption{Drug Consumption data set with Chebyshev metric: 
the various measures 
for the clusterings produced by 
average  linkage clustering 
and the two variants of our heuristic 
approach. 
}
\end{figure*}

\begin{figure*}[h]
\centering
\includegraphics[width=0.9\textwidth]{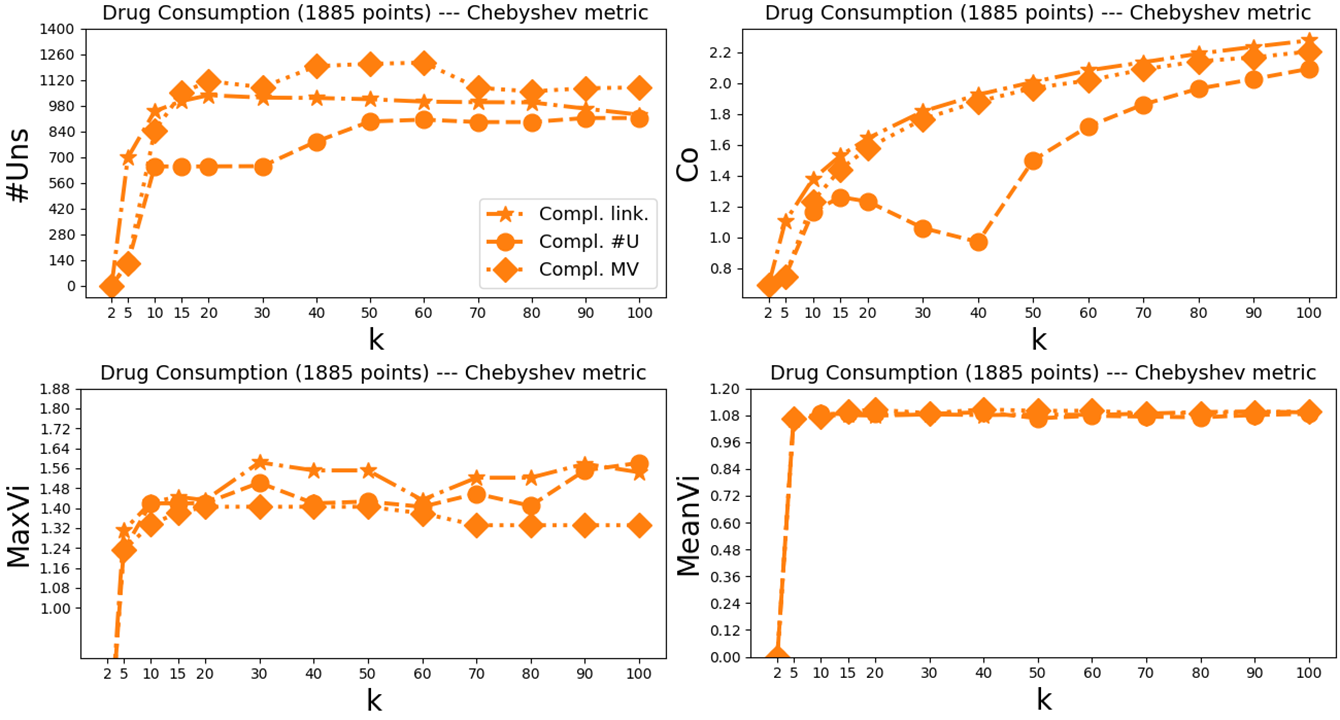}

\caption{Drug Consumption data set with Chebyshev metric: 
the various measures 
for the clusterings produced by 
complete  linkage clustering 
and the two variants of our heuristic 
approach. 
}
\end{figure*}

\begin{figure*}[h]
\centering
\includegraphics[width=0.9\textwidth]{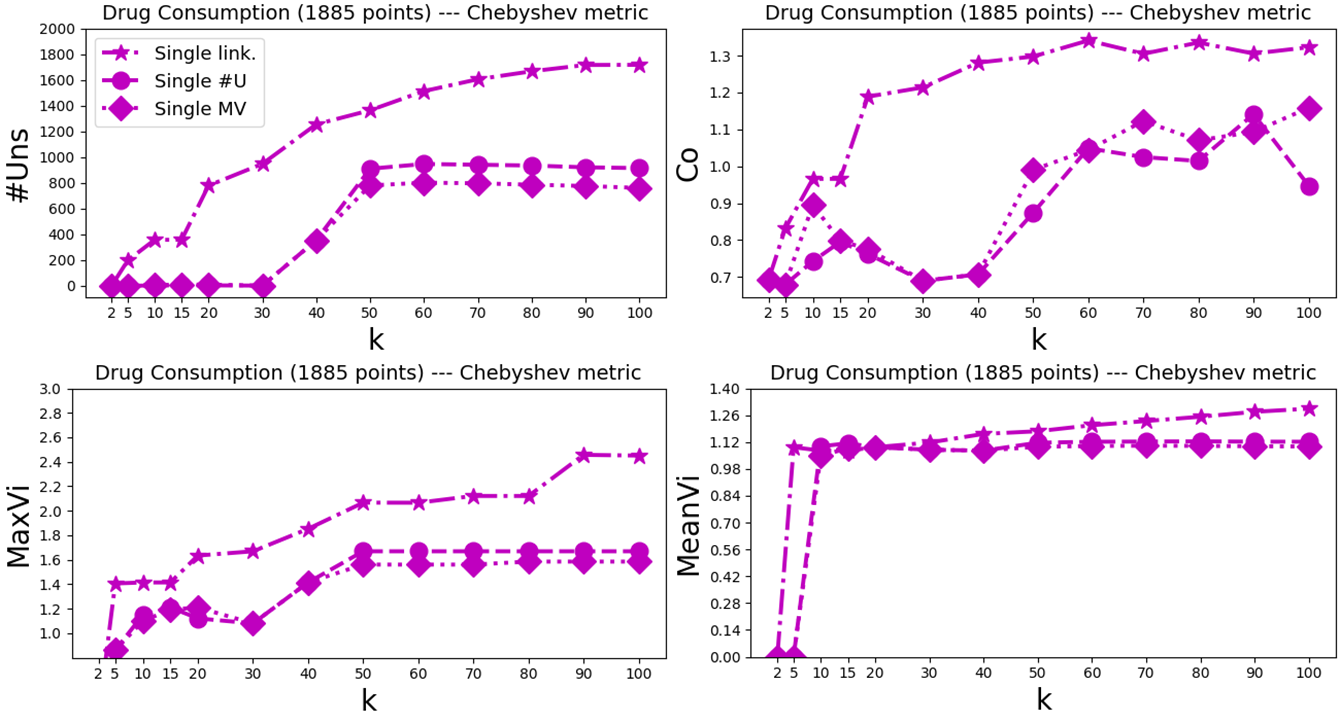}

\caption{Drug Consumption data set with Chebyshev metric: 
the various measures 
for the clusterings produced by 
single  linkage clustering 
and the two variants of our heuristic 
approach. 
}
\end{figure*}

\clearpage
\subsection{Experiments on the Indian Liver Patient Data Set}\label{appendix_exp_general_liver}

\vspace{5mm}
\begin{figure*}[h]
\centering
\includegraphics[width=0.74\textwidth]{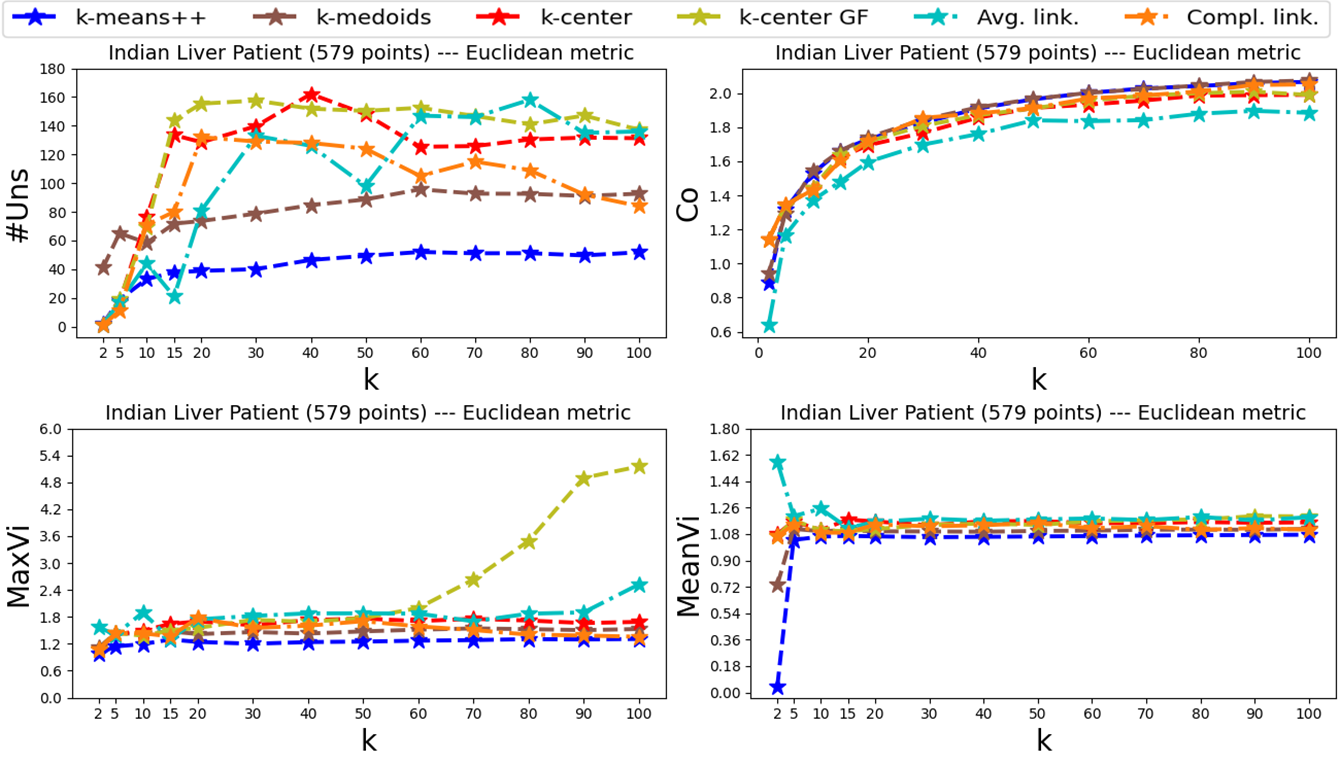}
\caption{Indian Liver Patient data set with Euclidean metric: 
$\Nrunf$ \textbf{(top-left)}, 
$\cost$ \textbf{(top-right)},
$\Maxviol$ \textbf{(bottom-left)} 
and $\Meanviol$ \textbf{(bottom-right)}
for the clusterings produced by the various standard algorithms as a function of 
 the number of clusters~$k$.
$k$.
}\label{exp_gen_standard_alg_IndianLiver_euclidean}
\end{figure*}

\vspace{12mm}
\begin{figure*}[h]
\centering
\includegraphics[width=0.74\textwidth]{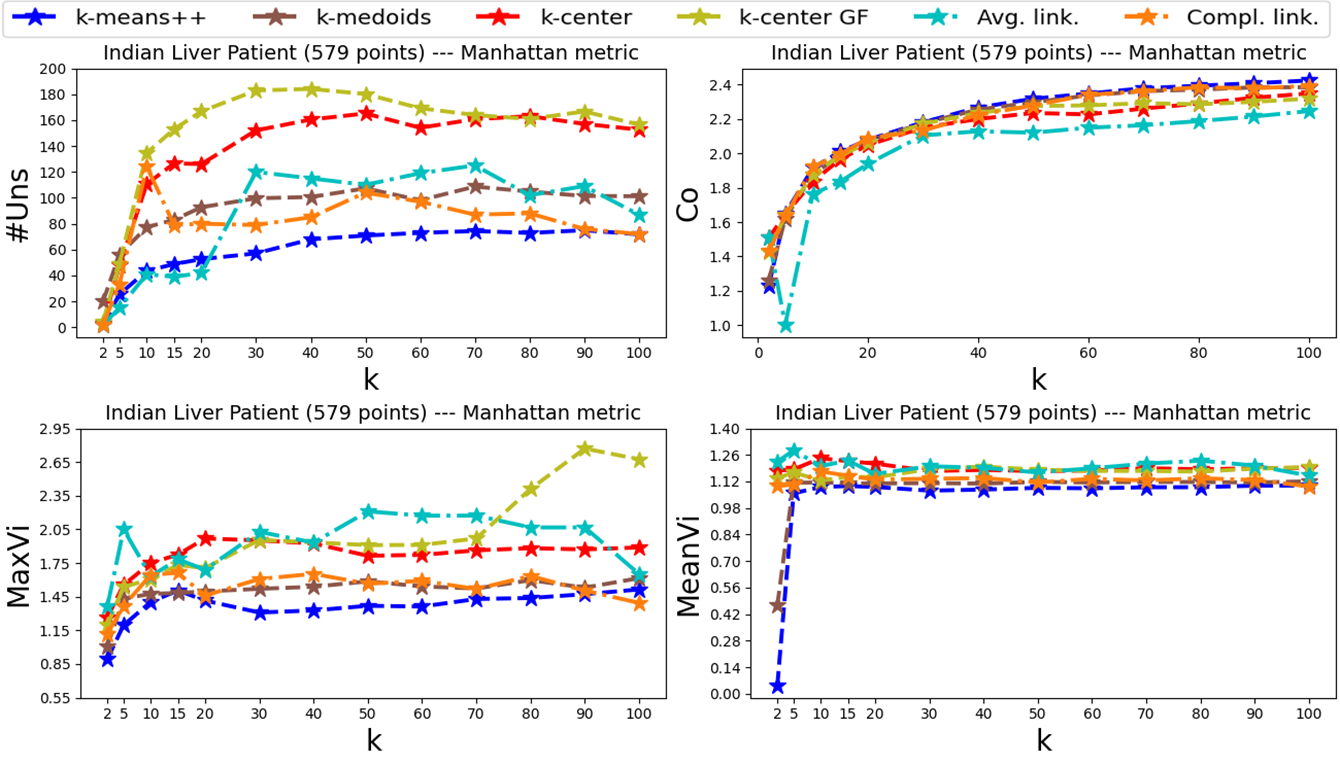}
\caption{Indian Liver Patient data set --- similar plots as in Figure~\ref{exp_gen_standard_alg_IndianLiver_euclidean}, 
but for the Manhattan metric. 
}
\end{figure*}

\begin{figure*}[h]
\centering
\includegraphics[width=0.9\textwidth]{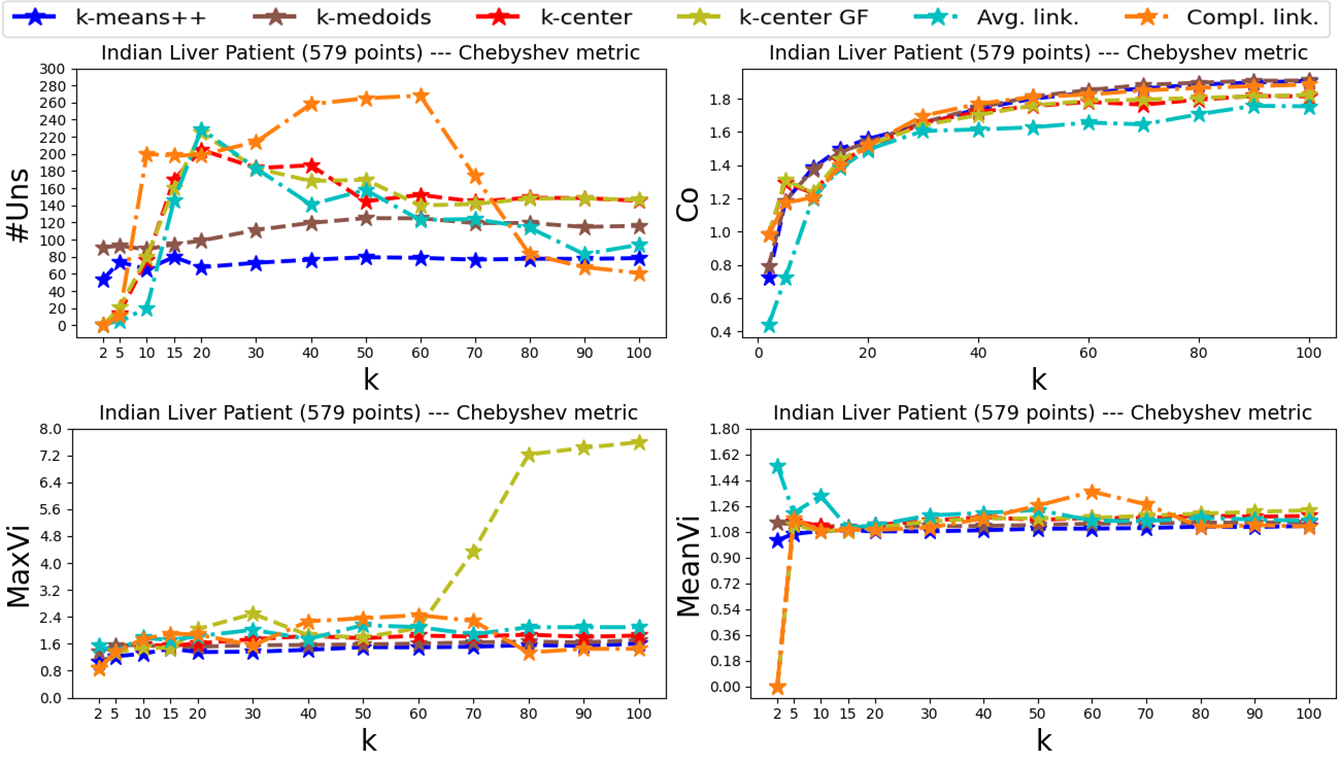}
\caption{Indian Liver Patient data set --- similar plots as in Figure~\ref{exp_gen_standard_alg_IndianLiver_euclidean}, but for the Chebyshev metric. 
}
\end{figure*}


\begin{figure}[h]
\centering
\includegraphics[width=0.9\textwidth]{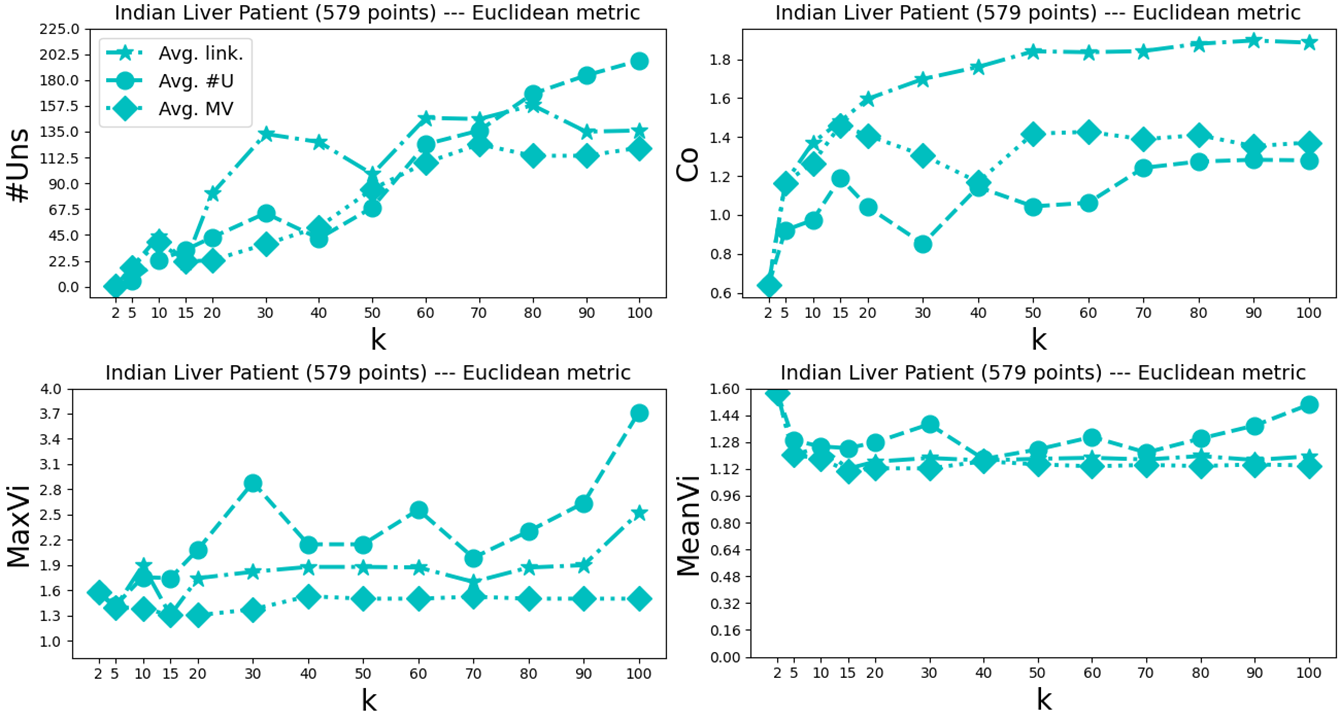}

\caption{Indian Liver Patient data set with Euclidean metric: 
$\Nrunf$  \textbf{(top-left)}, 
$\cost$ \textbf{(top-right)},
$\Maxviol$ \textbf{(bottom-left)}, 
and $\Meanviol$ \textbf{(bottom-right)}
for the clusterings produced by 
average linkage clustering 
and the two variants of our heuristic of Appendix~\ref{appendix_pruning_strategy} 
to improve it: the first ($\#$U in the legend) greedily chooses splits as to minimize $\Nrunf$, the second 
(MV)
as to minimize $\Maxviol$. 
}
\end{figure}

\begin{figure*}[h]
\centering
\includegraphics[width=0.9\textwidth]{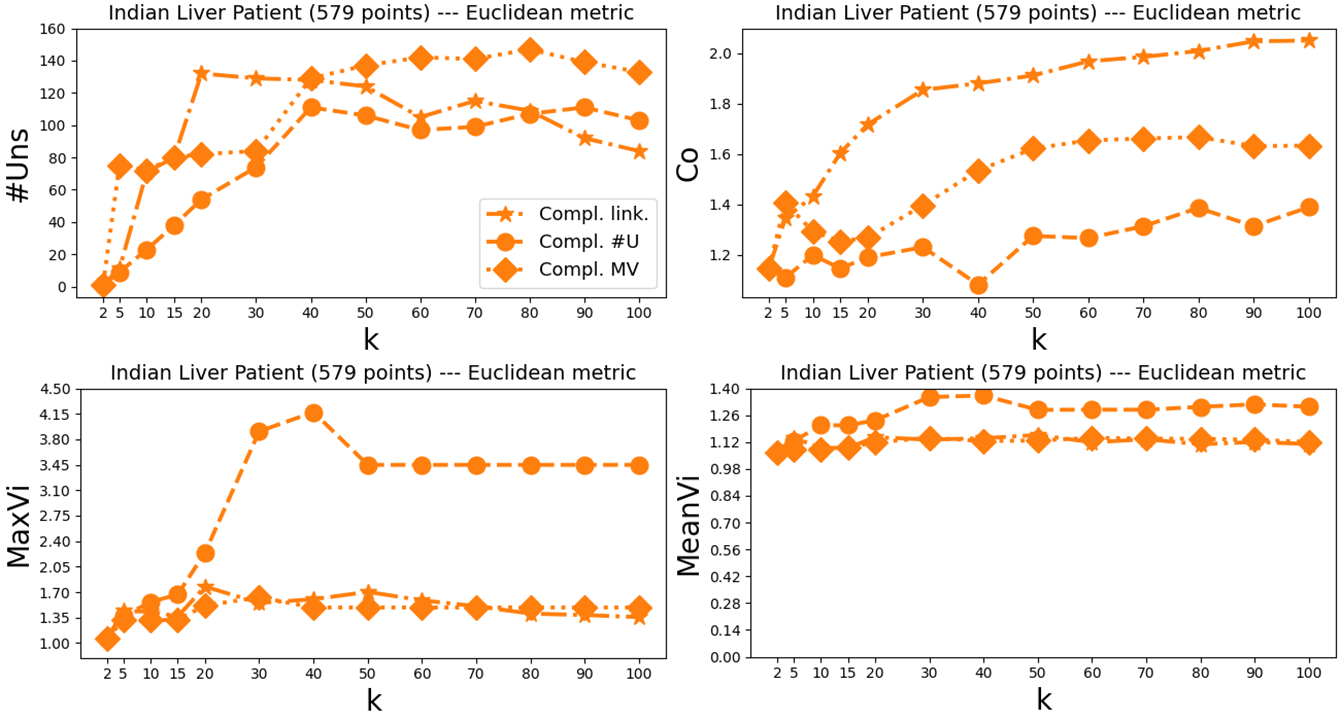}

\caption{Indian Liver Patient data set with Euclidean metric: 
the various measures 
for the clusterings produced by 
complete  linkage clustering 
and the two variants of our heuristic 
approach. 
}
\end{figure*}

\begin{figure*}[h]
\centering
\includegraphics[width=0.9\textwidth]{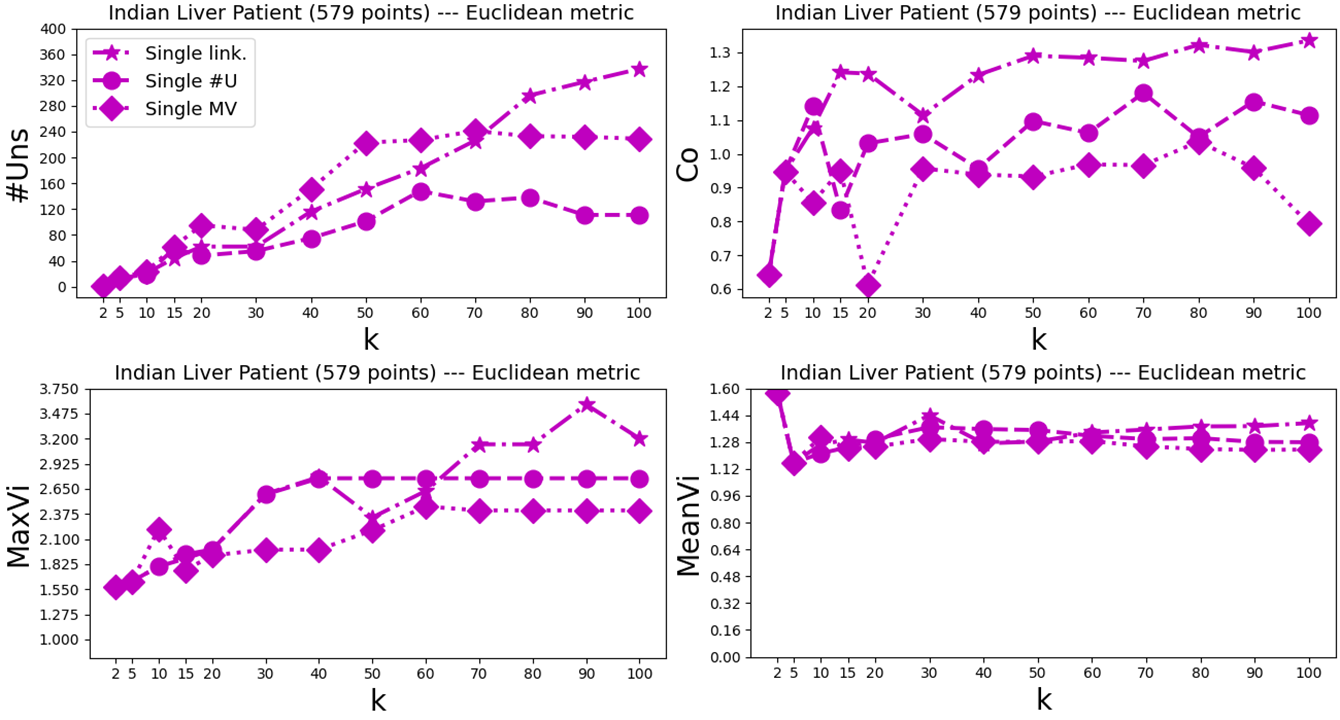}

\caption{Indian Liver Patient data set with Euclidean metric: 
the various measures 
for the clusterings produced by 
single  linkage clustering 
and the two variants of our heuristic 
approach.  
}
\end{figure*}

\begin{figure*}[h]
\centering
\includegraphics[width=0.9\textwidth]{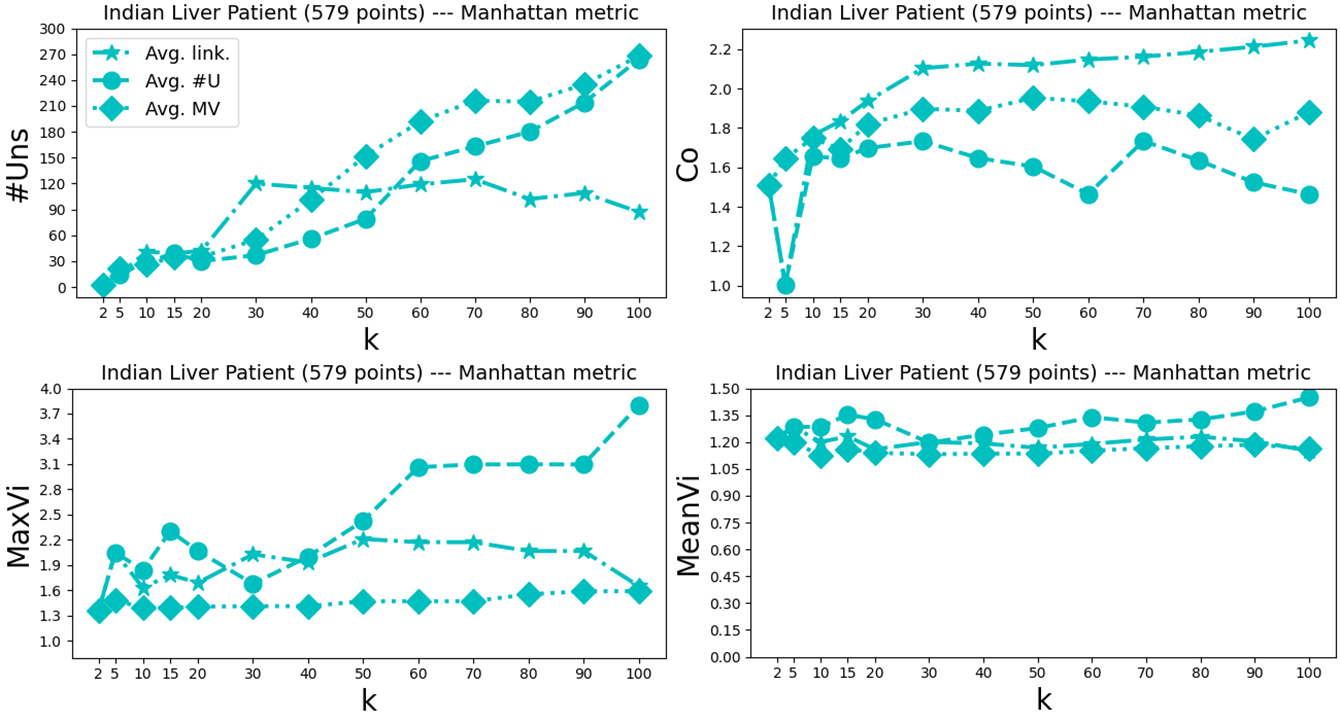}

\caption{Indian Liver Patient data set with Manhattan metric: 
the various measures 
for the clusterings produced by 
average  linkage clustering 
and the two variants of our heuristic 
approach. 
}
\end{figure*}

\begin{figure*}[h]
\centering
\includegraphics[width=0.9\textwidth]{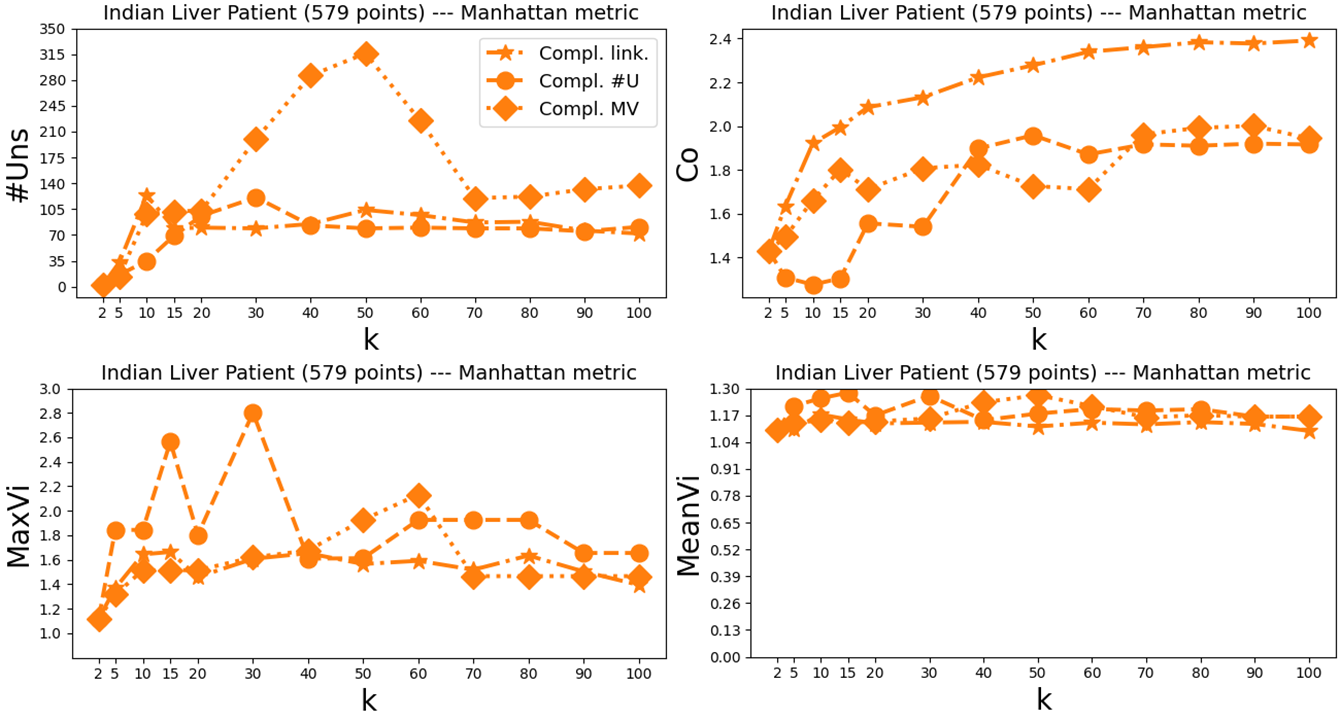}

\caption{Indian Liver Patient data set with Manhattan metric: 
the various measures 
for the clusterings produced by 
complete  linkage clustering 
and the two variants of our heuristic 
approach. 
}
\end{figure*}

\begin{figure*}[h]
\centering
\includegraphics[width=0.9\textwidth]{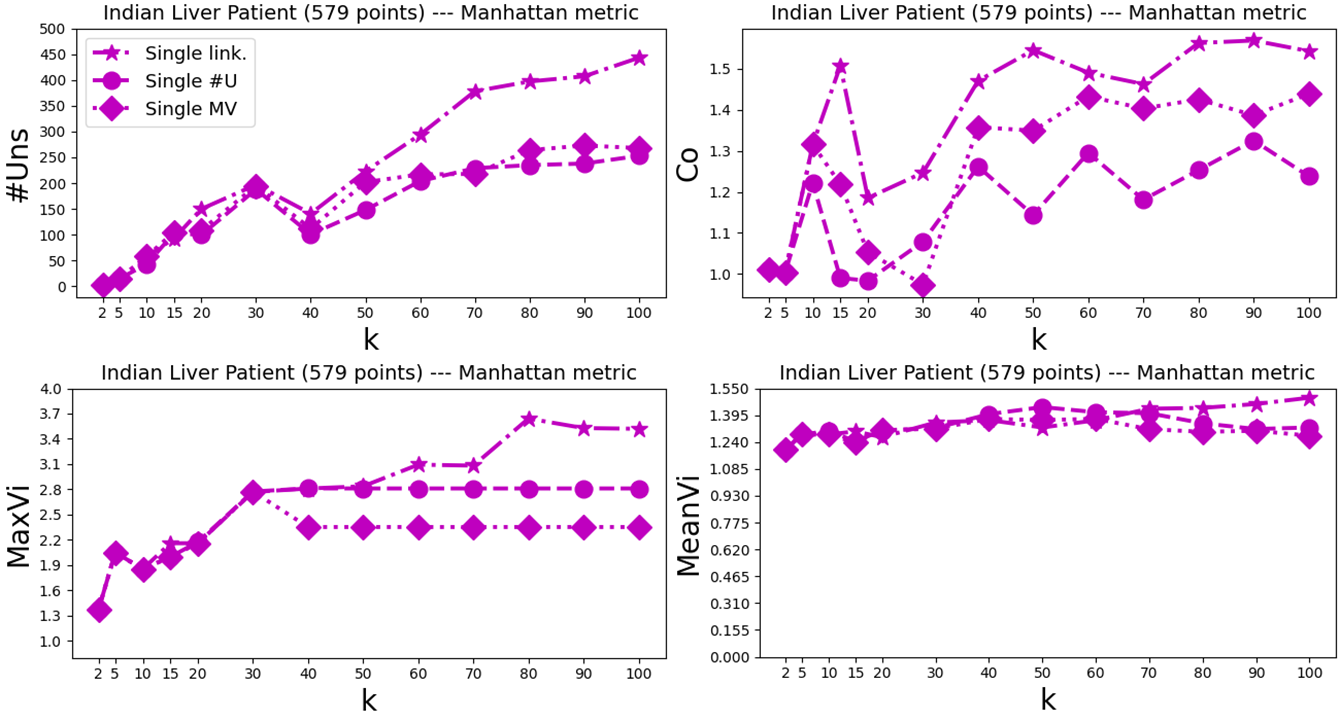}

\caption{Indian Liver Patient data set with Manhattan metric: 
the various measures 
for the clusterings produced by 
single  linkage clustering 
and the two variants of our heuristic 
approach. 
}
\end{figure*}

\begin{figure*}[h]
\centering
\includegraphics[width=0.9\textwidth]{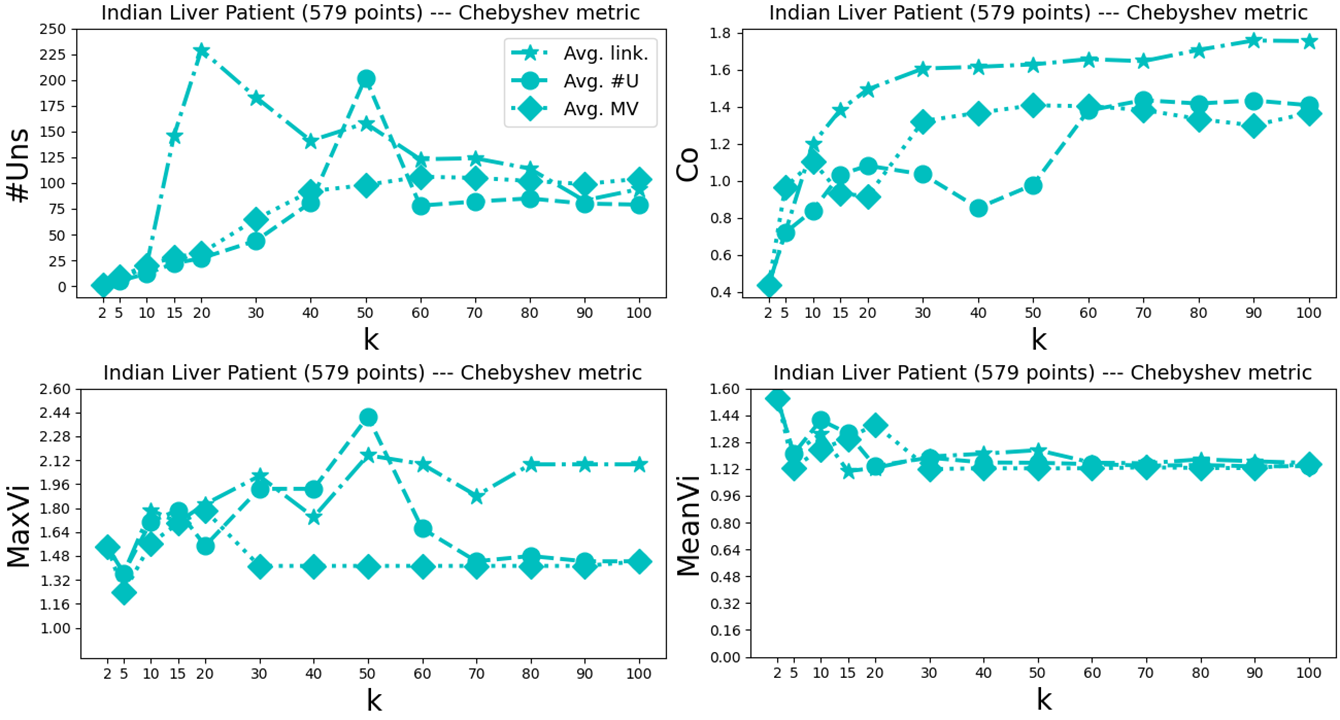}

\caption{Indian Liver Patient data set with Chebyshev metric: 
the various measures 
for the clusterings produced by 
average  linkage clustering 
and the two variants of our heuristic 
approach. 
}
\end{figure*}

\begin{figure*}[h]
\centering
\includegraphics[width=0.9\textwidth]{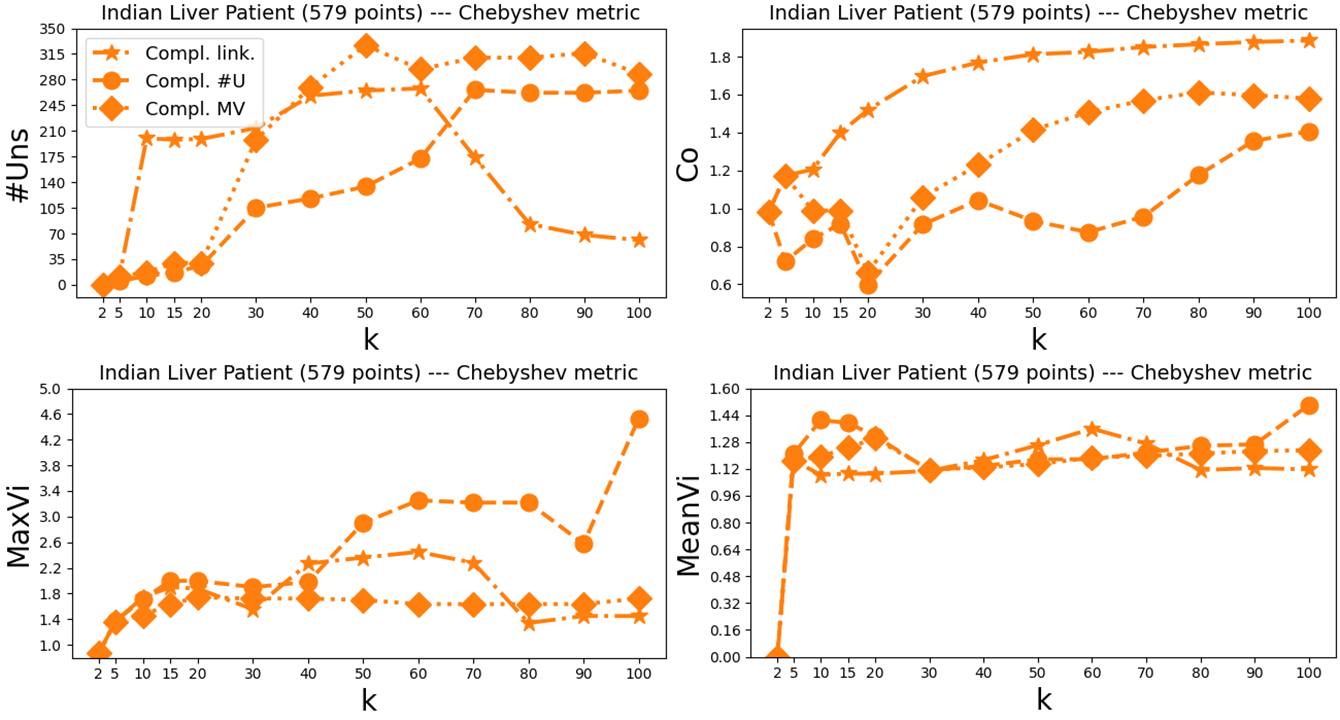}

\caption{Indian Liver Patient data set with Chebyshev metric: 
the various measures 
for the clusterings produced by 
complete  linkage clustering 
and the two variants of our heuristic 
approach. 
}
\end{figure*}

\begin{figure*}[h]
\centering
\includegraphics[width=0.9\textwidth]{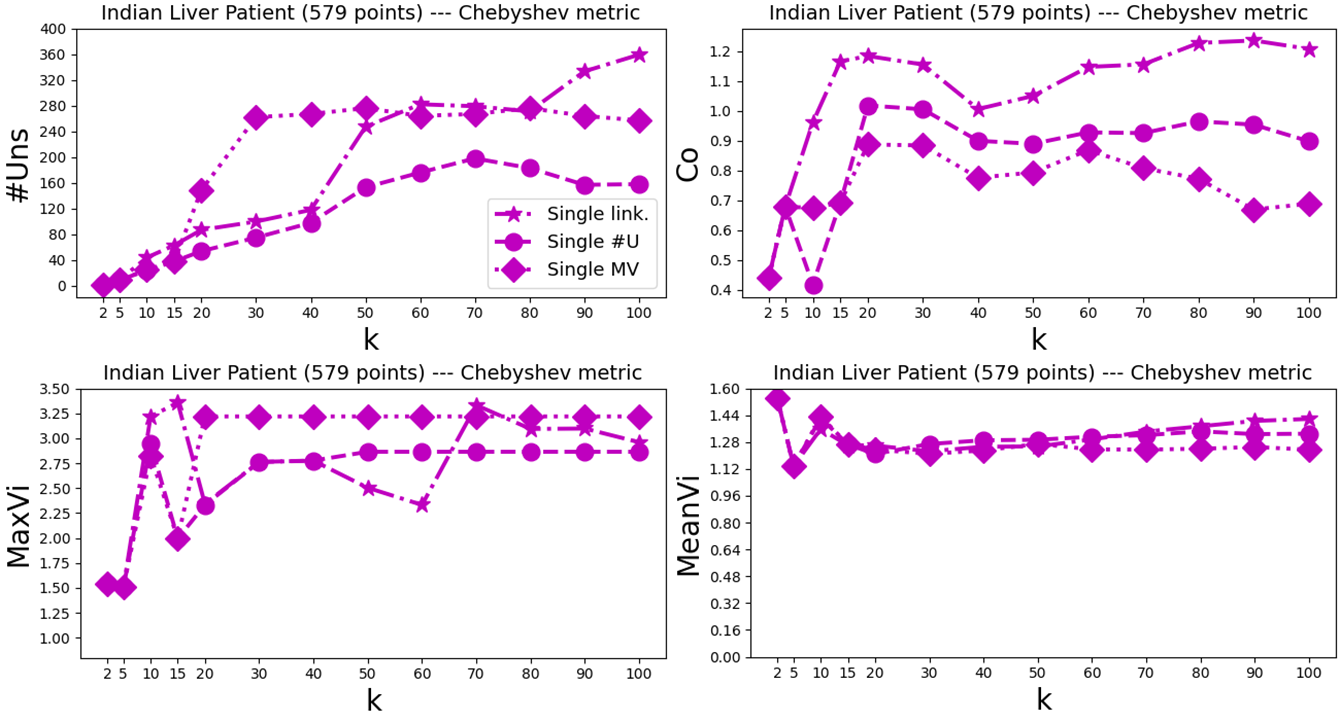}

\caption{Indian Liver Patient data set with Chebyshev metric: 
the various measures 
for the clusterings produced by 
single  linkage clustering 
and the two variants of our heuristic 
approach. 
}
\end{figure*}


\clearpage

\begin{figure}[t]
 \centering
 \includegraphics[scale=0.28]{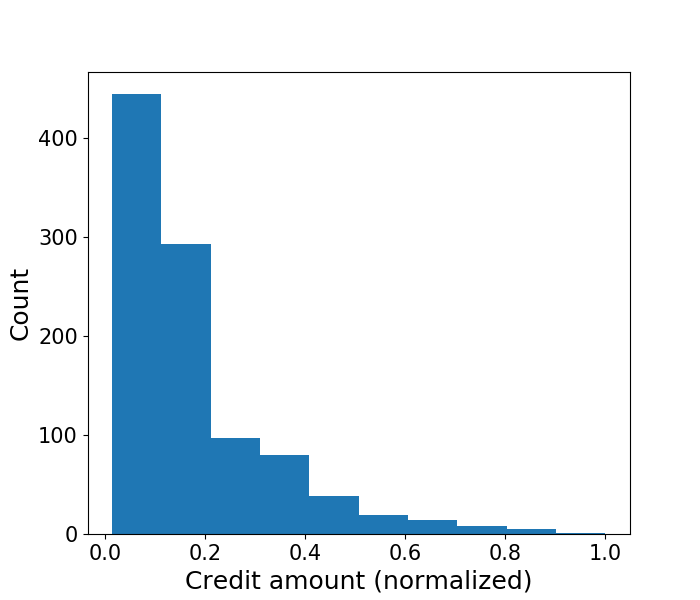}
 \hspace{1.5cm}
 \includegraphics[scale=0.28]{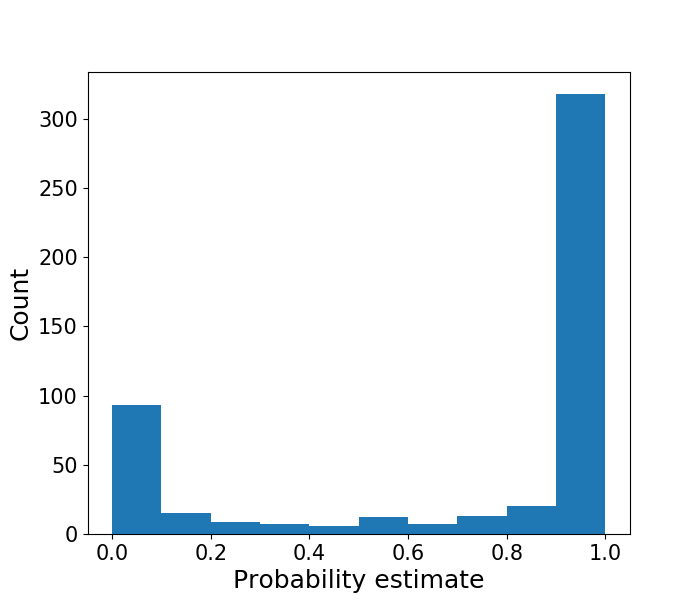}
 
 \caption{Histograms of the data sets used in the experiments of Appendix~\ref{appendix_exp_1dim}. \textbf{Left:} The credit amount 
 (one of the 20 features in the German credit data set; normalized to be in $[0,1]$) for the 1000 records in the German credit data set. 
 Note that there are only 921 unique values. \textbf{Right:} The estimated probability of having a good credit risk 
 for the second 500 records in the German credit data set. The estimates are obtained from a multi-layer perceptron trained on the first 500 records in the German
  credit data set.}\label{plot_histogram_1D_data}
\end{figure}

 \begin{table}[t]
 \caption{Experiment on German credit data set. Clustering 1000 people according to their credit amount. Target cluster sizes $t_i=\frac{1000}{k}$, $i\in [k]$. 
 \textsc{Naive}$=$naive clustering that matches the target cluster sizes, \textsc{DP}$=$dynamic programming approach of Appendix~\ref{appendix_p_equals_infty}, 
\textsc{$k$-means}$=k$-means initialized with medians of the 
clusters of the naive clustering, \textsc{$k$-me++}$=k$-means++. 
Results for 
\textsc{$k$-me++}  
averaged over 100 runs. Best values in~bold.
 }
 \label{experiment_1Da}
 \vskip 0.15in
 \begin{center}
 \begin{small}
 \begin{sc}
 \begin{tabular}{clccccc}
 \toprule
 & & $\Nrunf$ & $\Maxviol$ & $\Meanviol$ & $\Obj$ & $\cost$ \\
 \midrule
 \parbox[t]{2mm}{\multirow{4}{*}{\rotatebox[origin=c]{90}{$k=5$}}} & 
 Naive &  105 & 2.95 & 1.47 &\textbf{0} &  \textbf{0.24} \\
 &DP & \textbf{0} & \textbf{1.0} & \textbf{0}  & 172 &  0.28\\
 &$k$-means & 1 & \textbf{1.0}  & 1.0 & 170 & 0.28\\
 &$k$-me++& 1.11 & 1.01 & 0.61 & 275.3 &  0.29 \\
 \midrule
 \parbox[t]{2mm}{\multirow{4}{*}{\rotatebox[origin=c]{90}{$k=10$}}} & 
 Naive & 113 & 2.16 & 1.3 & \textbf{0} & \textbf{0.28} \\
 &DP & \textbf{0} & \textbf{1.0} &\textbf{0} & 131 & 0.3 \\
 &$k$-means &  4 & 1.01 &1.01 & 136 & 0.29 \\
 &$k$-me++& 3.16 & 1.02 & 0.99 & 159.47 & 0.29  \\
 \midrule
 \parbox[t]{2mm}{\multirow{4}{*}{\rotatebox[origin=c]{90}{$k=20$}}} & 
 Naive & 92 & 3.17 & 1.3 & \textbf{0} & 0.3 \\
 &DP & \textbf{0} & \textbf{1.0} & \textbf{0}& 37 & 0.3 \\
 &$k$-means & 5 & 1.01 & 1.01 & 37 & 0.29 \\
 &$k$-me++& 7.44 & 1.06 & 1.03 & 105.06 &\textbf{0.26} \\
 \midrule
 \parbox[t]{2mm}{\multirow{4}{*}{\rotatebox[origin=c]{90}{$k=50$}}} & 
 Naive & 101 & 2.6 & 1.3 & \textbf{0} & 0.32 \\
 &DP & \textbf{0} & \textbf{1.0} &\textbf{0} & 8 & 0.3\\
 &$k$-means & 18 & 1.26 & 1.06 & 10 &  0.3 \\
 &$k$-me++ & 12.7 & 1.19 & 1.05 & 50.22 &  \textbf{0.26} \\
 \bottomrule
 \end{tabular}
 \end{sc}
 \end{small}
 \end{center}
 \vskip -0.1in
 \end{table}

\section{Experiments on 1-dimensional Euclidean data sets
}\label{appendix_exp_1dim}

In this section we present experiments with our dynamic programming (DP)  approach of Appendix~\ref{appendix_p_equals_infty} for 1-dim Euclidean data sets. 
We used the German Credit data set 
\citep{UCI_all_four_data_sets_vers2}.  
It 
comprises 1000 records (corresponding to 
human beings) 
and for each record one binary label (good vs. bad credit risk) and~20~features. 


In our first experiment, we clustered the 1000 people according to 
their credit amount, which is one of the 20 features.   
A~histogram of the data  
can be seen in 
the left plot of 
Figure~\ref{plot_histogram_1D_data}. We 
were aiming 
for $k$-clusterings with clusters of equal size (i.e., 
target cluster sizes $t_i=\frac{1000}{k}$, $i\in [k]$) and compared our DP approach 
with $p=\infty$ to 
$k$-means clustering 
as well as 
a naive clustering that simply puts the 
$t_1$ smallest 
points in the first cluster, the next $t_2$ many points in the second cluster, and so on. 
We considered two initialization strategies for $k$-means:  we either used the medians of the clusters of the naive clustering for initialization 
(thus, hopefully, biasing $k$-means towards the target cluster sizes) or we ran $k$-means++ \citep{kmeans_plusplus}. 
For the latter we report average 
results obtained from running the experiment for 100 times. 
In addition to the four quantities 
 $\Nrunf$, $\Maxviol$, $\Meanviol$ and  $\cost$ 
considered in the experiments of Section~\ref{section_experiments_general} / Appendix~\ref{appendix_exp_general_adult}~-~\ref{appendix_exp_general_liver} (defined at the beginning of Section~\ref{section_experiments}), 
we report $\Obj$ 
(``objective''),  
which is the value of the objective function of 
\eqref{1dim-problem} 
for $p=\infty$. 
Note that $k$-means / $k$-means++ 
yields contiguous clusters and 
$\Obj$ is meaningful for all four clustering methods that we consider.

The results are provided in 
Table~\ref{experiment_1Da}, where we consider $k=5$, $k=10$, $k=20$ or $k=50$.  
As expected, for the naive clustering we always have $\Obj=0$, and
for our  
DP approach (\textsc{DP}) we have $\Nrunf=0$, $\Maxviol\leq 1$ and $\Meanviol=0$.
Most interesting to see is that both versions of $k$-means yield almost perfectly stable clusterings when $k$ is small 
and moderately stable clusterings when $k=50$ (with $k$-means++ outperforming $k$-means).

In our second experiment,  
we used the first 500 records 
of the data set 
to train a multi-layer perceptron (MLP)
for predicting the label (good vs. bad credit risk). 
We then 
applied 
the MLP to estimate the probabilities of having a good credit risk for the other 500 people. A histogram of the probability estimates is shown in the right plot of Figure~\ref{plot_histogram_1D_data}.     
We used the same clustering methods as in the first experiment to cluster 
the 500 people according to their probability estimate. 
We believe that such a clustering problem may arise 
frequently in practice (e.g., when a bank determines its lending policy) and that IP-stability is highly desirable in 
this 
context.

Table~\ref{experiment_1Db} and Table~\ref{experiment_1Dc} provide the results. 
In Table~\ref{experiment_1Db}, we 
consider 
uniform target cluster sizes $t_i=\frac{500}{k}$, $i\in [k]$, while in Table~\ref{experiment_1Dc} we consider various non-uniform target cluster sizes. The interpretation of the 
results is similar as for the first experiment of Section~\ref{section_experiments_1D}. Most notably, $k$-means can be quite unstable with up to 33 data points being 
unstable when $k$ is large, whereas $k$-means++ produces very stable clusterings with not more than three data points being unstable. However, $k$-means++ 
performs 
very poorly
in terms of $\Obj$, which can be almost ten times as large as for $k$-means and our dynamic programming approach  (cf.~Table~\ref{experiment_1Db},~$k=50$). 

The MLP that we used 
for predicting the label (good vs. bad credit risk) in the second experiment of Section~\ref{section_experiments_1D} has 
three hidden layers of size 100, 50 and 20, respectively, and a test accuracy of 0.724.

 \begin{table}[h!]
\caption{Experiment on German credit data set. Clustering the second 500 people according to their estimated probability of having a good credit risk. 
Target cluster sizes $t_i=\frac{500}{k}$, $i\in [k]$. 
\textsc{Naive}$=$naive clustering that matches the target cluster sizes, \textsc{DP}$=$dynamic programming approach of Appendix~\ref{appendix_p_equals_infty}, 
\textsc{$k$-means}$=k$-means initialized with medians of the 
clusters of the naive clustering, \textsc{$k$-me++}$=k$-means++. 
Results for 
\textsc{$k$-me++}  
averaged over 100 runs. Best values in~bold.
}
\label{experiment_1Db}
\vskip 0.15in
\begin{center}
\begin{small}
\begin{sc}
\begin{tabular}{cclccccc}
\toprule
& ~~~~~Target cluster sizes~~~~~ & & $\Nrunf$ & $\Maxviol$ & $\Meanviol$ & $\Obj$ & $\cost$ \\
\midrule
\parbox[t]{4mm}{\multirow{4}{*}{$k=5$}} & 
\multirow{4}{*}{\shortstack{$t_1=\ldots=t_5=100$}} 
& Naive & 197 & 58.28 & 9.52 & \textbf{0} &  0.34 \\
&&DP & \textbf{0} & \textbf{0.99} & \textbf{0} & 214  & \textbf{0.29} \\
&&$k$-means & 1 & 1.02 & 1.02 & 212 & \textbf{0.29} \\
&&$k$-me++ & 0.67 & 1.01  & 0.62 & 220.72  & 0.3 \\
\midrule
\parbox[t]{4mm}{\multirow{4}{*}{$k=10$}} & \multirow{4}{*}{\shortstack{$t_1=\ldots=t_{10}=50$}} 
& Naive & 162 & 10.27 & 3.06 & \textbf{0} & 0.34  \\
&&DP & \textbf{0} & \textbf{0.98} &\textbf{0} & 217 & \textbf{0.29}  \\
&&$k$-means & 8 & 1.25 & 1.09 & 207 & 0.34  \\
&&$k$-me++ & 1.11 & 1.02 & 0.87 & 248.86 & \textbf{0.29}\\
\midrule
\parbox[t]{4mm}{\multirow{4}{*}{$k=20$}} & \multirow{4}{*}{\shortstack{$t_1=\ldots=t_{20}=25$}} 
& Naive & 116 & 9.64 & 2.09 & \textbf{0} & 0.34 \\
&&DP & \textbf{0} & \textbf{1.0} &\textbf{0} & 155 & 0.29   \\
&&$k$-means & 33 & 2.13 & 1.38 & 95 & 0.3 \\
&&$k$-me++ & 2.53 & 1.07 & 0.97 & 239.55  & \textbf{0.28} \\ 
\midrule
\parbox[t]{4mm}{\multirow{4}{*}{$k=50$}} & \multirow{4}{*}{\shortstack{$t_1=\ldots=t_{50}=10$}} 
& Naive &  73 & 3.8 & 1.78 &  \textbf{0} & 0.34 \\
&&DP & \textbf{0} & \textbf{1.0} &\textbf{0} & 24 & 0.3\\
&&$k$-means &  28 & 2.29 & 1.25& 13 &  0.3 \\
&&$k$-me++ & 3.56  & 1.24 & 1.08 & 233.09 & \textbf{0.23}  \\
\bottomrule
\end{tabular}
\end{sc}
\end{small}
\end{center}
\vskip -0.1in
\vspace{4mm}
\end{table}

 \begin{table}[t]
\caption{Experiment on German credit data set. Clustering the second 500 people according to their estimated probability of having a good credit risk. 
Various 
non-uniform 
target cluster sizes.  
\textsc{Naive}$=$naive clustering that matches the target cluster sizes, \textsc{DP}$=$dynamic programming approach of Appendix~\ref{appendix_p_equals_infty}, \textsc{$k$-means}$=k$-means initialized with medians of the 
clusters of the naive clustering, \textsc{$k$-me++}$=k$-means++. 
Results for 
\textsc{$k$-me++}  
averaged over 100 runs. Best values in~bold.
}
\label{experiment_1Dc}
\vskip 0.15in
\begin{center}
\begin{small}
\begin{sc}
\begin{tabular}{cclccccc}
\toprule
& ~~~~~Target cluster sizes~~~~~ & & $\Nrunf$ & $\Maxviol$ & $\Meanviol$ & $\Obj$ & $\cost$ \\
\midrule
\parbox[t]{9mm}{\multirow{4}{*}{$k=12$}} & \multirow{4}{*}{\shortstack{$t_i=\begin{cases}50 &\text{\textnormal{for}~}3\leq i\leq 10 \\25& \text{\textnormal{else}} \end{cases}$}} 
& Naive & 188 & 12.85 & 3.22 & \textbf{0} & 0.34 \\
&&DP & \textbf{0} & \textbf{0.97} & \textbf{0}& 232  & \textbf{0.29} \\ 
&&$k$-means & 3 & 1.05 & 1.04 & 217 & \textbf{0.29} \\
&&$k$-me++ & 1.44 & 1.03 & 0.77 & 254.81 & \textbf{0.29} \\
\midrule
\parbox[t]{9mm}{\multirow{6}{*}{$k=12$}} & \multirow{6}{*}{\shortstack{$t_1=t_{12}=10$,\\$t_2=t_{11}=15$,\\$t_3=t_{10}=25$,\\$t_4=t_9=50$,\\$t_5=t_8=50$,\\$t_6=t_7=100$}} 
& & & & & &\\
&& Naive & 251 & 65.99 & 6.56 & \textbf{0} &  0.33 \\ 
&&DP & \textbf{0} & \textbf{0.97} & \textbf{0}  & 247 & \textbf{0.29}\\   
&&$k$-means & 14 & 1.52 & 1.16 & 238 & 0.3\\
&&$k$-me++ & 1.43 & 1.03 & 0.83 & 270.09 & \textbf{0.29}\\
& & & & & & & \\
\midrule
\parbox[t]{9mm}{\multirow{4}{*}{$k=20$}} & \multirow{4}{*}{\shortstack{$t_i=\begin{cases}10 &\text{\textnormal{for}~}i=1,3,5,\ldots \\40& \text{\textnormal{for}~}i=2,4,6,\ldots \end{cases}$}}
& Naive & 189 & 137.31 & 6.64& \textbf{0} & 0.35 \\ 
&&DP & \textbf{0} & \textbf{1.0} &\textbf{0} & 140 & 0.29 \\ 
&&$k$-means & 30 & 1.91 &1.28 & 91 & 0.3 \\
&&$k$-me++ & 2.62 & 1.07 & 0.98 & 224.26 & \textbf{0.28} \\ 
\midrule
\parbox[t]{9mm}{\multirow{4}{*}{$k=20$}} & \multirow{4}{*}{\shortstack{$t_i=\begin{cases}115 &\text{\textnormal{for}~}i=10,11 \\15& \text{\textnormal{else}} \end{cases}$}}
& Naive & 224 & 215.88 & 13.01 & \textbf{0} & 0.34 \\ 
&&DP & \textbf{0} & \textbf{1.0} & \textbf{0} & 165 & 0.29 \\ 
&&$k$-means & 51 & 3.04 & 1.67 & 112 & 0.3 \\ 
&&$k$-me++ & 2.44 & 1.07 & 0.99 & 249.79 & \textbf{0.28} \\
\bottomrule
\end{tabular}
\end{sc}
\end{small}
\end{center}
\vskip -0.1in

\vspace{1.4cm}
\end{table}

\end{document}